\documentclass{article} 
\usepackage{iclr2026_conference,times}

\usepackage{amsmath,amsfonts,bm}

\def\eqref#1{equation~\ref{#1}}

\def\1{\bm{1}}

\DeclareMathAlphabet{\mathsfit}{\encodingdefault}{\sfdefault}{m}{sl}
\SetMathAlphabet{\mathsfit}{bold}{\encodingdefault}{\sfdefault}{bx}{n}

\def\sP{{\mathbb{P}}}

\newcommand{\E}{\mathbb{E}}

\newcommand{\R}{\mathbb{R}}

\DeclareMathOperator*{\argmax}{arg\,max}

\usepackage{amsmath,amsthm}
\usepackage{graphicx}
\usepackage[T1]{fontenc}
\usepackage[utf8]{inputenc}
\usepackage[inline]{enumitem}
\usepackage{dsfont}

\usepackage[english]{babel}
\usepackage{amssymb,amsfonts}
\usepackage{bbm}
\usepackage{nicefrac}
\usepackage{xfrac}
\usepackage{subcaption}
\usepackage{booktabs}
\usepackage{multirow}
\usepackage{tikz}
\usepackage{tikz-3dplot}
\usetikzlibrary{3d,arrows.meta,patterns,decorations.markings}

\usepackage[colorinlistoftodos,bordercolor=purple,backgroundcolor=pink!20,linecolor=pink,textsize=scriptsize]{todonotes}
\usepackage{wrapfig}
\usepackage{microtype}      
\usepackage{parskip}
\usepackage{caption}

\usepackage{thm-restate}
\usepackage{empheq}
\usepackage{color}
\definecolor{cool-blue}{HTML}{1D6996}
\definecolor{cool-purple}{HTML}{5F4690}
\definecolor{cool-teal}{HTML}{38A6A5}
\definecolor{cool-green}{HTML}{0F8554}
\definecolor{dark-green}{rgb}{0.0,0.5,0.0}

\newcommand{\algname}[1]{{\small\sf#1}}
\usepackage[pagebackref]{hyperref}
\hypersetup{
	colorlinks=true,
	linkcolor=cool-blue,
	filecolor=cool-blue,
	citecolor=dark-green,
	urlcolor=cool-blue
}
\renewcommand*{\backrefalt}[4]{%
    \ifcase #1 \footnotesize{(Not cited.)}%
    \or        \footnotesize{(Cited on page~#2)}%
    \else      \footnotesize{(Cited on pages~#2)}%
    \fi}

\usepackage{algorithm}
\usepackage{algpseudocode}
\usepackage[capitalize]{cleveref}
\usepackage[theorems,skins]{tcolorbox}
\usepackage{accents}
\usepackage{stackengine}
\stackMath          

\newcommand{\set}[1]{\mathcal{#1}}
\newcommand{\SM}{\ensuremath{%
  \stackon[-8.5pt]{\mathcal{M}}{\widetilde{\phantom{\mathcal{M}}}}%
}}
\newcommand{\CM}{\set{M}_c}
\newcommand{\D}{\set{D}}

\newcommand{\model}{\set{F}}

\newcommand{\opi}{\pi^*}
\newcommand{\sopi}{\bar \pi^*}
\newcommand{\ocpi}{\pi_c^*}
\newcommand{\socpi}{\bar{\pi}^{*}_{c,n}}
\newcommand{\sppi}{\bar{\pi}_n}
\newcommand{\ppi}{\pi_n}
\newcommand{\bpi}{\hat{\pi}}

\newcommand{\oca}{a_c^*}
\newcommand{\soca}{\bar{a}^{*}_{c,n}}

\newcommand{\pVc}{\bar V_c^{\bpi}}
\newcommand{\pVr}{\underline{V}^{\bpi}_r}
\newcommand{\Vc}{V_c^{\bpi}}
\newcommand{\tVc}{\tilde{V}_c^{\bpi}}
\newcommand{\tVtc}{\tilde{V}_{\tilde c}^{\bpi}}
\newcommand{\Vcn}{V_{c,n}^{\bpi}}
\newcommand{\Qcn}{Q_{c,n}^{\bpi}}
\newcommand{\Qtcn}{Q_{\tilde c,n}^{\bpi}}
\newcommand{\Qcz}{Q_{c,0}^{\bpi}}
\newcommand{\Vr}{V_r^{\bpi}}
\newcommand{\Jr}{J_r}
\newcommand{\Jtr}{J_{\tilde r}}

\newcommand{\Jc}{J_c}
\newcommand{\ct}{c_{<t}}
\newcommand{\PERa}{\Delta_n^1}
\newcommand{\PERb}{\Delta_n^2}
\newcommand{\eqdef}{\coloneqq}

\DeclareMathOperator{\st}{\,\mathrm{s.t. }\,}
\DeclareMathOperator{\ones}{\mathbf{1}}

\newtheorem{definition}{Definition}
\crefname{definition}{definition}{definitions}
\newtheorem{assumption}{Assumption}
\crefname{assumption}{assumption}{assumptions}

\crefname{theorem}{theorem}{theorems}

\crefname{corollary}{corollary}{corollaries}
\newtheorem{lemma}{Lemma}
\crefname{lemma}{lemma}{lemmas}
\theoremstyle{remark}
\newtheorem{remark}{Remark}
\crefname{remark}{remark}{remarks}

\usepackage[colorinlistoftodos,bordercolor=purple,backgroundcolor=pink!20,linecolor=pink,textsize=scriptsize]{todonotes}

\newcounter{relctr} 
\everydisplay\expandafter{\the\everydisplay\setcounter{relctr}{0}} 
\newcommand\labelrel[2]{%
  \begingroup
    \refstepcounter{relctr}%
    \stackrel{\textnormal{(\text{\roman{relctr})}}}{\mathstrut{#1}}%
    \originallabel{#2}%
  \endgroup
}
\AtBeginDocument{\let\originallabel\label} 

\title{Safe Exploration via Policy Priors}

\author{%
Manuel Wendl, Yarden As\thanks{Equal contribution. Corresponding authors: \texttt{\{mwendl,yardas\}@ethz.ch}} \\
ETH Zurich \\
\And
Manish Prajapat 
\\
ETH Zurich \\
\And
Anton Pollak 
\\
ETH Zurich \\
\AND
Stelian Coros
\\
ETH Zurich \\
\And
Andreas Krause
\\
ETH Zurich \\
}

\iclrfinalcopy 
\begin{document}

\maketitle

\begin{abstract}
\looseness=-1
Safe exploration is a key requirement for reinforcement learning (RL) agents to learn and adapt online, beyond controlled (e.g. simulated) environments. In this work, we tackle this challenge by utilizing suboptimal yet conservative policies (e.g., obtained from offline data or simulators) as priors. Our approach, \algname{SOOPER}, uses probabilistic dynamics models to optimistically explore, yet pessimistically fall back to the conservative policy prior if needed. We prove that \algname{SOOPER} guarantees safety throughout learning, and establish convergence to an optimal policy by bounding its cumulative regret. Extensive experiments on key safe RL benchmarks and real-world hardware demonstrate that \algname{SOOPER} is scalable, outperforms the state-of-the-art and validate our theoretical guarantees in practice.
\end{abstract}

\section{Introduction}
\label{sec:introduction}
\looseness=-1A defining feature of intelligence is the ability to utilize online streams of information to learn and adapt over time \citep{silver2025welcome}.
Reinforcement learning (RL) provides a framework for online learning without supervision, driving several notable real-world applications \citep{silver2017mastering,ouyang2022traininglanguagemodelsfollow}.
Safety is key to unlocking RL in the physical world---RL agents simply cannot risk taking actions that lead to catastrophic outcomes~\citep{dalrymple2024towards}, even \emph{while} learning, when they lack complete knowledge about their environments \citep{legg2023system}.
Such agents must \emph{explore safely} \citep{amodei2016concrete}, gradually discovering new safe behaviors to solve their tasks.

\looseness=-1
Safe exploration is a major challenge in practice. Methods with theoretical safety guarantees often struggle to scale to complex, general-purpose tasks, whereas scalable approaches typically fail to ensure safety during learning (cf. \Cref{sec:related-works}). A key reason lies in how they incorporate prior knowledge about the environment---a crucial component for anticipating danger without learning through harmful trial-and-error. Such prior knowledge can be instantiated via \emph{policies}. %
Beyond primarily maintaining safety, as done in much of prior work, we argue that prior policies can also \emph{guide} exploration toward promising regions of the environment, enabling agents to provably learn \emph{near-optimal} policies.

Motivated by this insight, we propose \textbf{S}afe \textbf{O}nline \textbf{O}ptimism for \textbf{P}essimistic \textbf{E}xpansion in \textbf{R}L---\algname{SOOPER}---a model-based RL algorithm for safe exploration in constrained Markov decision processes \citep[CMDP,][]{altman}. \algname{SOOPER} uses prior policies that can be derived from scarce offline data or simulation under distribution shifts. Such policies are invoked \emph{pessimistically} during online rollouts to maintain safety. These rollouts are collected \emph{optimistically} so as to maximize information about a world model of the environment. Using this model, \algname{SOOPER} constructs a simulated RL environment, used for planning and exploration \citep{amodei2016concrete}, whose trajectories \emph{terminate} once the prior policy is invoked. The key idea is that early terminations incentivize the agent to avoid the prior policy when trajectories with higher returns can be obtained safely. This design allows \algname{SOOPER} to leverage standard RL methods, bypassing the complexity of directly solving CMDPs. More fundamentally, this enables a novel bound on the agent's cumulative regret.

\paragraph{Our contribution.}
\begin{itemize}[leftmargin=0.5cm, parsep=0.3em]
    \item First, we propose \algname{SOOPER}, a scalable model-based RL algorithm for safe exploration in continuous state-action spaces that leverages offline- and simulation-trained policies as safe priors.
    \item \looseness=-1Next, under standard regularity assumptions, we prove that \algname{SOOPER} guarantees constraint satisfaction with high probability throughout learning. We additionally show that \algname{SOOPER} gradually expands an implicit set of safe policies until it converges to the feasible set defined by the true CMDP. This is accomplished by reformulating the problem as an unconstrained MDP, making the approach naturally compatible with standard deep RL methods, while avoiding min-max formulations typical for CMDPs. Using this formulation, we further establish a novel upper bound on the cumulative regret, improving over prior works that provide optimality guarantees only at the end of training with arbitrarily poor performance during exploration~\citep{Wagener2021SafeRL,as2025actsafe}.
    \item \looseness=-1Finally, we conduct extensive empirical evaluation of \algname{SOOPER}, comparing it with other state-of-the-art safe exploration algorithms when initialized from offline- and simulation-trained policies in RWRL benchmark \citep{challengesrealworld} and SafetyGym \citep{Ray2019}. In all these experiments, \algname{SOOPER} substantially outperforms the baselines in terms of safety and performance. We further validate \algname{SOOPER} for safe online learning on \emph{hardware}, giving empirical evidence that our theory translates to the real world.
\end{itemize}

\section{Related Works}
\label{sec:related-works}
Prior works on safe exploration fall broadly into two categories: theoretically principled methods with guarantees on safety and optimality and those prioritizing practicality and scalability.

\paragraph{Provable safety.}\looseness=-1 
A canonical line of work leverages safe Bayesian optimization \citep{pmlr-v37-sui15}, which has seen real-world success \citep{Berkenkamp2016BayesianOW,kirschner19a,Bhavya}, but typically ignores the Markovian structure of sequential decision problems and is limited to a few parameters. 
\citet{Berkenkamp2017SafeMR} extend to continuous MDPs and provide safety and optimality guarantees via Lyapunov stability, though they rely on manual resets of the system at each timestep and state-space discretization for safe set computation.
In contrast, \citet{manish} leverage model predictive control (MPC) to implicitly compute these sets and ensure that the agent can always return to a safe state.
Other MPC-based methods ensure safety either through ``backup'' policies~\citep{Koller2018LearningBasedMP,prajapat2024towards,prajapat2025finite} or via low-fidelity models with reachability analysis \citep{hewing2019cautious} and safety filters \citep{wabersich2021predictive,curi2022safe}. However, these approaches lack optimality guarantees.
A substantial body of work addresses safe exploration in \emph{discrete} CMDPs~\citep{moldovan2012safe,turchetta2016safe,Wachi_Sui_Yue_Ono_2018,wachi2020safereinforcementlearning,ding2021provably,simao2021alwayssafe,prajapat2022near,bura2022dope}. 
Akin to our work, \citet{bura2022dope} propose a model-based approach that uses optimism and pessimism to provide optimality and safety guarantees for small-scale tabular problems. 
Although the above methods guarantee safety throughout learning, scaling them to high-dimensional problems is computationally hard.

\paragraph{Scalability.}
\looseness=-1
The works of \citet{dalal2018safe,eysenbach2018leave,srinivasan2020learning,RecoveryRL,bharadhwaj2021conservative,thomas2021safe,sun2021safe,liu2022constrained,as2022constrained,sootla2022saute,huang2024safedreamer} use scalable tools from deep RL for safe learning, with different levels of theoretical guarantees and effectiveness in maintaining safety throughout learning in practice. Safe exploration has also been studied from an \emph{optimization} perspective, with methods aiming to guarantee feasibility of all optimization iterates. These techniques span trust regions~\citep{cpo,milosevic2024embedding}, Lyapunov stability~\citep{chow2018lyapunov}, interior-point \citep{liu2020ipo,JMLR:v25:22-0878,ni2025safe}, and primal-dual optimization \citep{usmanova2025safe}. While these methods scale well, they often overlook the exploration aspect of learning, namely how to actively sample data that maximizes information about the environment. As a result, they often lack optimality except under assumptions on the initial state distribution.

\paragraph{Optimality.} 
\looseness=-1
\citet{as2025actsafe} propose \algname{ActSafe}, which guarantees safety and optimality for \emph{simple} regret, relying on reward-free exploration, and therefore might perform poorly during exploration. \citet{NEURIPS2023_5d4cd12e} introduce \algname{MASE}, which enforces safety under high-probability constraints on safe states but relies on a restrictive ``emergency stop'' action that can effectively `stop time' at any state. Lastly, \citet{Wagener2021SafeRL} propose \algname{SAILR}, a state-of-the-art algorithm that like \algname{SOOPER}, reformulates safety through terminations whenever their ``backup'' policy is invoked. However, its guarantee on simple regret hinges on the probability of such resets vanishing over time, yet no formal conditions are provided to ensure this occurs. In this work, we relax several of these assumptions and establish a novel bound on the cumulative regret, while also demonstrating improved empirical performance.

\section{Problem Setting}
\label{sec:problem-setting}
\paragraph{Constrained Markov decision processes.}
\looseness=-1
Safety in RL has been conceptualized in several ways across the literature \citep{JMLR:v16:garcia15a,brunke2022safe,gu2024review}. In this work, we focus on constrained Markov decision processes \citep[CMDP, ][]{altman}, due to their flexibility and capacity to generalize several notable formulations of safety \citep[cf.][]{Wagener2021SafeRL,curi2022safe,NEURIPS2023_5d4cd12e}. A discounted infinite-horizon CMDP is defined by the tuple \(\CM \eqdef (\set{S},\set{A},p,r,c,\gamma,\rho_0)\), where \(\set{S} \subseteq \R^{d_{\set{S}}}\) and \(\set{A}\subset \R^{d_{\set{A}}}\) are the continuous state and compact action spaces, with the transition probability \(p(s_{t+1}|s_t,a_t)\) of reaching \(s_{t+1}\in \set{S}\) by applying \(a_t\in \set{A}\) in \(s_t \in \set{S}\). We consider the transition probabilities \(p\) to be governed by the \emph{unknown} stochastic dynamics 
\begin{equation}\label{eq:dynamics}
    s_{t+1} = f(s_t,a_t) + \omega_t, \quad (s_t,a_t) \in \set{S}\times\set{A}, \quad s_0\sim\rho_0,
\end{equation}
with the initial state distribution \(\rho_0\) and additive noise \(\omega_t \in \set{W}\subseteq \R^{d_{\set{S}}}\). Throughout our theoretical analysis, we consider the setting of additive zero-mean Gaussian noise. 
\begin{assumption}[Gaussian noise \citep{KakadeInfo}]\label{assume:noise}
    The additive noise \(\omega_t\) is independent and identically distributed (i.i.d.) zero-mean Gaussian noise with variance \(\sigma^2\).
\end{assumption}
We note that this assumption can be relaxed to a richer class of sub-Gaussian distributions.
The reward and the cost are given by \(r:\set{S}\times \set{A} \to [0,R_{\text{max}}]\) and \(c: \set{S} \times \set{A} \to [0,C_{\text{max}}]\) with discount factor \(\gamma \in (0,1)\). We assume the following regularity conditions for the dynamics, reward, and cost.
\begin{assumption}[Lipschitz-continuity \citep{Berkenkamp2017SafeMR}]\label{assume:regularity}
    The dynamics \(f\), reward \(r\), and cost \(c\) are Lipschitz continuous with known constants \(L_f, L_r, L_c\).
\end{assumption}
\begin{assumption}[Regularity \citep{Berkenkamp2017SafeMR}]\label{assume:RKHS}
    The unknown dynamics \(f\) lie in an RKHS associated with a kernel \(k\) and have a component-wise bounded norm \(||f_i||_k \leq B\) with a known constant \(B\). Moreover, we assume the kernel \(k\) to be bounded by \(k((s,a),(s,a)) \leq k_{max}, \forall (s,a) \in \set{S}\times \set{A}\).
\end{assumption}
\looseness=-1Intuitively, this assumption means that the complexity of $f$ is controlled by the kernel, which in turn enables generalization from limited data. 
Our analysis uses \Cref{assume:RKHS} to relate the statistical complexity of learning the dynamics to the number of online episodes required for convergence. 
In \Cref{sec:experiments}, we model it using neural networks that can  learn features to represent $f$ accurately.

The goal in our setting is to find a stationary, stochastic policy \(\ocpi(a_t|s_t)\) which admits a global density bound \(\ocpi(a|s)\le p_{\max}\) for the true dynamics \(f\) that maximizes value function \(\smash{V_r^{\pi}(s) \eqdef \E_{\pi}[\sum_{t=0}^\infty \gamma^t r(s_t,a_t) \mid s_0=s]}\), while ensuring that the cost value function \(\smash{V_c^{\pi}(s) \eqdef \E_{\pi}[\sum_{t=0}^\infty \gamma^t c(s_t,a_t) \mid s_0=s]}\) remains below a safety budget \(d\in \R_{\geq 0}\):
\begin{align} \label{eq:objective}
\begin{aligned}
    \ocpi \in \argmax_\pi \Jr(\pi,f) \eqdef& \E_{s_0\sim\rho_0}\left[V_r^{\pi}(s_0)\right] \quad \st \quad \Jc(\pi,f) \eqdef \E_{s_0\sim\rho_0} \left[V_c^{\pi}(s_0)\right] \leq d.
\end{aligned}
\end{align}
We hereafter consider in our theoretical analysis strict feasibility of \(\ocpi\); namely that an optimal policy of the constrained problem does not lie on the boundary of the feasible set but may be arbitrarily close, such that \(\Jc(\ocpi,f)<d\).
\Cref{eq:objective} naturally decouples the objective from safety, so each can be solved independently. This explicit separation prevents agents from exploiting the reward function to achieve high performance at the expense of harmful behaviors, a problem commonly referred to as \emph{reward hacking} in the AI safety literature \citep{amodei2016concrete}.
 
\paragraph{Safe online learning.}
Since knowledge of the dynamics is limited a priori, the agent must learn them through online interaction. We consider learning in an \emph{episodic} setting, where in each episode \(n=1,\dots N\), infinite-horizon trajectories are truncated after $T_n$ steps and the agent is reset to a new initial state~\citep{puterman2014markov,sutton}. The agent deploys a policy \(\ppi\), collecting transitions \(\smash{\D_n = \{(s_{t},a_{t},s_{t+1})_{t=1,\dots T_n}\}}\). The optimality gap due to incomplete knowledge over episodes is quantified by cumulative regret, namely, the difference in performance between $\ppi$ and the optimal policy given full access to $\CM$:
\begin{equation}\label{eq:objective2}
    R(N) = \sum_{n=1}^N \big( \Jr(\ocpi,f) - \Jr(\ppi,f) \big)
    \quad \st \quad 
    \Jc(\ppi,f) \leq d, \ \forall n \in \{1,\dots,N\}.
\end{equation}
The crux of safe exploration is that unlike performance, where suboptimality during learning is tolerable, safety must be maintained throughout all \(n\) episodes.

\paragraph{Probabilistic world models.}
Given full knowledge of the CMDP $\CM$, the agent can solve \Cref{eq:objective2} with no regret. In practice however, the agent has only limited knowledge about $\CM$. We capture this uncertainty by defining a \emph{set of plausible models} of the dynamics $\model_n$ for each iteration 
$n$.
\begin{definition}\label{def:callib}
We define the set of plausible models in episode \(n'\) as \(\smash{\model_{n'} \eqdef \{\tilde f\mid |\tilde f-\mu_{n'}|\leq \beta_{n'} \sigma_{n'}\}}\), described by a nominal model \(\smash{\mu_{n'}: \set{S}\times \set{A} \to \R^{d_\set{S}}}\) and the uncertainty \(\smash{\sigma_{n'}: \set{S}\times\set{A} \to \R^{d_\set{S}}}\) given data \(\D_{1:n'}\). This implies the model of \(f\) is always well-calibrated \citep{HURCL}, if there exist \(\smash{\forall n' \leq n: \beta_{n'} \in \R_{>0}}\) such that, with probability of at least \(1-\delta\) it holds \(\forall (s,a) \in \set{S}\times\set{A}\) that \(f(s,a) \in \model_n := \bigcap_{n'=0}^n \model_{n'}\).
\end{definition}
\Cref{def:callib} formalizes the agent's \emph{epistemic} uncertainty about $\CM$. At each iteration $n$, the agent maintains a confidence set that is statistically consistent with the observed data. This is done by learning a model with nominal dynamics \(\mu_n:\set{S}\times\set{A}\to\R^{d_{\set{S}}}\) and epistemic uncertainty \(\sigma_n:\set{S}\times\set{A}\to\R^{d_{\set{S}}}\) for every \(n=1,\dots,N\), using transitions from all previously collected trajectories \(\D_{1:n}\). This type of statistical models can be learned in practice via various (approximate) Bayesian inference methods~\citep[e.g., ][]{GP,lakshminarayanan2017simple,liu2016stein}. \looseness=-1Additionally, calibration can be achieved empirically with post-hoc recalibration techniques \citep{pmlr-v80-kuleshov18a}.

\paragraph{Pessimistic policy priors.} 
\looseness=-1
The set $\model_0$ has a unique interpretation: it formally represents the agent's prior knowledge about the environment.
For instance, in fully data-driven settings, given an offline dataset $\D_0$, a model $\model_0$ can be learned from $\D_0$ and then used to construct a conservative MDP that penalizes the agent when it goes outside of the data distribution \citep{levine2020offline,yu2020mopo}. Alternatively, in simulator-driven workflows, one has access to a nominal model \(\mu_0:\set{S}\times\set{A} \to \R^{d_{\set{S}}} \), and a bounded `sim-to-real gap' \(\smash{u:\set{S}\times\set{A}\to \R^{d_\set{S}}}\), giving \(\smash{f\in\model_0 = \{\tilde f \mid |\tilde f-\mu_0|\leq u\}}\). Probabilistic models of this sort formalize the conservatism needed for safe prior policies.
\begin{assumption}[Safe policy prior]\label{assume:safeGap}
    We assume access to a pessimistic policy prior \(\bpi(a|s)\), of a Gaussian class with variance \(\sigma_a^2 > 0\), truncated to \(\set{A}\) with location parameter \(\hat \mu(s) \in \set{A}\), that satisfies the constraint in \Cref{eq:objective} for all \(f \in \model_0 \) starting from \(s_0 \sim \rho_0\), such that under the true dynamics \(\smash{\Vc(s) \leq V_c^{\ocpi}(s)}\) for all states with probability at least \(1-\delta\).
\end{assumption}
Learning such ``worst-case'' policies builds on the principle of pessimism in the face of uncertainty and has been extensively studied in the robust RL literature \citep{sota-rcmdp,zhangdistributionally,shi2022pessimistic,shi2024distributionally}, offline safe RL \citep{jin2021pessimism,liu2023datasets,zheng2024safe,wachi2025provable}, safe learning from demonstrations ~\citep{schlaginhaufen2023identifiability,lindner2024learning} and when transferring from simulators \citep{ramu,as2025spidrsimpleapproachzeroshot}. In our experiments in \Cref{sec:experiments}, we leverage a few of these methods and demonstrate the effectiveness of combining them with safe online learning.
\section{{\sf SOOPER}: \textbf{S}afe \textbf{O}nline \textbf{O}ptimism for \textbf{P}essimistic \textbf{E}xpansion in \textbf{R}L}
In a nutshell, \algname{SOOPER} can be explained by two modes of operation:
\begin{enumerate*}[label=\textbf{(\roman*)}]
\item safe online data collection by pessimistically invoking a safe prior policy early enough so as to ensure constraint satisfaction;
\item optimistic planning in ``simulation'' via model-based rollouts, driving exploration of the unknown dynamics through intrinsic rewards.
\end{enumerate*}

\paragraph{Quantifying pessimism.}
We define the pessimistic cost value function \(\pVc\) for the policy prior \(\bpi\): \begin{equation}\label{eq:pessValue}
     \pVc(s) \eqdef \max_{\tilde{f} \in \model_n} \E_{\bpi} \left[\sum_{t=0}^\infty \gamma^t c(s_t,a_t) ~\bigg|~ s_0=s\right], \quad s_{t+1} = \tilde f(s_t,a_t) + \omega_t.
\end{equation}
By definition, \Cref{eq:pessValue} considers the worst-case model $\tilde{f}$ in the well-calibrated set $\model_n$ (per \Cref{def:callib}), which implies that \(\pVc(s)\) upper-bounds \(\Vc(s)\) for all \(s\in\set{S}\) with probability \(1-\delta\).
We later formally show that, as the number of episodes \(n\) increases, the agent refines its model \(\model_n\) and tightens the pessimistic upper-confidence bound \(\smash{\pVc}\) on the expected accumulated cost.

\paragraph{A tractable upper bound.}
Computing the pessimistic cost-value \(\pVc\) in \Cref{eq:pessValue} is intractable due to the maximization over $\model_n$. Instead, we upper-bound $\pVc$ by augmenting the cost function with an uncertainty penalty such that
\begin{align}
    \label{eq:pessimistic-cost-value-approx}
    \begin{aligned}
    &\Qcn(s, a) \eqdef \E_{\bpi} \left[\sum_{t=0}^\infty\gamma^t (c(s_t,a_t) + \lambda_{\text{pessimism}} \|\sigma_n(s_t,a_t)\|) ~\bigg|~ s_0=s, a_0=a\right], \\
    &\pVc(s) \leq \Vcn(s) = \E_{a \sim \bpi} \left[\Qcn(s, a)\right], ~ s_{t+1} = \mu_n(s_t, a_t) + \omega_t.
    \end{aligned}
\end{align}
Crucially, directly penalizing the cost function avoids finding the worst-case model within $\model_n$, and lends itself gracefully to existing temporal difference techniques for value function learning. The derivation of \(\lambda_{\text{pessimism}}\) and the proof for this upper bound are deferred to \Cref{slemma:intrinsicBound} in \Cref{sec:inter}.

\paragraph{Online cost tracking.} 
\algname{SOOPER} operates in an online fashion, where for each of the trajectory rollouts, it uses $\Qcn$ to predict when \(a_t \sim \ppi(\cdot|s_t)\) is potentially unsafe. Specifically, this may occur if the sum of the accumulated cost until \(t\), i.e., $\smash{c_{<t} = \sum_{\tau=0}^{t-1} \gamma^\tau c(s_\tau,a_\tau)}$ and the pessimistic cost value \(\gamma^t\Qcn(s_t, a_t)\) exceeds the safety budget \(d\). We exploit this idea in our design of \Cref{algo:safeRollout}. In \Cref{theo:safety} we show that this design guarantees safety throughout all episodes $n = 1, \dots, N$.
\begin{algorithm}
\caption{Safe Rollout via Online Cost Tracking}\label{algo:safeRollout}
    \begin{algorithmic}[1]
        \Require Pessimistic cost value $\Qcn$, policy \(\ppi\), safe policy prior \(\bpi\), and safety budget \(d\)
        \State Initialize accumulated cost \(c_{<0} = 0\) and experience buffer \(\D_n = \emptyset\)
        \For{step \(t = 1, \dots, T_n\)}
            \State Execute \(a_t \sim \sppi(a_t|s_t,c_{<t},\Qcn)\), observe $s_t,a_t,s_{t + 1}$ \Comment{\Cref{theo:safety}}  \label{line:shield_safe}
            \State Append the tuple \((s_{t},a_{t},s_{t+1})\) to \(\D_n\)
            \State Update accumulated cost: \(c_{<t+1} = c_{<t} + \gamma^t c_t\)
        \EndFor
        \State \Return \(\D_n\)
    \end{algorithmic}
\end{algorithm}
\begin{restatable}[Safety guarantee]{theorem}{sthrm} \label{theo:safety}
    Suppose \Cref{assume:noise,assume:regularity,assume:safeGap,assume:RKHS} hold and \(\model_n\) is well-calibrated \(\forall n = 1,\dots,N\) according to \Cref{def:callib}. If actions are executed for all timesteps \(t\) according to
    \begin{equation*}
        \sppi(a_t|s_t, \ct, \Qcn) \eqdef \begin{cases}
            \pi_n(\cdot|s_t) &  \text{if }\, \Phi(s_t,a_t,\ct, \Qcn) < d\\
            \bpi(\cdot|s_t) & \text{otherwise,}
        \end{cases}
    \end{equation*}
    where \(\Phi\) is the discounted sum of the realized accumulated cost \(\smash{\ct \eqdef \sum_{\tau=0}^{t-1} \gamma^\tau c(s_\tau,a_\tau)}\) until \(t - 1\) and the pessimistic cost value $\smash{\Qcn}$ such that $\smash{\Phi(s_t,a_t,\ct,\Qcn) \eqdef \ct + \gamma^t\Qcn(s_t,a_t)}$
    with factor \(\lambda_{\text{pessimism}} = \bar C\frac{\gamma}{1-\gamma}\frac{\left(1+\sqrt{d_{\set{S}}}\right)\beta_N(\delta)}{\sigma}\), where \(\bar C =\max\{C_{\text{max}},k_{\text{max}}\}\). Then, \Cref{algo:safeRollout} satisfies the safety constraint with probability \(1-\delta\) on \(f\) for every episode \(n=1,\dots,N\).
\end{restatable}
We refer to \Cref{proof:sthrm} for the formal proof. Monitoring costs online allows the agent to execute $\ppi$ until $\bpi$ is required. This enables the agent to visit state-actions that $\bpi$ would avoid, since those would be deemed unsafe by $\pVc(s_0)$.

\paragraph{Simulated exploration.}
While \Cref{theo:safety} guarantees safety during \emph{online} rollouts, the policy $\ppi$ may explore potentially unsafe behaviors by using simulated rollouts on the learned world model.
This allows \algname{SOOPER} to train \(\ppi\) on an \emph{unconstrained} ``planning MDP'' \(\SM\), that differs from the real CMDP $\CM$ in that it reaches a terminal state \(s_\dagger\) after those state-actions that would trigger $\bpi$ during online deployment:
\begin{equation} \label{eq:terminalTransition}
    \tilde{p}(s_{t+1}\mid s_t,a_t) \eqdef \begin{cases} 
        \ones\{s_{t+1} = s_\dagger\}  & \text{if } \Phi(a_t,s_t,c_{<t}, \Qcn) \geq d \text{ or } s_t = s_\dagger, \\
        p(s_{t+1}\mid s_t,a_t) & \text{otherwise}.
    \end{cases} 
\end{equation}
The resulting termination signal encourages the agent to find policies that outperform $\bpi$ if possible. This is done by assigning the terminal reward on $\SM$ to the pessimistic value \(\pVr\) of $\bpi$ upon termination:
\begin{equation} \label{eq:terminalReward}
    \tilde{r}(s_t,a_t) \eqdef \begin{cases}
        \pVr(s_t) & \text{if } \Phi(a_t,s_t,c_{<t},\Qcn) \geq d,\\
        0 & \text{if } s_t =  s_\dagger, \\
        r(s_t,a_t) & \text{otherwise},
    \end{cases}
\end{equation}
whereby,
\begin{equation}\label{eq:pessValueR}
    \pVr(s) \eqdef \min_{\tilde f \in \model_n} \E_{\bpi}\left[\sum_{t=0}^\infty \gamma^tr(s_t,a_t) ~\bigg|~ s_0 = s\right], \quad s_{t+1} = \tilde f(s_t,a_t) + \omega_t.
\end{equation}
This choice of the terminal reward keeps the reward structure consistent with the rewards encountered during online rollouts of \(\bpi\) on \(f\), unlike prior works that treat the terminal reward as a hyperparameter~\citep{thomas2021safe,Wagener2021SafeRL}. \Cref{fig:planning-mdp} illustrates a concrete example of how this design encourages the agent to learn to outperform the prior policy.
\definecolor{sooper}{HTML}{5F4690}
\definecolor{crpo}{HTML}{38A6A5}
\begin{figure}
\vspace{-0.5cm}
\centering
    \begin{tikzpicture}
        \node[anchor=south west,inner sep=0] (image) at (0,0) {
            \includegraphics[trim=0.87cm 0.78cm 1.51cm 0.95cm, clip]{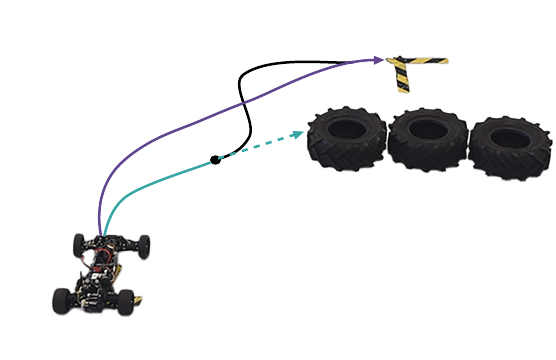}
        };
        \begin{scope}[x={(image.south east)},y={(image.north west)}]
            \node at (0.52, 0.49) {\textcolor{black}{$\Phi(s_t,\dots,Q^{\hat \pi}_{c,n})>d$}};
            \node at (0.41, 0.66) {\textcolor{black}{$s_\dagger$}};
            \node at (0.16, 0.73) {\textcolor{sooper}{$V_r^{\pi_N}(s_0)$}};
            \node at (0.29, 0.38) {\textcolor{crpo}{$V_r^{\ppi}(s_0)$}};
            \node at (0.37, 0.91) {\textcolor{black}{$\underline{V}^{\bpi}_r(s_t)$}};
            \node[font=\bfseries] at (0.45, 1.15) {
                $
                \begin{aligned}
                \textcolor{crpo}{\E_{\ppi}\!\left[\sum_{\tau=0}^{t - 1} \gamma^\tau   r(s_\tau,a_\tau)\right]} +
                \gamma^t\textcolor{black}{\pVr(s_t)}
                \leq \textcolor{sooper}{V_r^{\pi_N}(s_0)}
                \end{aligned}
                $
            };
        \end{scope}
        \node at (10.5cm,2.6cm) {\parbox{{0.45\textwidth}}{\centering
        \captionof{figure}{
        The agent's goal is to reach the cross marker while avoiding obstacles (tires). The agent deploys a policy $\ppi$, however at time $t$, it switches to the prior policy $\bpi$, to maintain the safety criterion of \Cref{algo:safeRollout}.
        The prior policy $\bpi$ ensures safety by following a conservative route, but sacrifices performance, resulting in a lower return $\underline{V}^{\bpi}_r(s_t)$. The trajectory of iteration $n$ is recorded to improve models of subsequent iterations. After $N$ iterations, as more data is gathered, the agent learns a more rewarding trajectory via model-generated rollouts in $\SM$. \href{https://yardenas.github.io/sooper/}{See provided video.}
        }
        \vspace{-0.5cm}
        \label{fig:planning-mdp}
        }};
    \end{tikzpicture}
\end{figure}
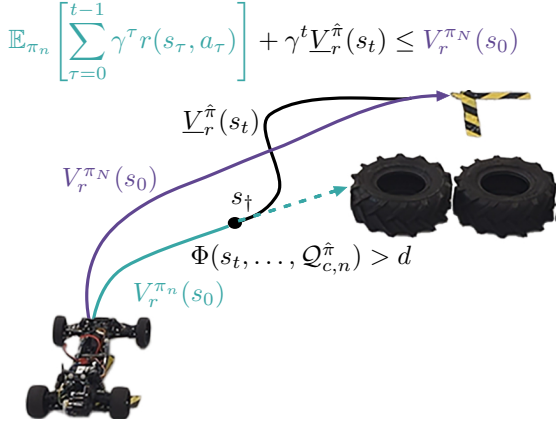

\paragraph{Exploration-exploitation and expansion.}
\looseness=-1
Probabilistic models have proven particularly effective in the unconstrained setting. Optimistic planning with these models enables a principled balance between exploration and exploitation, forming the basis of the celebrated UCRL algorithm \citep{NIPS2006_c1b70d96,HURCL}.
However, when safety must hold at every iteration $n$, exploration-exploitation is restricted to the set of policies that are safe with high probability. 
In particular, this set consists only those policies that satisfy the constraint under the \emph{worst-case} model in $\model_n$. Crucially, due to pessimism, an optimal policy $\ocpi$ may not lie in this set initially, as shown in \Cref{fig:expansion}. Nonetheless, this set can be \emph{expanded} by updating the model using the collected measurements from the environment, even in those regions where rewards are predicted to be low under the agent's model at iteration $n$. Prior works typically follow an expand-then-explore-exploit strategy \citep{pmlr-v37-sui15,Berkenkamp2017SafeMR,bura2022dope,as2025actsafe}, which begins by learning the dynamics with reward-free trajectory sampling to expand the safe set, followed by exploration-exploitation to find an optimal and safe policy.
This approach has two limitations:
\begin{enumerate*}[label=\textbf{(\roman*)}]
    \item it requires agents to first learn an auxiliary policy unrelated to the task, leading to wasted compute and exploration;
    \item it restricts the theoretical analysis to simple regret, allowing arbitrarily poor performance during exploration as long as the final policy is near-optimal.
\end{enumerate*}
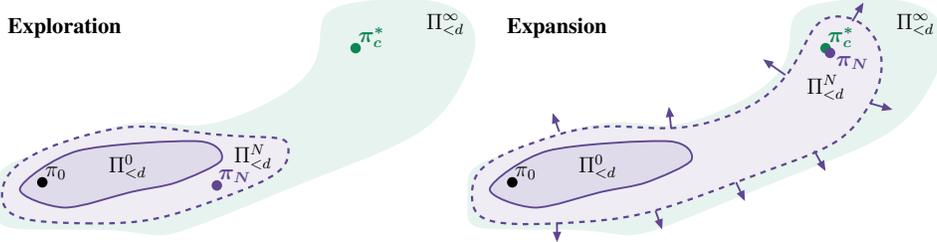
\begin{figure}
    \centering
    \scalebox{0.7}{
\definecolor{color1}{HTML}{5F4690}
\definecolor{color2}{HTML}{5F4690}
\definecolor{color3}{HTML}{5F4690}
\definecolor{color4}{HTML}{5F4690}
\definecolor{reward1}{HTML}{EDAD08} 
\definecolor{reward2}{HTML}{CC503E} 
\definecolor{reward3}{HTML}{0F8554} 

\begin{tikzpicture}[yscale=0.6, xscale=0.85]

\useasboundingbox (0.5,-1.5) rectangle (21.5,6.5);

\pgfmathsetmacro{\xoffset}{10.5}
\tikzset{every node/.style={font=\large}}

\fill[reward3, fill opacity=0.1]
  plot[smooth cycle, tension=1.3] coordinates {
    (0,0) (2,-1) (6,0) (10.5,4) (8,6.2) (6,3) (2,2) 
  };

\fill[white]
  plot[smooth cycle, tension=1] coordinates {
    (0.1,0.1) (2,-0.8) (5,0) (6.5,1.5) (5,2) (2,2)
  };
\draw[color3, line width=1.2pt, dashed, fill=color3, fill opacity=0.1]
  plot[smooth cycle, tension=1] coordinates {
    (0.1,0.1) (2,-0.8) (5,0) (6.5,1.5) (5,2) (2,2)
  };

\draw[color1, line width=1pt, fill=color1, fill opacity=0.1]
  plot[smooth cycle, tension=0.5] coordinates {
    (0.5,0.3) (1.5,-0.3) (3,0.2) (4,0.5) (5,1.2) (4.5,1.8) (1.8,1.5)
  };




\node[circle, fill=reward3, inner sep=2pt] at (8,4.75) {};
\node[color=reward3] at (8.35,5.1) {$\boldsymbol{\pi_c^*}$};

\node[circle, fill=color3, inner sep=2pt] at (4.9,0.4) {};
\node[color=color3] at (5.3,0.7) {$\boldsymbol{\pi_N}$};

\node at (10,5.5) {$\Pi_{<d}^\infty$};
\node at (5.7,1.3) {$\Pi_{<d}^N$};
\node at (2.9,1) {$\Pi_{<d}^0$};

\node[font=\large] at (1.5,5.4) {\textbf{Exploration}};

\node[circle, fill=black, inner sep=2pt] at (1,0.5) {};
\node at (1.3,0.8) {$\pi_0$};


\fill[reward3, fill opacity=0.1]
  plot[smooth cycle, tension=1.3] coordinates {
    (\xoffset,0) (\xoffset+2,-1) (\xoffset+6,0) (\xoffset+10.5,4) (\xoffset+8,6.2) (\xoffset+6,3) (\xoffset+2,2) 
  };

\fill[white]
  plot[smooth cycle, tension=1] coordinates {
    (\xoffset+0.1,0.1) (\xoffset+2,-0.8) (\xoffset+6,0.5) (\xoffset+9,3) (\xoffset+8.2,5.75) (\xoffset+6,2.7) (\xoffset+2,2)
  };
\draw[color3, line width=1.2pt, dashed, fill=color3, fill opacity=0.1]
  plot[smooth cycle, tension=1] coordinates {
    (\xoffset+0.1,0.1) (\xoffset+2,-0.8) (\xoffset+6,0.5) (\xoffset+9,3) (\xoffset+8.2,5.75) (\xoffset+6,2.7) (\xoffset+2,2)
  };

\draw[color1, line width=1pt, fill=color1, fill opacity=0.1]
  plot[smooth cycle, tension=0.5] coordinates {
    (\xoffset+0.5,0.3) (\xoffset+1.5,-0.3) (\xoffset+3,0.2) (\xoffset+4,0.5) (\xoffset+5,1.2) (\xoffset+4.5,1.8) (\xoffset+1.8,1.5)
  };

\draw[color3, -{Triangle[fill=color3,length=2mm,width=1.5mm]}, line width=1pt] (\xoffset+2,-0.8) -- (\xoffset+2,-1.4);
\draw[color3, -{Triangle[fill=color3,length=2mm,width=1.5mm]}, line width=1pt] (\xoffset+4.2,-0.4)-- (\xoffset+4.35,-1);
\draw[color3, -{Triangle[fill=color3,length=2mm,width=1.5mm]}, line width=1pt] (\xoffset+6,0.5) -- (\xoffset+6.2,-0.1);
\draw[color3, -{Triangle[fill=color3,length=2mm,width=1.5mm]}, line width=1pt] (\xoffset+7.75,1.5) -- (\xoffset+8,0.85);
\draw[color3, -{Triangle[fill=color3,length=2mm,width=1.5mm]}, line width=1pt] (\xoffset+9,3) -- (\xoffset+9.5,2.8);
\draw[color3, -{Triangle[fill=color3,length=2mm,width=1.5mm]}, line width=1pt] (\xoffset+8.6,5.6) -- (\xoffset+8.85,6.2);
\draw[color3, -{Triangle[fill=color3,length=2mm,width=1.5mm]}, line width=1pt] (\xoffset+7.1,3.9) -- (\xoffset+6.6,4.4);
\draw[color3, -{Triangle[fill=color3,length=2mm,width=1.5mm]}, line width=1pt] (\xoffset+4.55,2.2) -- (\xoffset+4.5,2.9);
\draw[color3, -{Triangle[fill=color3,length=2mm,width=1.5mm]}, line width=1pt] (\xoffset+2.05,2.1) -- (\xoffset+1.9,2.7);




\node[circle, fill=reward3, inner sep=2pt] at (\xoffset+8,4.75) {};
\node[color=reward3] at (\xoffset+8.34,5.1) {$\boldsymbol{\pi_c^*}$};

\node[circle, fill=color3, inner sep=2pt] at (\xoffset+8.1,4.6) {};
\node[color=color3] at (\xoffset+8.6,4.3) {$\boldsymbol{\pi_N}$};

\node at (\xoffset+10,5.5) {$\Pi_{<d}^\infty$};
\node at (\xoffset+8,3.5) {$\Pi_{<d}^N$};
\node at (\xoffset+2.9,1) {$\Pi_{<d}^0$};

\node[circle, fill=black, inner sep=2pt] at (\xoffset+1,0.5) {};
\node at (\xoffset+1.3,0.8) {$\pi_0$};

\node[font=\large] at (\xoffset+2,5.4) {\textbf{Expansion}};

\end{tikzpicture}}
    \caption{We denote the implicit set of safe policies in iteration \(n\) by \(\Pi_{<d}^n\), based on the learned model \(\model_n\). \textbf{Left:} exploration-exploitation in constrained tasks may not find an optimal policy because search is limited to $\Pi_{<d}^n$. \textbf{Right:} expansion proactively enlarges the safe set and reaches the optimum.
    \vspace{-0.7cm} 
    }
    \label{fig:expansion}
\end{figure}
\paragraph{A unified objective.}
\looseness=-1
In contrast, \algname{SOOPER} addresses the exploration-exploitation-expansion dilemma through \emph{intrinsic rewards} \citep{sukhija2025optimism} with a single tractable objective:
\begin{align} \label{eq:optiPlanning}
\begin{aligned}
        \ppi \in & \arg\max_\pi \E_{\pi}\left[\sum_{t=0}^\infty \gamma^t(\tilde{r}(s_t,a_t) + (\lambda_{\text{explore}}+\lambda_{\text{expand}})\|\sigma_n(s_t,a_t)\|)\right],\\
    &~s_{t+1} = \mu_n(s_t,a_t ) + \omega_t, \quad a_t \sim \pi(\cdot | s_t), \quad s_0 \sim \rho_0(\cdot),
\end{aligned}
\end{align}
where \(\lambda_{\text{explore}}\) and \(\lambda_{\text{expand}}\) are weighting factors for the respective components of the intrinsic reward bonus, encouraging exploration and expansion. In \Cref{lemma:optPlanningBound} in \Cref{sec:regret-analysis}, we derive these quantities in closed form and prove that \Cref{eq:optiPlanning} indeed enables expansion. Importantly, since we only augment the reward with exploration bonuses, an approximate solution to \Cref{eq:optiPlanning} can be computed efficiently.

\paragraph{Converging to the optimum.}
By iteratively updating the model \(\model_n\) and planning according to \Cref{eq:optiPlanning} in each episode, \algname{SOOPER} achieves \emph{sublinear cumulative regret} and therefore converges for a sufficient number of episodes \(N\) in \Cref{algo:modelBasedRL} according to \Cref{theo:regret}.
\begin{algorithm}
\caption{Safe Online Optimism for Pessimistic Expansion in RL (\algname{SOOPER})} \label{algo:modelBasedRL}
    \begin{algorithmic}[1]
        \Require Initial dynamics set \(\model_0\), pessimistic cost value \(\Qcz\), policy prior \(\bpi\), safety budget \(d\)
        \State Initialize experience buffer \(\D\)
        \For{episode \(n = 1,\dots, N\)}
            \State Execute \Cref{algo:safeRollout} to collect $\D_n$ \Comment{On the real environment}
            \State Update \(\model_n \gets \D_{1:n}\) \label{line:mb_rl_update_model} \Comment{Refine model}
            \State Obtain \(\ppi, \Qcn\) via \Cref{eq:optiPlanning,eq:pessimistic-cost-value-approx}\Comment{Model-based exploration and planning}\label{line:mb_rl_determine}
        \EndFor\\
        \Return \(\ppi\)
    \end{algorithmic}
\end{algorithm}
\begin{restatable}[Sublinear cumulative regret]{theorem}{rthrm}\label{theo:regret}
Suppose  \Cref{assume:noise,assume:regularity,assume:safeGap,assume:RKHS} hold and the model \(\model_n\) is well calibrated according to \Cref{def:callib}. Then, \Cref{algo:modelBasedRL} guarantees with probability \(1-\delta\)
\begin{align*}
    R(N) &= \sum_{n=1}^N \big( \Jr(\ocpi,f) - \Jr(\sppi,f) \big) \leq \mathcal{O}\left(\Gamma_{N\log(N)}^{7/2}\sqrt{N}\right),
    \\
    \text{and} ~ & \Jc(\sppi,f) \leq d, \ \forall n \in \{1,\dots,N\},
\end{align*}
that is, the cumulative regret grows sub-linearly in \(N\), dependent on the maximal information gain \(\Gamma_{N\log (N)}\), while satisfying the constraint throughout learning by \Cref{theo:safety}.
\end{restatable}

\looseness=-1 
\paragraph{Proof sketch.} 
\looseness=-1
Previous works explicitly expand the safe set of policies via pure exploration. This restricts them to a simple regret analysis, since performance can be evaluated \emph{only after} the reward-free expansion phase. \algname{SOOPER} optimizes a single objective, thus expansion occurs \emph{implicitly, during learning the task}. This allows us to extend the analysis to cumulative regret, providing guarantees on the agent's performance throughout learning. We bound the cumulative regret in \Cref{eq:objective2} by decomposing the performance gap in each iteration $n$ into two terms (\Cref{lemma:regretDecomposition}).
\begin{enumerate*}[label=\textbf{(\roman*)}]
\item The first term corresponds to the regret on the planning MDP \(\SM\) (cf. \Cref{eq:terminalTransition,eq:terminalReward}). Specifically, we compare the performance difference between an optimal policy of $\SM$ and the learned policy $\ppi$. Since this setting is purely unconstrained, we bound this term following the regret analysis for unconstrained MDPs of \citet{KakadeInfo}.
\item Next, we analyze the performance gap due to invoking the prior policy $\bpi$~(\Cref{algo:safeRollout}). In particular, we study the performance of executing $\ocpi$, an optimal policy of the true CMDP $\CM$, however under the safety criterion of \Cref{algo:safeRollout}. Intuitively, due to the pessimism of $\Qcn$, instead of performing optimally with $\ocpi$, \Cref{algo:safeRollout} invokes $\bpi$, thus leading to suboptimal performance. The gap in performance between executing $\ocpi$ freely and under \Cref{algo:safeRollout} is bounded  by the model uncertainty at iteration $n$ (\Cref{lemma:fromPolicyToUncertainty,lemma:firstRegretTerm}). The key intuition is that since $\ocpi$ is in fact safe, as model uncertainty shrinks, \Cref{algo:safeRollout} needs to invoke $\bpi$ less, hence converging to optimal performance.
\end{enumerate*}
Finally, we show that model uncertainty shrinks as $n$ grows under \Cref{assume:RKHS} due to \citet{Chowdhury} and that the sum of the above two terms is upper-bounded by the objective in \Cref{eq:optiPlanning} in each episode $n$. The complete proof is detailed in \Cref{sec:regret-analysis}.

\paragraph{Practical implementation.}
\label{sec:practica-implementation}
\looseness=-1
\href{https://github.com/lasgroup/safe-learning}{Our practical implementation}
follows closely  \Cref{algo:modelBasedRL}. We instantiate a ``deep'' version of \Cref{algo:modelBasedRL} that scales to high-dimensional continuous control tasks, following a model-based actor-critic architecture similar to MBPO \citep{janner2019trust}. \Cref{line:mb_rl_update_model} is implemented via a probabilistic ensemble of neural-networks \citep{chua2018deep}. In our practical implementation we additionally learn the reward and cost functions (see \Cref{sec:extension-world-models}). 
The standard deviation across ensemble predictions is used to estimate the epistemic uncertainty $\sigma_n$ \citep{depeweg2018decomposition}. We found this approach to deliver reliable uncertainty estimates in practice (cf. \Cref{sec:experiments,sec:experiments-appendix}). In \Cref{line:mb_rl_determine}, we solve \Cref{eq:optiPlanning,eq:pessimistic-cost-value-approx} using model-generated rollouts and by performing a fixed number of actor-critic updates as described by \citet{janner2019trust}. 
Adapting MBPO to our setting is straightforward, as it only requires wrapping the model predictions within the MDP defined by $\SM$. This suggests that \algname{SOOPER} is readily extensible to future advancements in deep RL.
\section{Experiments}
\label{sec:experiments}
\looseness=-1We now present our empirical results, demonstrating \algname{SOOPER}'s safety and performance guarantees in practice. In the \Cref{sec:experiments-appendix}, we provide additional ablations, illustrating how each component of \algname{SOOPER} plays a key role in attaining good performance while maintaining safety throughout learning.
\paragraph{Setup.}
In all experiments, we use the model-based architecture described in \Cref{sec:practica-implementation}. We compare \algname{SOOPER} with the following baselines:
\begin{enumerate*}[label=\textbf{(\roman*)}]
    \item \algname{SAILR}~\citep{Wagener2021SafeRL}, a state-of-the-art RL algorithm for safe exploration. 
    \item \algname{CRPO}~\citep{Xu2020CRPOAN}, a CMDP solver without safety guarantees during learning.
    \item Similarly, \algname{Primal-Dual}~\citep{bertsekas2016nonlinear} is a constrained optimization algorithm, which is commonly used as a CMDP solver.
\end{enumerate*}
\algname{CRPO} and \algname{Primal-Dual} serve as natural baselines, as they represent the standard approaches practitioners would typically adopt when fine-tuning policies initialized from simulation or offline training.
Unless specified otherwise, we use a batch of 128 trajectories to obtain empirical estimates of the undiscounted cumulative rewards $\hat{J}_r(\pi)$ and costs $\hat{J}_c(\pi)$ and run each experiment for five random seeds. We slightly abuse notation, so that for \algname{SOOPER}, $\pi$ denotes the policy induced by \Cref{algo:safeRollout} and for \algname{SAILR} the policy induced by their ``advantage-based intervention''.
All baselines are initialized with the same task-specific prior policy. 

\paragraph{Safe transfer under dynamics mismatch.}
We study policies that are trained under a mismatch in the dynamics on five tasks from RWRL \citep{challengesrealworld}, SafetyGym \citep{Ray2019} and RaceCar~\citep{kabzan2020amz}. We provide additional details in \Cref{sec:safety-gym,sec:rwrl,sec:racecar} about how mismatches in the dynamics are implemented for each task. We refer to \Cref{sec:experiments-appendix} for additional details on how we train the pessimistic prior policies. We train prior policies under the shifted dynamics ($\mu_0$, cf. \Cref{sec:problem-setting}) and continue learning on the true dynamics. \Cref{fig:sim-to-sim} presents
the largest cumulative cost recorded on the true dynamics throughout training, along with the relative performance improvement at the end of training compared to the policy prior before online learning. As shown, \algname{SOOPER} satisfies the constraints throughout learning in all of the tasks. For each task, among all algorithms that satisfy the constraints, \algname{SOOPER} is consistently on par or outperforms the baselines. This demonstrates \algname{SOOPER}'s ability to transfer safely under distribution shifts.
\begin{figure}
    \centering
    \includegraphics{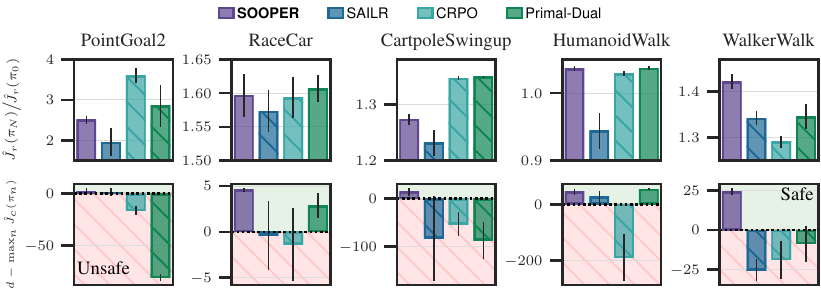}
    \caption{Performance improvement over the baseline policy and largest constraint recorded throughout learning. Among all methods, \algname{SOOPER} remains safe in all tasks while consistently outperforming or being on par with the other baselines when they satisfy the constraints.
    }
    \label{fig:sim-to-sim}
\end{figure}

\paragraph{Scaling to vision control.}
\begin{wrapfigure}[12]{r}{0.35\textwidth}
\vspace{-1.0\baselineskip}
  \centering
  \begin{tikzpicture}
    \node[anchor=south west, inner sep=0] (img) at (0,0)
      {\includegraphics[width=64pt, trim=0 4 4 0, clip]{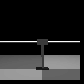}};
    \draw[red, thick] (22pt,26pt) -- (22pt,35pt);
    \draw[red, thick] (46pt,26pt) -- (46pt,35pt);
  \end{tikzpicture}
  \hspace{2pt}
  \begin{tikzpicture}
    \node[anchor=south west, inner sep=0] (img) at (0,0)
      {\includegraphics[width=64pt, trim=0 4 4 0, clip]{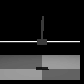}};
    \draw[red, thick] (22pt,26pt) -- (22pt,35pt);
    \draw[red, thick] (46pt,26pt) -- (46pt,35pt);
  \end{tikzpicture}
  \caption{\looseness=-1Image observations for the CartpoleSwingup task. The goal is to swing the pendulum to the upright position while avoiding the range outside the vertical red lines.}
  \label{fig:cartpole-pixels}
\end{wrapfigure}
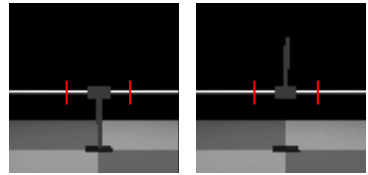
\looseness=-1Next, we demonstrate that \algname{SOOPER} scales to control tasks with image-based observations. Specifically, we use the CartpoleSwingup task from \citet{challengesrealworld} and train policies that operate on three temporally-stacked 64$\times$64 grayscale images as depicted in \Cref{fig:cartpole-pixels}. As in \Cref{fig:sim-to-sim}, we introduce a distribution shift in the dynamics and pretrain a vision policy under the perturbed dynamics using DrQ~\citep{yarats2021image}. We then train \algname{SOOPER} by learning the dynamics model directly on the embeddings produced by the pretrained DrQ vision encoder. \Cref{fig:vision} shows that even when trained on these embeddings, \algname{SOOPER} satisfies the safety constraints and achieves near-optimal performance.

\paragraph{Safe offline-to-online.}
We evaluate \algname{SOOPER} when initialized with a conservative policy using only offline data. To this end, we collect a fixed offline data of 2M transitions from the PointGoal1 task of SafetyGym. We train a pessimistic prior policy offline by following MOPO \citep{yu2020mopo}, a well-established model-based offline RL algorithm. We use a \algname{Primal-Dual} optimizer to solve the constraint optimization problem with this offline data. We found this approach to yield policies that satisfy the constraint while still delivering nontrivial performance. We use these policies to initialize all the baselines for our experiments. \Cref{fig:offline} presents the objective and constraint during online learning. Notably, \algname{SOOPER} outperforms all the baselines while maintaining safety.

\begin{figure}
    \centering
    \includegraphics{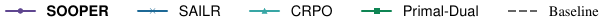}
    \begin{subfigure}[t]{0.49\textwidth}
        \centering
        \includegraphics{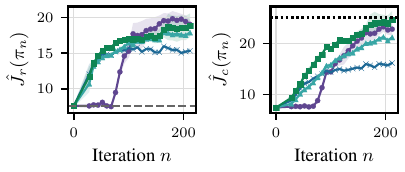}
        \caption{Safe online learning from an offline-trained policy.}
        \label{fig:offline}
    \end{subfigure}%
    \begin{subfigure}[t]{0.49\textwidth}
        \centering
        \includegraphics{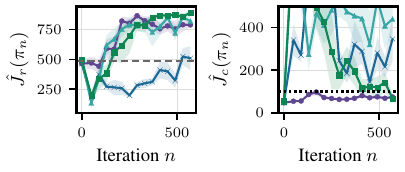}
        \caption{Vision control task.}
        \label{fig:vision}
    \end{subfigure}
    \caption{
    Learning curves of the objective and constraint when learning from offline and vision policies. 
    \algname{SOOPER} satisfies the constraints while significantly improving over the initial prior policy.
    }
\end{figure}

\paragraph{Safe online learning on real hardware.}
\looseness=-1
Finally, we evaluate \algname{SOOPER} when trained online \emph{directly on real-world robotic hardware}: a highly dynamic, remote-controlled race car operating at 60 Hz. The task is to reach a designated goal position located behind three obstacles (tires) that must be avoided~(illustrated in \Cref{fig:rccar-experiment}). This setting is particularly challenging due to stochasticity introduced by high-frequency control and delays in both actuation and motion-capture measurements. We initialize training with a prior policy derived from a first-principles simulator \citep{kabzan2020amz}, and then fine-tune it on the physical system using trajectories collected in real time. We repeat the experiment with five random seeds and report cumulative rewards and costs after each iteration in \Cref{fig:rccar-experiment}. As shown, after training, \algname{SOOPER} is roughly twice as good as the prior policy in terms of rewards while satisfying the safety constraint throughout learning. In \Cref{sec:experiments-appendix}, we extend this experiment to the offline setting, and report results when using a prior policy trained offline with data collected from the real system.
\newcommand{\elevation}{-10}
\newcommand{\anglerot}{77}
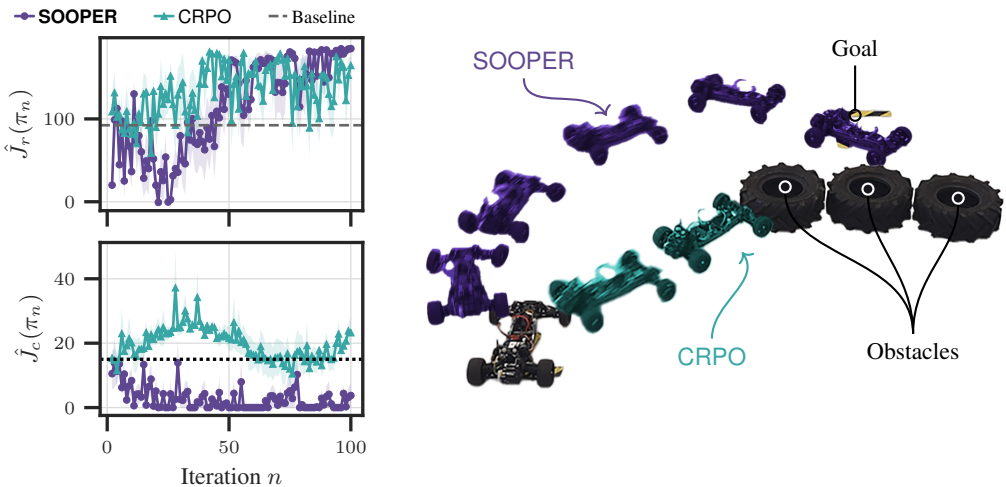
\begin{figure}
    \centering
    \begin{subfigure}[t]{0.36\textwidth}
       \vspace{0pt} 
        \centering
        \includegraphics{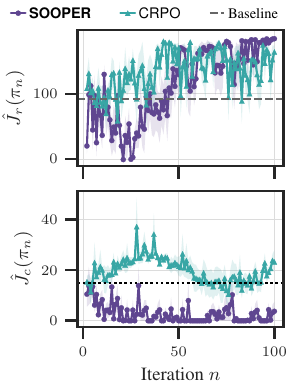} 
        \label{fig:hardware-results}
    \end{subfigure}%
    \hfill
    \begin{subfigure}[t]{0.63\textwidth}
    \vspace{12pt} 
    \centering
        \begin{tikzpicture}
            \node[anchor=south west, inner sep=0] (image2) at (0, 0) {\includegraphics[width=0.9\textwidth]{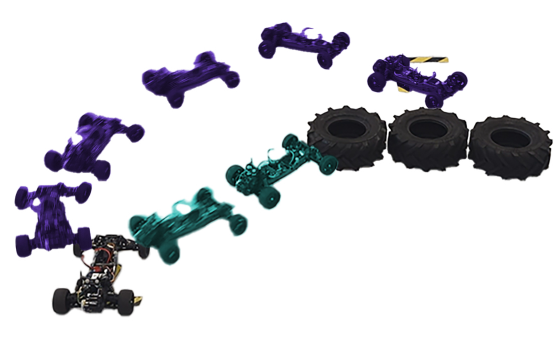}};
            \node[anchor=north east, fill=white, font=\small, text=black, rounded corners, xshift=-1.6cm, yshift=0.25cm] (goal) at (image2.north east) {Goal};
            \draw[-{Circle[open]}, thick] (goal.south) -- ++(0.0, -0.75);
            \node[anchor=south west, fill=white, font=\small, text=black, rounded corners, xshift=-2cm, yshift=-4.25cm] (obstacles) at (image2.north east) {Obstacles};
            \draw[-{Circle[open, color=white]}, thick] (obstacles.north) to[out=90, in=-90]  ++(-1.7, 2.04);
            \draw[-{Circle[open, color=white]}, thick] (obstacles.north) to[out=90, in=-90]  ++(-0.6, 2.01);
            \draw[-{Circle[open, color=white]}, thick] (obstacles.north) to[out=90, in=-90]  ++(0.6, 1.89);
            \node[fill=white, font=\small, text=sooper, rounded corners] (sooper) at (1.5,5) {\algname{SOOPER}};
            \draw[->, thick, sooper] (sooper.south)  to[out=-90, in=120] ++(1.1,-0.5);
            \node[fill=white, font=\small, text=crpo, rounded corners] (crpo) at (4.0,1.05) {\algname{CRPO}};
            \draw[->, thick, crpo] (crpo.north) to[out=90,in=-55] ++(0.4,1.15);
        \end{tikzpicture}
    \label{fig:rccar-demo}
    \end{subfigure}
    \vspace{-0.3cm}
    \caption{Safe exploration on real hardware with \algname{SOOPER}. We report the mean and standard error across five seeds of the objective and constraint measured on the real system. \algname{SOOPER} learns to improve over the prior policy while satisfying the constraints throughout learning. \href{https://anonymous.4open.science/r/sooper/docs/videos/SOOPER/timelapse.mp4}{Video of training.}
    }
    \vspace{-0.25cm}
    \label{fig:rccar-experiment}
\end{figure}
\section{Conclusion}
\looseness=-1
In this work, we address a critical challenge in systems that learn autonomously, namely, maintaining safety during learning. We present \algname{SOOPER}, a novel algorithm that provably improves performance over policies trained in simulation or offline, while satisfying safety constraints throughout learning. \algname{SOOPER} is simple to implement, achieves state-of-the-art performance and consistently maintains safety when deployed in practice. In addition, we establish a new theoretical result on the cumulative regret, enabling a guarantee of good performance during learning and not only at the last iteration.
These contributions advance our understanding of safe exploration and make an important step towards deployable RL. Nonetheless, many important challenges remain. In terms of theory and methodology, extending our work to high-probability constraints over particular states and developing new methods for the non-episodic setting, where agents learn from a single trajectory without resets, are two challenging yet crucial problems. On the practical side, considering more complex tasks with more powerful models is a promising direction for future work.

\subsubsection*{Acknowledgments}
We would like to thank Armin Lederer, Bruce D. Lee and Dongho Kang for insightful discussions during the development of this project. We thank the anonymous reviewers for their valuable comments and suggestions.
This project has received funding from a grant of the Hasler foundation~(grant no. 21039) and the ETH AI Center.

\bibliography{bibliography}

@inproceedings{Wagener2021SafeRL,
  title={Safe Reinforcement Learning Using Advantage-Based Intervention},
  author={Nolan Wagener and Byron Boots and Ching-An Cheng},
  booktitle={International Conference on Machine Learning},
  year={2021},
}

@inproceedings{
bharadhwaj2021conservative,
title={Conservative Safety Critics for Exploration},
author={Homanga Bharadhwaj and Aviral Kumar and Nicholas Rhinehart and Sergey Levine and Florian Shkurti and Animesh Garg},
booktitle={International Conference on Learning Representations},
year={2021},
}

@article{Koller2018LearningBasedMP,
  title={Learning-Based Model Predictive Control for Safe Exploration},
  author={Torsten Koller and Felix Berkenkamp and Matteo Turchetta and Andreas Krause},
  journal={IEEE Conference on Decision and Control},
  year={2018},
}

@inproceedings{Xu2020CRPOAN,
  title={CRPO: A New Approach for Safe Reinforcement Learning with Convergence Guarantee},
  author={Tengyu Xu and Yingbin Liang and Guanghui Lan},
  booktitle={International Conference on Machine Learning},
  year={2020},
}

@inproceedings{
thomas2021safe,
title={Safe Reinforcement Learning by Imagining the Near Future},
author={Garrett Thomas and Yuping Luo and Tengyu Ma},
booktitle={International Conference on Neural Information Processing Systems},
year={2021},
}

@inproceedings{as2022constrained,
  title={Constrained Policy Optimization via Bayesian World Models},
  author={As, Yarden and Usmanova, Ilnura and Curi, Sebastian and Krause, Andreas},
  booktitle={International Conference on Learning Representations},
  year={2022},
}

@inproceedings{
as2025actsafe,
title={ActSafe: Active Exploration with Safety Constraints for Reinforcement Learning},
author={Yarden As and Bhavya Sukhija and Lenart Treven and Carmelo Sferrazza and Stelian Coros and Andreas Krause},
booktitle={International Conference on Learning Representations},
year={2025},
}

@inproceedings{Berkenkamp2017SafeMR,
  title={Safe Model-based Reinforcement Learning with Stability Guarantees},
  author={Felix Berkenkamp and Matteo Turchetta and Angela P. Schoellig and Andreas Krause},
  booktitle={International Conference on Neural Information Processing Systems},
  year={2017}
}

@inproceedings{HURCL,
author = {Curi, Sebastian and Berkenkamp, Felix and Krause, Andreas},
title = {Efficient model-based reinforcement learning through optimistic policy search and planning},
year = {2020},
booktitle = {International Conference on Neural Information Processing Systems},
}

@book{sutton,
author = {Sutton, Richard S. and Barto, Andrew G.},
title = {Reinforcement Learning: An Introduction},
year = {2015},
}

@article{RecoveryRL,
  author={Thananjeyan, Brijen and Balakrishna, Ashwin and Nair, Suraj and Luo, Michael and Srinivasan, Krishnan and Hwang, Minho and Gonzalez, Joseph E. and Ibarz, Julian and Finn, Chelsea and Goldberg, Ken},
  journal={IEEE Robotics and Automation Letters}, 
  title={Recovery RL: Safe Reinforcement Learning With Learned Recovery Zones}, 
  year={2021},
}

@article{Bhavya,
title = {GoSafeOpt: Scalable safe exploration for global optimization of dynamical systems},
journal = {Artificial Intelligence},
year = {2023},
author = {Bhavya Sukhija and Matteo Turchetta and David Lindner and Andreas Krause and Sebastian Trimpe and Dominik Baumann}
}

@inproceedings{sootla2022saute,
  title={Saute RL: Almost Surely Safe Reinforcement Learning Using State Augmentation},
  author={Sootla, Aivar and Cowen-Rivers, Alexander I and Jafferjee, Taher and Wang, Ziyan and Mguni, David H and Wang, Jun and Ammar, Haitham},
  booktitle={International Conference on Machine Learning},
  year={2022},
}

@book{altman,
	author = {Altman, E.},
	description = {bandit problems},
	publisher = {Chapman and Hall},
	title = {Constrained Markov Decision Processes},
	year = 1999
}

@book{GP,
    author = {Carl Edward Rasmussen and Christopher K. I. Wiliams},
    title = {Gaussian Processes for Machine Learning},
    year = {2006}
}

@InProceedings{pmlr-v37-sui15,
  title = 	 {Safe Exploration for Optimization with Gaussian Processes},
  author = 	 {Sui, Yanan and Gotovos, Alkis and Burdick, Joel and Krause, Andreas},
  booktitle = 	 {International Conference on Machine Learning},
  year = 	 {2015},
}

@article{Berkenkamp2016BayesianOW,
  title={Bayesian optimization with safety constraints: safe and automatic parameter tuning in robotics},
  author={Felix Berkenkamp and Andreas Krause and Angela P. Schoellig},
  journal={Machine Learning},
  year={2016},
}

@InProceedings{kirschner19a,
  title = 	 {Adaptive and Safe {B}ayesian Optimization in High Dimensions via One-Dimensional Subspaces},
  author =       {Kirschner, Johannes and Mutny, Mojmir and Hiller, Nicole and Ischebeck, Rasmus and Krause, Andreas},
  booktitle = 	 {International Conference on Machine Learning},
  year = 	 {2019},
}

@inproceedings{Chowdhury,
author = {Chowdhury, Sayak Ray and Gopalan, Aditya},
title = {On kernelized multi-armed bandits},
year = {2017},
booktitle = {International Conference on Machine Learning},
}

@ARTICLE{manish,
  author={Prajapat, Manish and Köhler, Johannes and Turchetta, Matteo and Krause, Andreas and Zeilinger, Melanie N.},
  journal={IEEE Transactions on Automatic Control}, 
  title={Safe Guaranteed Exploration for Non-Linear Systems}, 
  year={2025},
}

@inproceedings{cpo,
author = {Achiam, Joshua and Held, David and Tamar, Aviv and Abbeel, Pieter},
title = {Constrained policy optimization},
year = {2017},
booktitle = {International Conference on Machine Learning},
}

@inproceedings{
huang2024safedreamer,
title={SafeDreamer: Safe Reinforcement Learning with World Models},
author={Weidong Huang and Jiaming Ji and Chunhe Xia and Borong Zhang and Yaodong Yang},
booktitle={International Conference on Learning Representations},
year={2024},
}

@inproceedings{OptiActive,
author = {Sukhija, Bhavya and Treven, Lenart and Sancaktar, Cansu and Blaes, Sebastian and Coros, Stelian and Krause, Andreas},
title = {Optimistic active exploration of dynamical systems},
year = {2023},
booktitle = {International Conference on Neural Information Processing Systems}
}

@inproceedings{mbpo,
    author = {Janner, Michael and Fu, Justin and Zhang, Marvin and Levine, Sergey},
    title = {When to trust your model: model-based policy optimization},
    year = {2019},
    booktitle = {International Conference on Neural Information Processing Systems},
    }

@inproceedings{ramu,
author = {Queeney, James and Benosman, Mouhacine},
title = {Risk-averse model uncertainty for distributionally robust safe reinforcement learning},
year = {2023},
booktitle = {International Conference on Neural Information Processing Systems}
}

@misc{challengesrealworld,
      title={An empirical investigation of the challenges of real-world reinforcement learning}, 
      author={Gabriel Dulac-Arnold and Nir Levine and Daniel J. Mankowitz and Jerry Li and Cosmin Paduraru and Sven Gowal and Todd Hester},
      year={2021},
      eprint={2003.11881},
      archivePrefix={arXiv},
      primaryClass={cs.LG},
}

@inproceedings{KakadeInfo,
author = {Kakade, Sham and Krishnamurthy, Akshay and Lowrey, Kendall and Ohnishi, Motoya and Sun, Wen},
title = {Information theoretic regret bounds for online nonlinear control},
year = {2020},
booktitle = {International Conference on Neural Information Processing Systems},
}

@misc{amodei2016concrete,
	title={Concrete Problems in AI Safety},
	author={Dario Amodei and Chris Olah and Jacob Steinhardt and Paul Christiano and John Schulman and Dan Mané},
	year={2016},
	eprint={1606.06565},
	archivePrefix={arXiv},
	primaryClass={cs.AI}
}

@misc{legg2023system,
  author       = {Shane Legg},
  title        = {System two safety},
  howpublished = {The Alignment Workshop},
  year         = {2023},
}

@article{dalrymple2024towards,
  title={Towards guaranteed safe ai: A framework for ensuring robust and reliable ai systems},
  author={Dalrymple, David and Skalse, Joar and Bengio, Yoshua and Russell, Stuart and Tegmark, Max and Seshia, Sanjit and Omohundro, Steve and Szegedy, Christian and Goldhaber, Ben and Ammann, Nora and others},
  journal={arXiv preprint arXiv:2405.06624},
  year={2024}
}

@article{yu2020mopo,
  title={Mopo: Model-based offline policy optimization},
  author={Yu, Tianhe and Thomas, Garrett and Yu, Lantao and Ermon, Stefano and Zou, James Y and Levine, Sergey and Finn, Chelsea and Ma, Tengyu},
  journal={International Conference on Neural Information Processing Systems},
  year={2020}
}

@article{JMLR:v16:garcia15a,
	author  = {Javier Garc{{\'i}}a and Fern and o Fern{{\'a}}ndez},
	title   = {A Comprehensive Survey on Safe Reinforcement Learning},
	journal = {Journal of Machine Learning Research},
	year    = {2015},
}

@inproceedings{NEURIPS2023_5d4cd12e,
 author = {Wachi, Akifumi and Hashimoto, Wataru and Shen, Xun and Hashimoto, Kazumune},
 booktitle = {International Conference on Neural Information Processing Systems},
 title = {Safe Exploration in Reinforcement Learning: A Generalized Formulation and Algorithms},
 year = {2023}
}

@misc{ouyang2022traininglanguagemodelsfollow,
      title={Training language models to follow instructions with human feedback}, 
      author={Long Ouyang and Jeff Wu and Xu Jiang and Diogo Almeida and Carroll L. Wainwright and Pamela Mishkin and Chong Zhang and Sandhini Agarwal and Katarina Slama and Alex Ray and John Schulman and Jacob Hilton and Fraser Kelton and Luke Miller and Maddie Simens and Amanda Askell and Peter Welinder and Paul Christiano and Jan Leike and Ryan Lowe},
      year={2022},
      eprint={2203.02155},
      archivePrefix={arXiv},
      primaryClass={cs.CL},
}

@article{silver2017mastering,
  title={Mastering the game of {G}o without human knowledge},
  author={Silver, David and Schrittwieser, Julian and Simonyan, Karen and Antonoglou, Ioannis and Huang, Aja and Guez, Arthur and Hubert, Thomas and Baker, Lucas and Lai, Matthew and Bolton, Adrian and others},
  journal={Nature},
  year={2017},
  publisher={Nature Publishing Group}
}

@book{puterman2014markov,
  title={Markov decision processes: discrete stochastic dynamic programming},
  author={Puterman, Martin L},
  year={2014},
  publisher={John Wiley \& Sons}
}

@article{sun2021safe,
  title={Safe exploration by solving early terminated mdp},
  author={Sun, Hao and Xu, Ziping and Fang, Meng and Peng, Zhenghao and Guo, Jiadong and Dai, Bo and Zhou, Bolei},
  journal={arXiv preprint arXiv:2107.04200},
  year={2021}
}

@inproceedings{prajapat2024towards,
  title={Towards safe and tractable Gaussian process-based MPC: Efficient sampling within a sequential quadratic programming framework},
  author={Prajapat, Manish and Lahr, Amon and K{\"o}hler, Johannes and Krause, Andreas and Zeilinger, Melanie N},
  booktitle={IEEE Conference on Decision and Control},
  year={2024},
  organization={IEEE}
}

@inproceedings{curi2022safe,
  title={Safe reinforcement learning via confidence-based filters},
  author={Curi, Sebastian and Lederer, Armin and Hirche, Sandra and Krause, Andreas},
  booktitle={Conference on Decision and Control},
  year={2022},
  organization={IEEE}
}

@article{hewing2019cautious,
  title={Cautious model predictive control using gaussian process regression},
  author={Hewing, Lukas and Kabzan, Juraj and Zeilinger, Melanie N},
  journal={IEEE Transactions on Control Systems Technology},
  year={2019},
  publisher={IEEE}
}

@article{turchetta2016safe,
  title={Safe exploration in finite markov decision processes with gaussian processes},
  author={Turchetta, Matteo and Berkenkamp, Felix and Krause, Andreas},
  journal={International Conference on Neural Information Processing Systems},
  year={2016}
}

@article{prajapat2022near,
  title={Near-optimal multi-agent learning for safe coverage control},
  author={Prajapat, Manish and Turchetta, Matteo and Zeilinger, Melanie and Krause, Andreas},
  journal={International Conference on Neural Information Processing Systems},
  year={2022}
}

@article{prajapat2025finite,
  title={Finite-Sample-Based Reachability for Safe Control with Gaussian Process Dynamics},
  author={Prajapat, Manish and K{\"o}hler, Johannes and Lahr, Amon and Krause, Andreas and Zeilinger, Melanie N},
  journal={arXiv preprint arXiv:2505.07594},
  year={2025}
}

@article{Wachi_Sui_Yue_Ono_2018, title={Safe Exploration and Optimization of Constrained MDPs Using Gaussian Processes}, journal={AAAI}, author={Wachi, Akifumi and Sui, Yanan and Yue, Yisong and Ono, Masahiro}, year={2018} }

@inproceedings{wachi2020safereinforcementlearning,
  title={Safe reinforcement learning in constrained markov decision processes},
  author={Wachi, Akifumi and Sui, Yanan},
  booktitle={International Conference on Machine Learning},
  year={2020},}

@article{moldovan2012safe,
  title={Safe exploration in markov decision processes},
  author={Moldovan, Teodor Mihai and Abbeel, Pieter},
  journal={arXiv preprint arXiv:1205.4810},
  year={2012}
}

@article{bura2022dope,
  title={DOPE: Doubly optimistic and pessimistic exploration for safe reinforcement learning},
  author={Bura, Archana and HasanzadeZonuzy, Aria and Kalathil, Dileep and Shakkottai, Srinivas and Chamberland, Jean-Francois},
  journal={International Conference on Neural Information Processing Systems},
  year={2022}
}

@inproceedings{liu2020ipo,
  title={Ipo: Interior-point policy optimization under constraints},
  author={Liu, Yongshuai and Ding, Jiaxin and Liu, Xin},
  booktitle={AAAI},
  year={2020}
}

@inproceedings{ni2025safe,
  title={A Safe Exploration Approach to Constrained Markov Decision Processes},
  author={Ni, Tingting and Kamgarpour, Maryam},
  booktitle={International Conference on Artificial Intelligence and Statistics},
  year={2025},
}

@article{milosevic2024embedding,
  title={Embedding Safety into RL: A New Take on Trust Region Methods},
  author={Milosevic, Nikola and M{\"u}ller, Johannes and Scherf, Nico},
  journal={arXiv preprint arXiv:2411.02957},
  year={2024}
}

@article{JMLR:v25:22-0878,
  author  = {Ilnura Usmanova and Yarden As and Maryam Kamgarpour and Andreas Krause},
  title   = {Log Barriers for Safe Black-box Optimization with Application to Safe Reinforcement Learning},
  journal = {Journal of Machine Learning Research},
  year    = {2024},
}

@article{usmanova2025safe,
  title={Safe Primal-Dual Optimization with a Single Smooth Constraint},
  author={Usmanova, Ilnura and Levy, Kfir Yehuda},
  journal={arXiv preprint arXiv:2505.09349},
  year={2025}
}

@article{simao2021alwayssafe,
  title={Alwayssafe: Reinforcement learning without safety constraint violations during training},
  author={Sim{\~a}o, Thiago D and Jansen, Nils and Spaan, Matthijs TJ},
  year={2021},
  publisher={New York: ACM}
}

@article{srinivasan2020learning,
  title={Learning to be safe: Deep rl with a safety critic},
  author={Srinivasan, Krishnan and Eysenbach, Benjamin and Ha, Sehoon and Tan, Jie and Finn, Chelsea},
  journal={arXiv preprint arXiv:2010.14603},
  year={2020}
}

@inproceedings{
eysenbach2018leave,
title={Leave no Trace: Learning to Reset for Safe and Autonomous Reinforcement Learning},
author={Benjamin Eysenbach and Shixiang Gu and Julian Ibarz and Sergey Levine},
booktitle={International Conference on Learning Representations},
year={2018},
}

@article{chow2018lyapunov,
  title={A lyapunov-based approach to safe reinforcement learning},
  author={Chow, Yinlam and Nachum, Ofir and Duenez-Guzman, Edgar and Ghavamzadeh, Mohammad},
  journal={International Conference on Neural Information Processing Systems},
  year={2018}
}

@inproceedings{liu2022constrained,
  title={Constrained variational policy optimization for safe reinforcement learning},
  author={Liu, Zuxin and Cen, Zhepeng and Isenbaev, Vladislav and Liu, Wei and Wu, Steven and Li, Bo and Zhao, Ding},
  booktitle={International Conference on Machine Learning},
  year={2022},
}

@article{dalal2018safe,
  title={Safe exploration in continuous action spaces},
  author={Dalal, Gal and Dvijotham, Krishnamurthy and Vecerik, Matej and Hester, Todd and Paduraru, Cosmin and Tassa, Yuval},
  journal={arXiv preprint arXiv:1801.08757},
  year={2018}
}

@article{wabersich2021predictive,
  title={A predictive safety filter for learning-based control of constrained nonlinear dynamical systems},
  author={Wabersich, Kim Peter and Zeilinger, Melanie N},
  journal={Automatica},
  year={2021},
  publisher={Elsevier}
}

@article{Ray2019,
	author = {Ray, Alex and Achiam, Joshua and Amodei, Dario},
	title = {{Benchmarking Safe Exploration in Deep Reinforcement Learning}},
	year = {2019}
}

@book{bertsekas2016nonlinear,
  title={Nonlinear Programming},
  author={Bertsekas, D.},
  series={Athena scientific optimization and computation series},
  year={2016},
  publisher={Athena Scientific}
}

@article{kabzan2020amz,
  title={AMZ driverless: The full autonomous racing system},
  author={Kabzan, Juraj and Valls, Miguel I and Reijgwart, Victor JF and Hendrikx, Hubertus FC and Ehmke, Claas and Prajapat, Manish and B{\"u}hler, Andreas and Gosala, Nikhil and Gupta, Mehak and Sivanesan, Ramya and others},
  journal={Journal of Field Robotics},
  year={2020},
  publisher={Wiley Online Library}
}

@article{liu2023datasets,
  title={Datasets and benchmarks for offline safe reinforcement learning},
  author={Liu, Zuxin and Guo, Zijian and Lin, Haohong and Yao, Yihang and Zhu, Jiacheng and Cen, Zhepeng and Hu, Hanjiang and Yu, Wenhao and Zhang, Tingnan and Tan, Jie and others},
  journal={arXiv preprint arXiv:2306.09303},
  year={2023}
}

@inproceedings{
yarats2021image,
title={Image Augmentation Is All You Need: Regularizing Deep Reinforcement Learning from Pixels},
author={Denis Yarats and Ilya Kostrikov and Rob Fergus},
booktitle={International Conference on Learning Representations},
year={2021},
}

@article{levine2020offline,
  title={Offline reinforcement learning: Tutorial, review, and perspectives on open problems},
  author={Levine, Sergey and Kumar, Aviral and Tucker, George and Fu, Justin},
  journal={arXiv preprint arXiv:2005.01643},
  year={2020}
}

@article{wachi2025provable,
  title={A Provable Approach for End-to-End Safe Reinforcement Learning},
  author={Wachi, Akifumi and Miyaguchi, Kohei and Tanabe, Takumi and Sato, Rei and Akimoto, Youhei},
  journal={arXiv preprint arXiv:2505.21852},
  year={2025}
}

@inproceedings{schlaginhaufen2023identifiability,
  title={Identifiability and generalizability in constrained inverse reinforcement learning},
  author={Schlaginhaufen, Andreas and Kamgarpour, Maryam},
  booktitle={International Conference on Machine Learning},
  year={2023},
}

@article{gu2024review,
  title={A review of safe reinforcement learning: Methods, theories and applications},
  author={Gu, Shangding and Yang, Long and Du, Yali and Chen, Guang and Walter, Florian and Wang, Jun and Knoll, Alois},
  journal={IEEE Transactions on Pattern Analysis and Machine Intelligence},
  year={2024},
  publisher={IEEE}
}

@article{brunke2022safe,
  title={Safe learning in robotics: From learning-based control to safe reinforcement learning},
  author={Brunke, Lukas and Greeff, Melissa and Hall, Adam W and Yuan, Zhaocong and Zhou, Siqi and Panerati, Jacopo and Schoellig, Angela P},
  journal={Annual Review of Control, Robotics, and Autonomous Systems},
  year={2022},
  publisher={Annual Reviews}
}

@inproceedings{lindner2024learning,
  title={Learning safety constraints from demonstrations with unknown rewards},
  author={Lindner, David and Chen, Xin and Tschiatschek, Sebastian and Hofmann, Katja and Krause, Andreas},
  booktitle={International Conference on Artificial Intelligence and Statistics},
  year={2024},
}

@article{lakshminarayanan2017simple,
  title={Simple and scalable predictive uncertainty estimation using deep ensembles},
  author={Lakshminarayanan, Balaji and Pritzel, Alexander and Blundell, Charles},
  journal={International Conference on Neural Information Processing Systems},
  year={2017}
}

@article{liu2016stein,
  title={Stein variational gradient descent: A general purpose bayesian inference algorithm},
  author={Liu, Qiang and Wang, Dilin},
  journal={International Conference on Neural Information Processing Systems},
  year={2016}
}

@inproceedings{
zheng2024safe,
title={Safe Offline Reinforcement Learning with Feasibility-Guided Diffusion Model},
author={Yinan Zheng and Jianxiong Li and Dongjie Yu and Yujie Yang and Shengbo Eben Li and Xianyuan Zhan and Jingjing Liu},
booktitle={International Conference on Learning Representations},
year={2024},
}

@inproceedings{NIPS2006_c1b70d96,
 author = {Auer, Peter and Ortner, Ronald},
 booktitle = {International Conference on Neural Information Processing Systems},
 editor = {B. Sch\"{o}lkopf and J. Platt and T. Hoffman},
 title = {Logarithmic Online Regret Bounds for Undiscounted Reinforcement Learning},
 year = {2006}
}

@inproceedings{sukhija2025optimism,
  title={Optimism via intrinsic rewards: Scalable and principled exploration for model-based reinforcement learning},
  author={Sukhija, Bhavya and Treven, Lenart and Sferrazza, Carmelo and Dorfler, Florian and Abbeel, Pieter and Krause, Andreas},
  booktitle={7th Robot Learning Workshop: Towards Robots with Human-Level Abilities},
  year = {2025}
}

@article{tassa2018deepmind,
  title={Deepmind control suite},
  author={Tassa, Yuval and Doron, Yotam and Muldal, Alistair and Erez, Tom and Li, Yazhe and Casas, Diego de Las and Budden, David and Abdolmaleki, Abbas and Merel, Josh and Lefrancq, Andrew and others},
  journal={arXiv preprint arXiv:1801.00690},
  year={2018}
}

@article{janner2019trust,
  title={When to trust your model: Model-based policy optimization},
  author={Janner, Michael and Fu, Justin and Zhang, Marvin and Levine, Sergey},
  journal={International Conference on Neural Information Processing Systems},
  year={2019}
}

@article{chua2018deep,
  title={Deep reinforcement learning in a handful of trials using probabilistic dynamics models},
  author={Chua, Kurtland and Calandra, Roberto and McAllister, Rowan and Levine, Sergey},
  journal={International Conference on Neural Information Processing Systems},
  year={2018}
}

@inproceedings{depeweg2018decomposition,
  title={Decomposition of uncertainty in Bayesian deep learning for efficient and risk-sensitive learning},
  author={Depeweg, Stefan and Hernandez-Lobato, Jose-Miguel and Doshi-Velez, Finale and Udluft, Steffen},
  booktitle={International conference on machine learning},
  year={2018},
}

@article{sota-rcmdp,
    author = {Toshinori Kitamura and Tadashi Kozuno and Wataru Kumagai and Kenta Hoshino and Yohei Hosoe and Kazumi Kasaura and Masashi Hamaya and Paavo Parmas and Yutaka Matsuo},
    title = {Near-Optimal Policy Identification in Robust Constrained Markov Decision Processes via Epigraph Form},
    journal = {arXiv preprint arXiv:2408.16286},
    year = 2024
}

@inproceedings{zhangdistributionally,
  title={Distributionally Robust Constrained Reinforcement Learning under Strong Duality},
  author={Zhang, Zhengfei and Panaganti, Kishan and Shi, Laixi and Sui, Yanan and Wierman, Adam and Yue, Yisong},
  booktitle={Reinforcement Learning Conference}
}

@article{shi2024distributionally,
  title={Distributionally robust model-based offline reinforcement learning with near-optimal sample complexity},
  author={Shi, Laixi and Chi, Yuejie},
  journal={Journal of Machine Learning Research},
  year={2024}
}

@inproceedings{shi2022pessimistic,
  title={Pessimistic q-learning for offline reinforcement learning: Towards optimal sample complexity},
  author={Shi, Laixi and Li, Gen and Wei, Yuting and Chen, Yuxin and Chi, Yuejie},
  booktitle={International conference on machine learning},
}

@article{silver2025welcome,
  title={Welcome to the era of experience},
  author={Silver, David and Sutton, Richard S},
  journal={Google AI},
  year={2025}
}

@inproceedings{jin2021pessimism,
  title={Is pessimism provably efficient for offline rl?},
  author={Jin, Ying and Yang, Zhuoran and Wang, Zhaoran},
  booktitle={International conference on machine learning},
  year={2021},
}

@InProceedings{pmlr-v80-kuleshov18a,
  title = 	 {Accurate Uncertainties for Deep Learning Using Calibrated Regression},
  author =       {Kuleshov, Volodymyr and Fenner, Nathan and Ermon, Stefano},
  booktitle = 	 {International Conference on Machine Learning},
  year = 	 {2018},
}

@book{Lehman2017mathematics,
    author = {Lehman, Eric and Thomas, Leighton and Albert, Meyer},
    title = {Mathematics for Computer Science},
    year = {2018},
}

@inproceedings{
as2025spidrsimpleapproachzeroshot,
title={{SP}i{DR}: A Simple Approach for Zero-Shot Safety in Sim-to-Real Transfer},
author={Yarden As and Chengrui Qu and Benjamin Unger and Dongho Kang and Max van der Hart and Laixi Shi and Stelian Coros and Adam Wierman and Andreas Krause},
booktitle={International Conference on Neural Information Processing Systems},
year={2025},
}

@inproceedings{ding2021provably,
  title={Provably efficient safe exploration via primal-dual policy optimization},
  author={Ding, Dongsheng and Wei, Xiaohan and Yang, Zhuoran and Wang, Zhaoran and Jovanovic, Mihailo},
  booktitle={International conference on artificial intelligence and statistics},
  year={2021},
}

@article{vignola2026sampling,
  title={Sampling-Based Safe Reinforcement Learning},
  author={Vignola, Luca and Lee, Bruce D and Prajapat, Manish and Wendl, Manuel and Zeilinger, Melanie and Krause, Andreas and As, Yarden},
  journal={arXiv preprint arXiv:2605.19469},
  year={2026}
}
\bibliographystyle{iclr2026_conference}

\newpage
\appendix

\allowdisplaybreaks
\section{Safety Guarantee} \label{sec:inter}

First, we derive in \Cref{slemma:safetyPessValue} that the pessimistic safety-condition with online cost tracking (\Cref{theo:safety}) guarantees safety given the true pessimistic cost value function \(\pVc\) of \Cref{eq:pessValue}.

\begin{lemma}\label{slemma:safetyPessValue}
    Suppose \Cref{assume:noise,assume:regularity,assume:safeGap,assume:RKHS} hold and \(\model_n\)  is well-calibrated \(\forall n = 1,\dots,N\) according to \Cref{def:callib}. Given the safe prior \(\bpi\), the rollout of the policy \(\ppi\) satisfies the safety constraint with probability \(1-\delta\) on \(f\) for every episode \(n=1,\dots,N\), if the actions are executed for all timesteps \(t\) according to the pessimistic safety-condition:
    \begin{equation*}
        \tilde{\pi}_n(a_t|s_t,\ct,\model_n) = \begin{cases}
            \pi_n(\cdot|s_t) &  \text{if }\, \bar \Phi(s_t,a_t,\ct,\model_n) < d\\
            \bpi(\cdot | s_t) & \text{otherwise,}
        \end{cases}
    \end{equation*}
    \begin{align*}
        \text{where }\quad  \bar \Phi(s_t,a_t,\ct,\model_n) & = \ct + \gamma^t c(s_t,a_t) + \gamma^{t+1} \max_{\tilde{f} \in \model_n} \E_\omega\left[\bar{V}_c^{\bpi}(\tilde{f}(s_{t},a_{t}) + \omega_t )\right],
    \end{align*}
    indicates the sum of the accumulated costs \(\ct = \sum_{\tau=0}^{t-1} \gamma^\tau c(s_\tau,a_\tau)\) until \(t-1\), the immediate cost at \(t\), and the worst-case pessimistic cost value \(\pVc\) defined in \Cref{eq:pessimistic-cost-value-approx}.
\end{lemma}
\begin{proof}
    We prove the safe rollout of \(\ppi\) for the unknown \(f\) by induction. Throughout, \(\ct\) denotes the \emph{realized} accumulated cost of \Cref{algo:safeRollout}, and all statements at time \(t\) are conditioned on the executed history up to \(t\). Once the safe prior \(\bpi\) is invoked at some timestep, \Cref{algo:safeRollout} keeps executing \(\bpi\) for all remaining timesteps; hence the fallback branch is absorbing and, from its activation onward, the executed continuation cost is upper bounded by \(\pVc\).
\begin{enumerate}[leftmargin=0.5cm, parsep=0.3em]
    \item Base Case \((t=0)\):
    By definition of \(\tilde{\pi}_n\), we only execute \(a_0 \sim \ppi(\cdot | s_0)\) if \(\bar\Phi(s_0,a_0,c_{<0},\model_n) \leq d\) with \(c_{<0} = 0\) and otherwise the safe prior \(\bpi\).
    By definition of \(\bar \Phi\), either \(a_0\) is safe to execute with
    \begin{align}
        d & > c(s_0,a_0) + \gamma \max_{\tilde{f} \in \model_n} \E_{\omega}\left[\pVc(\tilde f(s_0,a_0)+\omega_0)\right],
    \end{align}
    and otherwise executing \(a_0\sim\bpi(\cdot|s_0)\) is safe since
    \(\E_{\hat a_0 \sim \bpi(\cdot|s_0)}\bigl[\bar\Phi(s_0,\hat a_0,c_{<0},\model_n)\bigr] = \pVc(s_0) \le d\)
    by \Cref{assume:safeGap} together with \(\model_n \subseteq \model_0\).
    \item Inductive Step:
    Given the agent is safe up to \(t-1\), conditioned on \(\mathcal{H}_{t-2}=s_0,a_0,\dots,s_{t-2},a_{t-2} \),
    \begin{equation}\label{eq:safePrevStep}
        \E_{\omega_{t-2}}\left[c_{< t-1} + \gamma^{t-1}c(s_{t-1},a_{t-1}) + \gamma^t \max_{\tilde{f} \in \model_n} \E_{\omega_{t-1}}\left[\pVc(\tilde f(s_{t-1},a_{t-1})+\omega_{t-1})\right]\Big|\mathcal{H}_{t-2}\right] < d,
    \end{equation}
    where the outer expectation averages over \(\omega_{t-2}\) (generating \(s_{t-1}\)) and the executed \(a_{t-1}\sim\tilde\pi_n(\cdot|s_{t-1})\), we show constraint satisfaction for time \(t\), namely
    \begin{align}\label{eq:safety}
        d &> \E_{\omega_{t-1}}\left[\ct + \gamma^t\, \E_{a_t}\!\left[c(s_t,a_t) +\gamma\,\E_{\omega_t}\left[\pVc(f(s_t,a_t)+\omega_t)\right]\right]\Big|\mathcal{H}_{t-1}\right],
    \end{align}
    in expectation over the executed action \(a_t \sim \tilde\pi_n\).
    Using \(\tilde{\pi}_n\), we either execute the sampled \(a_t \sim \ppi(\cdot|s_t)\), for which \Cref{eq:safety} holds by the test \(\bar\Phi(s_t,a_t,\ct,\model_n)<d\), since \(f \in \model_n\) on the calibration event of \Cref{def:callib}. Otherwise, we invoke \(\hat a_t \sim \bpi(\cdot | s_t)\) and, by the absorbing fallback, follow \(\bpi\) for all subsequent time steps, so that \(\pVc(s_t)\) upper bounds the continuation cost after \(t\).
    Using the realized cost \(c_{< t} = c_{< t-1} + \gamma^{t-1}c(s_{t-1},a_{t-1})\) of \Cref{algo:safeRollout}, we obtain
    \begin{align}
        & \E_{\omega_{t-1}}\left[ \ct + \gamma^{t}\,\E_{\hat a_t \sim \bpi(\cdot|s_t)}\!\left[c(s_{t},\hat a_{t}) + \gamma \max_{\tilde{f} \in \model_n} \E_{\omega_t}\!\left[\pVc(\tilde{f}(s_{t},\hat a_{t}) + \omega_{t})\right]\right]\Big|\mathcal{H}_{t-1}\right] \\
        & \overset{(i)}{=}\; \E_{\omega_{t-1}}\left[c_{< t-1} + \gamma^{t-1}c(s_{t-1},a_{t-1}) + \gamma^t\,\pVc(s_t)\,\Big|\,\mathcal{H}_{t-1}\right] \\
        & \overset{(ii)}{\leq}\;
      \underbrace{c_{< t-1} + \gamma^{t-1}c(s_{t-1},a_{t-1})
      + \gamma^t \max_{\tilde{f} \in \model_n}
        \E_{\omega_{t-1}}\!\left[\pVc(\tilde f(s_{t-1},a_{t-1})+\omega_{t-1})\right]}
      _{=\ \bar\Phi(s_{t-1},a_{t-1},c_{<t-1},\model_n)}.
    \end{align}
    Taking \(\E[\,\cdot\mid\mathcal H_{t-2}]\) of both \(\mathcal H_{t-1}\)-measurable sides and using the tower property \(\E\!\left[\E_{\omega_{t-1}}[\,\cdot\mid\mathcal H_{t-1}]\mid\mathcal H_{t-2}\right] = \E[\,\cdot\mid\mathcal H_{t-2}]\),
    \begin{align}
        & \E\!\left[\ct + \gamma^{t}\,\E_{\hat a_t \sim \bpi(\cdot|s_t)}\!\left[
            c(s_{t},\hat a_{t}) + \gamma \max_{\tilde{f} \in \model_n}
            \E_{\omega_t}\!\left[\pVc(\tilde f(s_{t},\hat a_{t}) + \omega_{t})\right]
          \right]\ \Big|\ \mathcal H_{t-2}\right] \\
        & \overset{(iii)}{\leq}\;
          \E\!\left[\bar\Phi(s_{t-1},a_{t-1},c_{<t-1},\model_n)\mid\mathcal H_{t-2}\right]
          \overset{(iv)}{<} d,
    \end{align}
    where \((iii)\) applies \(\E[\,\cdot\mid\mathcal H_{t-2}]\) to the monotone chain above, and \((iv)\) is the induction hypothesis \Cref{eq:safePrevStep}. Since the fallback is absorbing, once \(\bpi\) is activated the executed trajectory coincides with a pure \(\bpi\)-rollout. On the behavioral policy case, the test forces the potential to be below \(d\).
\end{enumerate}
The per-step bound \(\E[\bar\Phi(s_t,a_t,\ct,\model_n)\mid\mathcal H_{t-1}]<d\) certifies the cost of bailing out to the prior at any time \(t\). The trajectory constraint of \Cref{eq:objective} distinguishes two cases:

    \emph{Case 1 (prior activates at some \(t\)).} By the absorbing fallback,
    the executed trajectory follows \(\ppi\) up to \(t-1\) and \(\bpi\) thereafter, so its expected discounted cost is
    \begin{align}
        \E\!\left[\sum_{t=0}^\infty \gamma^t c(s_t,a_t)\right]
        = \E\!\left[c_{<t} + \gamma^{t} V_c^{\bpi,f}(s_{t})\right]
        \le \E\!\left[\bar\Phi(s_{t-1},a_{t-1},c_{<t-1},\model_n)\right] < d,
    \end{align}
    where \(V_c^{\bpi,f}\) is the continuation cost under the true \(f\), the first inequality uses the Bellman relation of \(\bar\Phi\) together with \(c_{<t}=c_{<t-1}+\gamma^{t-1}c(s_{t-1},a_{t-1})\),
    and the last is the invariant at \(t-1\).

    \emph{Case 2 (the prior never activates).} Then \(a_t\sim\ppi\), so \(\E[\bar\Phi(s_t,a_t,\ct,\model_n)\mid\mathcal H_{t-1}]<d\) holds for all \(t\). Since \(\bar\Phi(s_t,a_t,\ct,\model_n)\ge \E[c_{<t+1}\mid\mathcal H_{t-1}]\), with \(\E[c_{<t+1}]<d\) and monotone convergence directly follows \(J_c(\tilde \pi_n, f) \leq d \).

    In both cases the expected accumulated cost is bounded by \(d\), which is the
    constraint in \Cref{eq:objective}. Since all steps hold on the calibration event of \Cref{def:callib}, which occurs with probability \(1-\delta\), the safety constraint is satisfied with probability \(1-\delta\), completing the induction. \qedhere
\end{proof}
Next, we show that adding the cost penalties in \(\Qcn\) (\Cref{eq:pessimistic-cost-value-approx}) upper-bounds the true cost value, using \citet[Lemma~A.5]{sukhija2025optimism}:
\begin{lemma}\label{slemma:intrinsicBound}
    Suppose \Cref{assume:noise,assume:regularity,assume:RKHS} hold and \(\model_n\) is well-calibrated by \Cref{def:callib}. Given  
    \begin{align*}
        \gamma^{t+1} &\max_{\tilde{f} \in \model_n} \E_\omega\left[\bar{V}_c^{\bpi}(\tilde{f}(s_{t},a_{t}) + \omega_t )\right] =  \E_{\bpi}\left[\sum_{\tau=t+1}^\infty \gamma^\tau c(s_\tau,a_\tau)\right]\\
        & \st ~ s_{\tau+1} = \tilde f(s_\tau,a_\tau)+\omega_\tau, \quad a_\tau \sim \bpi(\cdot | s_\tau),\quad s_{t+1} = \tilde f(s_t,a_t) + \omega_t ,\\
        & \lambda_{\text{pessimism}} =\bar C\frac{\gamma}{1-\gamma}\frac{\left(1+\sqrt{d_{\set{S}}}\right)\beta_{N}(\delta)}{\sigma},
    \end{align*}
    then for all episodes \(n=1,\dots,N\) with probability of \(1-\delta\):
    \begin{align*}
        \gamma^{t+1} & \max_{\tilde{f} \in \model_n} \E_\omega\left[\bar{V}_c^{\bpi}(\tilde{f}(s_{t},a_{t}) + \omega_t )\right] \leq \gamma^t\lambda_{\text{pessimism}}\|\sigma_n(s_t,a_t)\| \\ & + \E_{\bpi}\left[\sum_{\tau=t+1}^\infty \gamma^\tau\left(c(s_\tau,a_\tau) + \lambda_{\text{pessimism}} \|\sigma_n(s_\tau,a_\tau)\|\right)\right],\\
        & \st ~s_{\tau+1} = \mu_n(s_\tau,a_\tau)+\omega_\tau, \quad a_\tau \sim \bpi(\cdot | s_\tau), \quad s_{t+1}=\mu_n(s_t,a_t) + \omega_t.
    \end{align*}
\end{lemma}
\begin{proof}
    Let us denote the value following the nominal dynamics \(\mu_n\) of \(\model_n\) as
    \begin{align}
    \begin{aligned}\label{eq:cValueNominal}
        \tVc(s) & = \E_{\bpi}\left[\sum_{t=0}^\infty \gamma^t c(s_t,a_t)\mid s_0=s\right],\\
        \st~ s_{t+1} & = \mu_n(s_t,a_t)+\omega_t, \quad a_t \sim \bpi(\cdot|s_t), 
    \end{aligned}
    \end{align}
    and the upper bound on the cost value following the worst dynamics \(\tilde f\in \model_n\) by
    \begin{align}
    \begin{aligned}
        \pVc(s) & = \max_{\tilde f \in \model_n}\Vc(s) = \E_{\bpi}\left[\sum_{t=0}^\infty \gamma^t c(s_t,a_t) \mid s_0=s\right],\\
        \st~ s_{t+1} & = \tilde f(s_t,a_t)+\omega_t, \quad a_t \sim \bpi(\cdot|s_t).
    \end{aligned}
    \end{align}
    Following \citet[Lemma~A.5]{sukhija2025optimism}, the difference between these values is bounded by
    \begin{align}
    \begin{aligned}
        \left|\pVc(s) - \tVc(s)\right| &\leq \lambda_{\text{pessimism}} \E_{\bpi}\left[\sum_{t=0}^\infty \gamma^t \|\sigma_n(s_t,a_t)\| \mid s_0=s\right],\\
        \st ~ s_{t+1} &=\mu_n(s_t,a_t) + \omega_t, \quad a_t \sim \bpi(\cdot|s_t).
    \end{aligned}
    \end{align}
    Bringing \(\tVc(s)\) to the other side and using the definition in \Cref{eq:cValueNominal} results in:
    \begin{align}\label{eq:upperBoundCVp1}
    \pVc(s) &\leq \tVc(s) + \lambda_{\text{pessimism}} \E_{\bpi}\left[\sum_{t=0}^\infty \gamma^t \|\sigma_n(s_t,a_t)\|\mid s_0=s \right],\\\label{eq:upperBoundCVp}
        &= \E_{\bpi}\left[\sum_{t=0}^\infty \gamma^t\left(c(s_t,a_t) + \lambda_{\text{pessimism}} \|\sigma_n(s_t,a_t)\|\right) \mid s_0=s \right],\\
        \st ~&  s_{t+1}=\mu_n(s_t,a_t) + \omega_t, \quad a_t \sim \bpi(\cdot|s_t).
    \end{align}
    Further, we use this upper bound having \(c(s_t,a_t)\) already observed and rewriting
    \begin{align}
        &\gamma^{t+1} \max_{\tilde{f} \in \model_n} \E_\omega\left[\bar{V}_c^{\bpi}(\tilde{f}(s_{t},a_{t}) + \omega_t )\right] = \gamma^t \pVc(s_t) - \gamma^tc(s_t,a_t)\\
         &\labelrel\leq{step:use_pess_def} \E_{\bpi}\left[\sum_{\tau=t}^\infty \gamma^\tau\left(c(s_\tau,a_\tau) + \lambda_{\text{pessimism}} \|\sigma_n(s_\tau,a_\tau)\|\right)\right] - \gamma^tc(s_t,a_t)\\
         &\leq \E_{\bpi}\left[\sum_{\tau=t+1}^\infty \gamma^\tau\left(c(s_\tau,a_\tau) + \lambda_{\text{pessimism}} \|\sigma_n(s_\tau,a_\tau)\|\right)\right] + \gamma^t\lambda_{\text{pessimism}}\|\sigma_n(s_t,a_t)\|,\\
         \st& ~s_{\tau+1} = \mu_n(s_\tau,a_\tau)+\omega_\tau, \quad a_\tau \sim \bpi(\cdot|s_\tau), \quad s_{t+1}=\mu_n(s_t,a_t) + \omega_t,
    \end{align}
    where step \ref{step:use_pess_def} follows from using the upper bound in \Cref{eq:upperBoundCVp}.
\end{proof}

\sthrm*
\begin{proof}\label{proof:sthrm}
    \Cref{theo:safety} is proven using \Cref{slemma:safetyPessValue,slemma:intrinsicBound}.
    Consider the accumulate cost \(\bar \Phi(s_t,a_t,\ct,\model_n)\) from \Cref{slemma:safetyPessValue}
    \begin{align}
        \bar \Phi(s_t,a_t,\ct,\model_n) = 
        & \ct + \gamma^t c(s_t,a_t) + \gamma^{t+1}\max_{\tilde{f} \in \model_n} \E_\omega\left[\bar{V}_c^{\bpi}(\tilde{f}(s_{t},a_{t}) + \omega_t)\right] \\
        \labelrel\leq{step:lem2_up_bound} & \ct + \gamma^t(c(s_t,a_t) + \lambda_{\text{pessimism}}\|\sigma_n(s_t,a_t)\|) \nonumber\\ &+ \E_{\bpi}\left[\sum_{\tau=t+1}^\infty \gamma^\tau \left(c(s_\tau,a_\tau) + \lambda_{\text{pessimism}}\|\sigma_n(s_\tau,a_\tau)\|\right)\right]\\
        \labelrel={step:def_qcn}& \ct + \gamma^t \Qcn(s_t,a_t)
        \labelrel\eqqcolon{step:lem1_def} \Phi(s_t,a_t,\ct,\Qcn). \label{eq:Qcn-bound}
    \end{align}
Step~\ref{step:lem2_up_bound} follows from the upper bound on the cost penalties in \Cref{slemma:intrinsicBound}, and Step~\ref{step:def_qcn} and \ref{step:lem1_def} follow from the definitions of \(\Qcn\) and \(\Phi(s_t,a_t,\ct,\Qcn)\) from \Cref{eq:pessimistic-cost-value-approx,slemma:safetyPessValue}.
\end{proof}

\section{Regret Analysis}
\label{sec:regret-analysis}
\paragraph{Overview of proof.}
\Cref{theo:regret} states that the tractable objective in \Cref{eq:optiPlanning} achieves sublinear cumulative regret while satisfying the safety constraint. We start our proof by decomposing the cumulative regret into per-episode regret terms. Using \Cref{lemma:switchingRegret}, we lower-bound the performance of the true system \(\CM\) by the performance on the planning MDP \(\SM\). This bound allows us to decompose the per-episode regret into two components in \Cref{lemma:regretDecomposition} that can be bounded separately:      
\begin{enumerate*}[label=\textbf{(\roman*)}]
\item \textbf{Safe rollout vs.\ optimal rollout.}  
The first regret term captures the performance gap between the optimal policy \(\ocpi\) and the same policy executed via \Cref{algo:safeRollout} on \(\CM\). The difference arises from invoking the prior policy \(\bpi\) due to uncertainty in the learned dynamics. We express this gap in \Cref{lemma:regretPolicyDistance} in terms of action differences between \(\ocpi\) and the potentially used safe prior \(\bpi\). By \Cref{theo:safety}, these action differences can be related in \Cref{lemma:fromPolicyToUncertainty} via the safety criterion \(\Phi\) to the accumulated model uncertainty along potentially unsafe trajectories starting with \(\ocpi(a_t | s_t)\) at step \(t\), followed by \(\bpi\). Since it may be unsafe to execute them, we instead express their cumulative uncertainty in terms of the corresponding safe trajectories starting with \(\bpi(a_t|s_t)\) at step \(t\) (\Cref{lemma:relatingUncertaintyTrajectories}). This is achieved using the change of measure between the optimal and backup action. Consequently, we establish in \Cref{lemma:firstRegretTerm} an upper bound on the first regret term, which depends on the model uncertainty along the safe rollout of \(\ocpi\) under \Cref{algo:safeRollout}.
\item \textbf{Planning model vs.\ true dynamics.}  
The second regret term quantifies the value gap between an optimal policy for the planning MDP \(\SM\) (with true dynamics \(f\)) and the policy \(\ppi\) obtained by optimizing the nominal model \(\mu_n\) in episode \(n\). This gap can be bounded in terms of the model uncertainty evaluated along the optimal policy \(\opi\) when it is safely executed using \Cref{algo:safeRollout}.  
\end{enumerate*}
By \Cref{lemma:optPlanningBound}, the tractable objective in \Cref{eq:optiPlanning} is an optimistic upper bound on the best achievable performance across all models \(\model_n\). Consequently, \Cref{lemma:optPlanningUncertainty} provides a bound on the sum of the two per-episode regret components along the safely rolled-out behavioral policy \(\ppi\) using \Cref{algo:safeRollout}. Summing these bounds over all episodes yields \Cref{theo:regret}, which states that the objective in \Cref{eq:optiPlanning} attains sublinear cumulative regret.

To decompose \(R(N)\), we first establish a lower bound on the performance of \(\sppi\), obtained by safely rolling out \(\ppi\) on \(\CM\) with \Cref{algo:safeRollout} using \(\Phi\).

\paragraph{Planning at the lower performance bound.}
Given the policy \(\ppi\) on \(\SM\), we derive in \Cref{lemma:switchingRegret} the performance bound of the safe policy \(\sppi\) on \(\CM\) with dynamics \(f\) using the pessimistic value function \(\underline{V}_r^{\bpi}(s_0)\) defined in \Cref{eq:pessValueR}. 
\begin{lemma}\label{lemma:switchingRegret}
    Suppose \Cref{assume:noise,assume:RKHS,assume:regularity} hold and the model \(\model_n\) is well-calibrated according to \Cref{def:callib}. 
    Then the performance \(\Jr(\sppi,f)\) of the safe policy \(\sppi\) induced by \Cref{algo:safeRollout} on \(\CM\) is lower-bounded by the performance \(\Jtr(\ppi,f)\) of the policy \(\ppi\) on \(\SM\) following the true dynamics \(f\):
    \begin{equation*} \label{eq:perfIneq}
        \Jr(\sppi, f) \geq \Jtr(\ppi, f).
    \end{equation*}
\end{lemma}
\begin{proof}\label{proof:rlemmaOne}
We consider for all timesteps \(k\) two cases:
\begin{enumerate}[leftmargin=0.5cm, parsep=0.3em]
    \item \emph{Never fall back to prior} (i.e. \(\ppi\) is always safe). Then, by definition of \(\SM\),
\begin{equation}
\Jtr(\ppi,f) = \E_{\ppi,s_0}\Big[\sum_{t=0}^\infty \gamma^t r(s_t,a_t)\Big] = \Jr(\sppi,f).
\end{equation}
This equality follows by definition from \(\SM\) and \(\CM\) being identical in case the agent never falls back to the safe prior \(\bpi\) during the rollout and therefore \(\ppi = \sppi\).
\item \emph{Agent falls back to safe prior} (i.e. at \(s_t\) action \(\hat a_t \sim \bpi(\cdot | s_t)\) is executed instead of potentially unsafe \(a_t\)). For the performance \(\Jr(\sppi,f)\) of \(\sppi\) on \(\CM\) we get by combining the policy values of \(\ppi\) and \(\bpi\)
\begin{equation}
\Jr(\sppi,f) = \E_{\ppi, s_0}\Big[\sum_{t=0}^{k-1}\gamma^t r(s_t,a_t)\Big] + \gamma^k\,V_r^{\bpi}(s_k),
\end{equation}
whereas for the performance of \(\Jtr(\ppi,f)\) of \(\ppi\) on \(\SM\), we get
\begin{align}
    \Jtr(\ppi,f) &\labelrel={step:def_mod_trans} \E_{\ppi,s_0}\Big[\sum_{t=0}^{k-1}\gamma^t \tilde{r}(s_t,a_t) + \gamma^k\tilde{r}(s_k,\hat{a}_k) + \sum_{t=k+1}^{\infty} \gamma^t \tilde{r}(s_\dagger,\cdot) \Big] \label{eq:SMterminalTransition}\\
    &\labelrel={step:def_term_reward} \E_{\ppi,s_0}\Big[\sum_{t=0}^{k-1}\gamma^t r(s_t,a_t) + \gamma^k\,\pVr(s_k)\Big].\label{eq:SMterminalReward}
\end{align}
Step \ref{step:def_mod_trans} follows from the modified transitions in \Cref{eq:terminalTransition} to the absorbing state \(s_\dagger\) as soon as the agent would fall back to the safe prior \(\bpi\), and \ref{step:def_term_reward} follows from the terminal reward given in \Cref{eq:terminalReward} for transitioning to \(s_\dagger\).
Since \(\underline{V}_r^{\bpi}\) is defined pessimistically in \Cref{eq:pessValueR}, it holds that \(\pVr(s) \leq \Vr(s)\) and thereby follows
\begin{equation} \label{eq:ineqValueOneSwitch}
\Jtr(\ppi,f) \leq \E_{\ppi, s_0}\Big[\sum_{t=0}^{k-1}\gamma^t r(s_t,a_t)\Big] + \gamma^k\,\Vr(s_k) = \Jr(\sppi,f),
\end{equation}
which completes the proof.
\end{enumerate}
\end{proof}

\paragraph{Regret decomposition.}
\Cref{lemma:switchingRegret} enables us to decompose \(R(N)\) into two separate terms. The first term models the regret of executing \(\sppi\) on \(\CM\), caused by falling back to the safe prior \(\bpi\) when safely rolling out \(\ppi\) with \Cref{algo:safeRollout}. The second component determines the (constraint-free) regret of \(\ppi\) on \(\SM\).
\begin{lemma}\label{lemma:regretDecomposition}
Suppose \Cref{assume:noise,assume:RKHS,assume:regularity,assume:safeGap} hold and the model \(\model_n\) is well-calibrated according to \Cref{def:callib}. Let \(\ocpi\) be the optimal policy on \(\CM\) with \(\Jr(\ocpi,f) = \max_{\pi : \Jc(\pi,f) < d}\Jr(\pi,f)\) and applying \Cref{algo:safeRollout} to \(\ocpi\) yields \(\socpi\) with \(\Jr(\socpi,f)\) for the current learned model \(\model_n\) in episode \(n\). Further, let \(\opi\) be the optimal policy on \(\SM\) that has \(\Jtr(\opi,f) = \max_\pi\Jtr(\pi,f)\) for \(\Phi\) and \(\Qcn\) in episode \(n\). Given \(\Jtr(\ppi,f)\) denotes the performance of the behavioral policy \(\ppi\) on \(\SM\), the cumulative regret \(R(N) = \sum_{n=1}^N (\Jr(\ocpi,f) - \Jr(\sppi,f))\) of \(\sppi\) on \(\CM\) can be decomposed into:
\begin{align*}
R(N) \leq \sum_{n=1}^N\big(\underbrace{\Jr(\ocpi,f) - \Jr(\socpi,f)}_{\PERa \,:\,\text{P.E.R. of } \sopi_c\text{ with }\Phi}\big)
+ \sum_{n=1}^N \big(\underbrace{\Jtr(\opi,f) - \Jtr(\ppi,f)}_{\PERb \,:\, \text{P.E.R. of } \ppi \text{on }\tilde{\set{M}}}\big),
\end{align*}
with the per-episode regret (P.E.R.) terms \(\PERa\) and \(\PERb\) related to the safe prior and the regret on \(\SM\).
\end{lemma}
\begin{proof}\label{proof:rlemmaTwo}
Starting with the cumulative regret definition in \Cref{eq:objective} for \(\CM\), we upper bound \(R(N)\) using 
\begin{align}
    R(N) & \labelrel={step:bound_with_lem3} \sum_{n=1}^N ( \Jr(\ocpi,f) - \Jr(\sppi,f) ) \leq \sum_{n=1}^N ( \Jr(\ocpi,f) - \Jtr(\ppi,f) )\\\label{eq:passiveZero}
& \labelrel={step:intro_zer_term} \sum_{n=1}^N (\Jr(\ocpi,f) - \Jr(\socpi,f)
+ \Jr(\socpi,f) - \Jtr(\ppi,f)),
\end{align}
where step \ref{step:bound_with_lem3} follows applying \Cref{lemma:switchingRegret} and \ref{step:intro_zer_term} introducing a zero term by adding and subtracting \(\Jr(\socpi,f)\).
Since \(\CM\) and \(\SM\) are identical for safe policies (e.g. \(\socpi\)) by not transitioning to \(s_\dagger\) (see \Cref{eq:terminalTransition,eq:terminalReward}) and \(\opi\) is the optimal policy on \(\SM\) with \(\model_n\) in episode \(n\)
\begin{align}\label{eq:optiPiPlanning}
 \Jr(\socpi,f) = \Jtr(\socpi,f) \leq \Jtr(\opi,f).
\end{align}
Thus, we obtain the following upper bound of \(R(N)\) by plugging \Cref{eq:optiPiPlanning} into \Cref{eq:passiveZero}:
\begin{equation}
R(N) \le \sum_{n=1}^N ( \Jr(\ocpi,f) - \Jr(\socpi,f) )
+ \sum_{n=1}^N ( \Jtr(\opi,f) - \Jtr(\ppi,f) ).
\end{equation}
\end{proof}

\paragraph{\(\boldsymbol{\PERa}\) regret bound.}
Recall from \Cref{lemma:regretDecomposition} that the per episode regret of executing \(\ocpi\) with \(\Phi\) for \(\model_n\) in episode \(n\) on the planning MDP \(\CM\) is \(\PERa = \Jr(\ocpi) - \Jr(\socpi)\). We rewrite \(\PERa\) in terms of the distance between the optimal policy \(\ocpi\) and the safely rolled-out policy \(\socpi\) in \Cref{lemma:regretPolicyDistance}.

\begin{lemma}\label{lemma:regretPolicyDistance}
    Suppose \Cref{assume:RKHS,assume:noise,assume:regularity,assume:safeGap} hold and the model \(\model_n\) is calibrated according to \Cref{def:callib}. Then the per-episode regret \(\PERa = \Jr(\ocpi,f) - \Jr(\socpi,f)\) admits the bound
    \begin{align*}
        \PERa & \leq \E_{\sopi_c,s_0}\bigg[ \left(L_r + \frac{\bar RL_f}{\sigma}\frac{\gamma}{1-\gamma}\right)\sum_{t=0}^{\infty} \gamma^t \,\E_{\substack{\oca \sim \ocpi(\cdot|s_t)\\\soca \sim \socpi(\cdot|s_t)}}\left[\|\oca-\soca\|\right]\bigg],\\
        & s_{t+1} = f(s_t,a_t) + \omega_t, \quad a_t \sim \socpi(\cdot|s_t)
    \end{align*}
    with the Lipschitz constants \(L_r,L_f\) for the reward and dynamics, and \(\bar R=\max\{R_{max},k_{max}\}\).
\end{lemma}
\begin{proof}\label{proof:rlemmaThree}
    We bound the episodic regret with the advantage formulation \cite[Lemma~A.2]{as2025actsafe}
\begin{align}
    \PERa = & \E_{\socpi,s_0}\left[\sum_{t=0}^{\infty} \gamma^t \, \E_{\oca\sim\ocpi(\cdot|s_t)} \left[A_{r}(\socpi,s_t,\oca)\right]\right],
\end{align}
with \(A_{r,t}(\pi,s_t,a_t) = r(s_t,a_t) + \gamma \E_{\omega}[V_{r}^\pi(s_{t+1}) - V_{r}^\pi(s_t)]\) and \(s_{t+1} = f(s_t,a_t) + \omega_t\).
\begin{align}
    \PERa \labelrel={step:def_adv} 
    & \E_{\socpi,s_0}\Bigg[\sum_{t=0}^{\infty} \gamma^t\, \E_{\oca\sim \ocpi(\cdot|s_t)}\bigg[\Big(r(s_t,\oca) \label{eq:advantageDef} \\ & + \gamma\left(\E_{s_{t+1}|s_t,\ocpi,\omega} \left[V^{\socpi}_r(s_{t+1})\right]- V^{\socpi}_r(s_t)\right)\Big)\bigg]\Bigg]\nonumber\\
    \labelrel={step:def_value} & \E_{\socpi,s_0}\left[\sum_{t=0}^{\infty} \gamma^t \, \E_{\substack{\oca \sim \opi(\cdot|s_t)\\\soca \sim \socpi(\cdot|s_t)}} \left[(r(s_t,\oca) - r(s_t,\soca))\right]\right] \label{eq:advantageExtension}\\ \nonumber& + \E_{\socpi,s_0}\bigg[\sum_{t=0}^{\infty} \gamma^{t+1} \big(\E_{s_{t+1}|s_t,\ocpi,\omega} [V^{\socpi}_r(s_{t+1})]- \E_{s_{t+1}|s_t,\socpi,\omega} [V^{\socpi}_r(s_{t+1})]\big)\bigg]\\
    \labelrel\leq{step:advantageKakade} & \E_{\socpi,s_0}\bigg[ \sum_{t=0}^{\infty} \gamma^t \E_{\substack{\oca \sim \opi(\cdot|s_t)\\\soca \sim \socpi(\cdot|s_t)}}\left[\min\{L_r\|\oca-\soca\|,R_{max}\}\right] \\ \nonumber& + \sum_{t=0}^\infty \gamma \sqrt{\max\{\E_{s_{t+1}|s_t,\ocpi,\omega_t}[R(s_{t+1})],\E_{s_{t+1}|s_t,\socpi,\omega_t}[R(s_{t+1})]\}} \\ \nonumber&\times \gamma^t \E_{\substack{\oca \sim \opi(\cdot|s_t)\\\soca \sim \socpi(\cdot|s_t)}} \left[\min\bigg\{\frac{L_f\|\oca-\soca\|}{\sigma},1\bigg\}\right]\bigg], 
\end{align}
where step \ref{step:def_adv} follows from the advantage definition and \ref{step:def_value} uses the value function to derive the dependence on the reward difference in timestep \(t\). We rewrite the advantage in terms of the expected action difference in step \ref{step:advantageKakade} using the expectation difference under two Gaussians \citep[lemma~C.2]{KakadeInfo} with \(R(s) = \left(V^{\socpi}_r(s)\right)^2\). Step \ref{step:advantageKakade} uses the Lipschitz constants \(L_r,L_f\) for the reward and the dynamics (\Cref{assume:regularity}).
Further, we bound this by:
\begin{align}
    \PERa \labelrel\leq{step:advantageRmax} & \E_{\socpi,s_0}\bigg[ \sum_{t=0}^{\infty} \gamma^t \E_{\substack{\oca \sim \opi(\cdot|s_t)\\\soca \sim \socpi(\cdot|s_t)}} \left[\min\{L_r\|\oca-\soca\|,R_{max}\}\right] \\\nonumber & + \bar R\frac{\gamma}{1-\gamma} \sum_{t=0}^\infty \gamma^t \E_{\substack{\oca \sim \opi(\cdot|s_t)\\\soca \sim \socpi(\cdot|s_t)}} \left[\min\bigg\{\frac{L_f\|\oca-\soca\|}{\sigma},1\bigg\}\right]\bigg]\\
    \labelrel\leq{step:advantageMin} & \E_{\socpi,s_0}\bigg[ \left(L_r + \frac{\bar RL_f}{\sigma}\frac{\gamma}{1-\gamma}\right)\sum_{t=0}^{\infty} \gamma^t \E_{\substack{\oca \sim \opi(\cdot|s_t)\\\soca \sim \socpi(\cdot|s_t)}} \left[\|\oca-\soca\|\right]\bigg]. 
\end{align}
Step \ref{step:advantageRmax} uses \(\sqrt{\max\{\E_{\omega_t,s_{t+1}|s_t,\ocpi,\omega}[R(s_{t+1})],\E_{\omega_t,s_{t+1}|s_t,\socpi,\omega}[R(s_{t+1})]\}} \leq \bar R \frac{1}{1-\gamma}\) with \(\bar R = \max\{R_{\text{max}},k_{\text{max}}\}\) due to \(r\geq0\) and in \ref{step:advantageMin} we drop the \(\min\) operator due to non-negativity.
\end{proof}

\Cref{lemma:regretPolicyDistance} expresses the regret in terms of the difference between actions taken by an optimal policy \(\ocpi\) and those taken by $\socpi$ in episode \(n\) with model \(\model_n\). In the following, we use \Cref{lemma:regretPolicyDistance} to relate $\PERa$ to model uncertainties ``collected'' along trajectories induced by these policies. Bounding these uncertainties ensures that the learner improves its model of the dynamics after each episode.
\begin{lemma}\label{lemma:fromPolicyToUncertainty}
Suppose \Cref{assume:RKHS,assume:noise,assume:regularity,assume:safeGap} hold and the model \(\model_n\) is well calibrated according to \Cref{def:callib}. 
Given \(\Phi\) is defined as in \Cref{theo:safety}, \(\PERa\) depends on the accumulated and discounted uncertainty \(\Sigma_t\) along the trajectories of action \(a_t\) and the safe \(\bpi\), and the diameter of the bounded action space \(D_{\set{A}}\). Given \(\nu_n = (2+\lambda_{\text{pessimism}})\lambda_{\text{pessimism}}\) with \(\lambda_{\text{pessimism}}\) being defined as in \Cref{theo:safety}, 
\begin{align*}
    & \PERa \leq \E_{\socpi,s_0}\bigg[ \left(L_r + \frac{\bar RL_f}{\sigma}\frac{\gamma}{1-\gamma}\right)\sum_{t=0}^{\infty} \gamma^t  D_{\set{A}} \min\left\{\E_{a_t\sim\ocpi(\cdot | s_t),\omega_t}\left[\frac{\nu_n\Sigma_t(a_t,f)}{\delta_c}\right],1\right\} \bigg],\\  
    &\Sigma_t(a_t,f) = \gamma^t\|\sigma_n(s_t,a_t)\| +\sum_{\tau=t+1}^\infty \E_{\hat a_\tau 
    \sim 
    \hat \pi(\cdot|s_\tau)}[\gamma^\tau\|\sigma_n( s_\tau,\hat a_\tau)\|], \quad s_{\tau+1} = f(s_\tau,\hat a_\tau) + \omega_\tau\\
    &\text{and } \delta_c = \min_t (\E_{a_t\sim\ocpi(\cdot | s_t)}[d - (c_{<t} + \gamma^t c(s_t,a_t)+ \gamma^{t+1} \E_{\omega_t}[V_c^{\bpi}(s_{t+1})])])\in \R_{>0}.
\end{align*}
\end{lemma}
\begin{proof}\label{proof:rlemmaFour}
We can upper bound the expected difference between the policies by the action bound \(D_{\set{A}}\) times the probability \(\sP\) of falling back to the safe prior 
    \begin{align}
    D_{\socpi}(s_t) = & \E_{\substack{\oca \sim \opi(\cdot|s_t)\\\soca \sim \socpi(\cdot|s_t)}} \left[\|\oca-\soca\|\right]\\   \labelrel\leq{step:termProbDef} &D_{\set{A}}\sP_{a_t \sim \ocpi(\cdot|s_t)}\left(c_{<t} + \gamma^t \Qcn(s_t,a_t) \geq d \right)\\
      \labelrel={step:termProbZero1} & D_{\set{A}}\sP_{a_t \sim \ocpi(\cdot|s_t)}\bigg(c_{<t} + \gamma^t c(s_t,a_t) \\\nonumber& + \gamma^{t+1}\Big(\E_{\omega_t}[\tVc(\mu_n(s'_t,a'_t)+\omega_t)]  + \lambda_{\text{pessimism}} \Sigma_t(a_t,\mu_n)\Big) \geq d\bigg)
    \end{align}
    with \(s'_{t+1} = \mu_n(s'_t,a'_t) + \omega_t\), \(\Sigma_t(a_t,\mu_n) = \|\sigma_n(s_t,a_t)\| + \sum_{\tau=t+1}^\infty \E_{\hat a_\tau' \sim \hat \pi(\cdot|s_\tau')}[\gamma^{\tau-t}\|\sigma_n(s'_\tau,\hat a_\tau'))\|]\) and \(\tVc(s_0) = \sum_{t=0}^\infty \gamma^t c(s'_t,a'_t)\) following the model dynamics \(\mu_n\).
    Step \ref{step:termProbDef} follows from the switching criterion in \Cref{theo:safety} and step \ref{step:termProbZero1} by the definition in \Cref{eq:pessimistic-cost-value-approx}.
    Next, we relate the uncertainty along  trajectories under the model dynamics \(\mu_n\) to those under the true dynamics \(f\) by
    \begin{align}
      D_{\socpi}(s_t)       \labelrel\leq{step:termProbZero} & D_{\set{A}}\sP_{a_t \sim \ocpi(\cdot|s_t)}\bigg(c_{<t} + \gamma^t c(s_t,a_t) + \gamma^{t+1}\Big(\E_{\omega_t}[\tVc(\mu_n(s'_t,a'_t)+\omega_t)]  \\\nonumber& + \E_{\omega_t}[V^{\bpi}_{c}(f(s_t,a_t)+\omega_t)] - \E_{\omega_t}[V^{\bpi}_{c}(f(s_t,a_t)+\omega_t)] \\ \nonumber & + \lambda_{\text{pessimism}} (\Sigma_t(a_t,\mu_n) + \Sigma_t(a_t,f) - \Sigma_t(a_t,f))\Big) \geq d\bigg)
 \\
      \labelrel\leq{step:termBound} & D_{\set{A}}\sP_{a_t \sim \ocpi(\cdot|s_t)}\bigg(c_{<t} + \gamma^t c(s_t,a_t) + \gamma^{t+1}\Big(\E_{\omega_t}[V^{\bpi}_{c}(f(s_t,a_t)+\omega_t)]  \\\nonumber& + (2 + \lambda_{\text{pessimism}})\lambda_{\text{pessimism}} \Sigma_t(a_t,f)\Big) \geq d\bigg)\\
      \labelrel\leq{step:compacter} & D_{\set{A}}\sP_{a_t \sim \ocpi(\cdot|s_t)}\bigg(c_{<t} + \gamma^t c(s_t,a_t) + \gamma^{t+1}\Big(\E_{\omega_t}[V^{\bpi}_{c}(s_{t+1})] + \nu_n \Sigma_t(a_t,f)\Big) \geq d\bigg)
\end{align}
We first expand the inequality in \ref{step:termProbZero} with two zero terms. In step \ref{step:termBound}, we bound the difference of the value and the uncertainty for the different dynamics using Lemma~A.5 of \citet{sukhija2025optimism}. In \ref{step:compacter} we define \(\nu_n = (2+\lambda_{\text{pessimism}})\lambda_{\text{pessimism}}\)
and sample according to the true dynamics \(s_{t+1} = f(s_t,a_t) + \omega_t\).  

Consider the set of all policies that are safely reachable given $\bpi$
\begin{equation}
    \Pi_{<d}^{\bpi} \coloneqq \{\pi \mid \E_\pi[c_{<t} + \gamma^t c(s_t,a_t)] + \gamma^{t+1}\E_{\omega_t}[V_c^{\bpi}(s_{t+1})] < d, \forall t \}.
\end{equation}
Since $\ocpi$ is strictly feasible we have $\E_{\ocpi}[c_{<t} + \gamma^t c(s_t,a_t) + \gamma^{t+1}\E_{\omega_t}[V_c^{\ocpi}(s_{t+1})]] < d$; in addition, $\Vc(s) \leq V_c^{\ocpi}(s)$  for all initial states (due to \Cref{assume:safeGap}), therefore $\E_{\ocpi}[c_{<t} + \gamma^t c(s_t,a_t) + \gamma^{t+1}\E_{\omega_t}[V_c^{\bpi}(s_{t+1})]] < d$ implying that $\ocpi \in \Pi_{<d}^{\bpi}$.
We define the smallest (time-wise) safety gap due to switching from $\ocpi$ to $\bpi$ as
\begin{equation}
    \delta_c \eqdef \inf_{t>0,\mathcal{H}_{t-1}} (\E_{a_t\sim\ocpi(\cdot|s_t)}[d - (c_{<t} + \gamma^t c(s_t,a_t)+ \gamma^{t+1} \E_{\omega_t}[V_c^{\bpi}(s_{t+1})])])\in\R_{>0}.
\end{equation}
Intuitively, this gap tells us the ``latest'' point after which the agent must switch to the prior policy.
We thus rewrite the probability of invoking the safe prior \(\bpi\)
\begin{align}
    D_{\socpi}(s_t) & = D_{\set{A}}\sP_{a_t \sim \ocpi(\cdot|s_t)}\bigg(\nu_n \Sigma_t(a_t,f) \geq d- (c_{<t} + \gamma^t c(s_t,a_t) + \gamma^{t+1} \E_{\omega_t}[V^{\bpi}_{c}(s_{t+1})])\bigg)\\
    & = D_{\set{A}}\sP_{a_t \sim \ocpi(\cdot|s_t)}\bigg(\nu_n \Sigma_t(a_t,f) \geq \delta_c\bigg).
\end{align}
Using \emph{Markov's inequality}, we can bound the probability of invoking the prior by
\begin{align} \label{eq:policyDiff}
    D_{\socpi}(s_t) & \leq D_{\set{A}}\min\Bigg\{\E_{a_t\sim\ocpi(\cdot|s_t),\omega_t}\left[\frac{\nu_n \Sigma_t(a_t,f)}{\delta_c}\right],1\Bigg\}
\end{align}
Finally, we insert \Cref{eq:policyDiff} into \Cref{lemma:regretPolicyDistance} and obtain
\begin{align}
    \label{eq:optimal-policy-prob}
    \PERa \leq & \E_{\socpi,s_0}\bigg[ \left(L_r + \frac{\bar RL_f}{\sigma}\frac{\gamma}{1-\gamma}\right)\sum_{t=0}^{\infty} \gamma^t  D_{\set{A}} \min\left\{\E_{a_t\sim\ocpi(\cdot|s_t),\omega_t}\left[\frac{\nu_n \Sigma_t(a_t,f)}{\delta_c}\right],1\right\} \bigg] .
\end{align}
\end{proof}
\begin{remark}\looseness=-1
Indeed, at the beginning of learning, $\ocpi$ may not lie within the set $\{\pi \mid \E_\pi[c_{<t} + \gamma^t c(s_t,a_t) + \gamma^{t+1}\E_{\omega_t}[\pVc(s_{t+1})] | \mathcal{H}_{t}] < d, \forall t \}\subseteq \Pi_{<d}^{\bpi}$ as the pessimistic cost value $V_c^{\bpi} \leq \pVc$ for all $s \in \set{S}$ with probability $1 - \delta$ by definition (recall \Cref{eq:pessValue}). Reducing pessimism such that $\pVc$ converges to $V_c^{\bpi}$ effectively \emph{expands} this set until it converges to $\Pi_{<d}^{\bpi}$.
Moreover, let us consider the case where the prior policy does \emph{not} satisfy the assumption that $\smash{\Vc(s) \leq V_c^{\ocpi}(s)}$ for all $s \in \set{S}$, i.e., the pessimistic prior policy $\bpi$ in fact accumulates \emph{more} costs (in expectation) than an optimal policy $\ocpi$ when executed on the true dynamics $f$. In this case, $\ocpi$ is not guaranteed to lie within $\Pi_{<d}^{\bpi}$ and \algname{SOOPER} converges to the optimum within the set of all feasible policies \(\Pi_{<d}^{\bpi}\) that are safely reachable given $\bpi$.
\end{remark}

\begin{wrapfigure}[14]{r}{0.4\textwidth}  
    \vspace{-0.4cm}
    \centering
    \begin{tikzpicture}[scale=0.8]
  \definecolor{lightgreen}{RGB}{15, 133, 84}
  \definecolor{lightred}{RGB}{204,80,62}
  \definecolor{purple}{RGB}{95,70,144}

  \begin{scope}
    \clip (0,0) rectangle (7,3.5);
    \fill[lightgreen, fill opacity=0.3] (0,0) -- (0,3.5) -- plot[smooth, tension=1] coordinates {
       (1.3,3.5) (1.5,3) (2,2.5) (2.5,2) (4,0.25) (7,0) (7,0)
    } -- cycle;
    \fill[lightred, fill opacity=0.3] (7,3.5) -- plot[smooth, tension=1] coordinates {
       (1.3,3.5) (1.5,3) (2,2.5) (2.5,2) (4,0.25) (7,0)
    } -- (7,1) -- cycle;
  \end{scope}

  \draw[black, thick] plot[smooth, tension=1] coordinates {
    (0.5,0) (2,0.75) (3.5,1.75) (4.5,2.875) (5,3.5)
  };

  \draw[lightgreen, thick] plot[smooth, tension=1] coordinates {
    (2,0.75) (2.3,1.5) (1.8,2.1) (0.8,3) (0.5,3.5)
  };

  \draw[lightgreen, dashed, thick] plot[smooth, tension=1] coordinates {
    (2,0.75) (2.4,1.5) (2,2.1) (1.1,3) (0.9,3.5)
  };
  \draw[lightgreen, dashed, thick] plot[smooth, tension=1] coordinates {
    (2,0.75) (2.2,1.5) (1.6,2.1) (0.5,3) (0.1,3.5)
  };
  \fill[lightgreen, opacity=0.5]
    plot[smooth, tension=1] coordinates {(2,0.75) (2.4,1.5) (2,2.1) (1.1,3) (0.9,3.5)}
    -- plot[smooth, tension=1] coordinates {(0.1,3.5) (0.5,3) (1.6,2.1) (2.2,1.5) (2,0.75)} -- cycle;

  \draw[lightred, thick] plot[smooth, tension=1] coordinates {
    (3.5,1.75) (3.6,2.25) (2.6,3) (2.2,3.5)
  };
  \draw[lightred, dashed, thick] plot[smooth, tension=1] coordinates {
    (3.4,1.65) (3.7,1.75) (3.85,2.25) (3,3) (2.7,3.5)
  };
  \draw[lightred, dashed, thick] plot[smooth, tension=1] coordinates {
    (3.4,1.65) (3.3,1.75) (3.35,2.25) (2.2,3) (1.7,3.5)
  };

  \fill[lightred, opacity=0.5]
    plot[smooth, tension=1] coordinates {(3.4,1.65) (3.7,1.75) (3.85,2.25) (3,3) (2.7,3.5)}
    -- plot[smooth, tension=1] coordinates {(1.7,3.5) (2.2,3) (3.35,2.25) (3.3,1.75) (3.4,1.65)} -- cycle;

  \node at (2,0.75) {$\bullet$};      
  \node at (3.5,1.75) {$\bullet$};      

  \node[black, anchor=west] at (1.7,0.45) {$s_t:\Phi(s_t,a_t,\ct,Q_{c,n}^{\hat \pi})\geq d$};
  \node[black, anchor=west] at (3.7,1.75) {$\mu_n(s_t,a_t)$};
  \node[black, anchor=west] at (4.7,3) {$\pi^*_c$};
\end{tikzpicture}
\vspace{-0.5cm}
    \caption{
    Relating the uncertainty of a safe trajectory (green) to a trajectory that executes $\ocpi$ freely at \(t\) (i.e. not under \Cref{algo:safeRollout}, in red) and therefore may be (possibly wrongly) considered unsafe due to model uncertainties.}
    \label{fig:relatedTrajectories}
\end{wrapfigure}
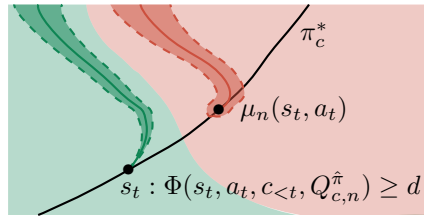
\looseness=-1\Cref{lemma:fromPolicyToUncertainty} establishes an upper bound on the first regret term \(\Delta^1_n\), based on the probability of action \(a_t \sim \ocpi(\cdot|s_t)\) to be unsafe given the model uncertainty. More concretely, even though actions $a_t \sim \ocpi(\cdot|s_t)$ in \Cref{eq:optimal-policy-prob} are determined by a \emph{safe} policy $\ocpi$, they may be regarded unsafe due to limited knowledge during learning, quantified by epistemic model uncertainty. Since such trajectories would trigger $\bpi$ when executing \Cref{algo:safeRollout}, we must relate the bound in \Cref{eq:optimal-policy-prob} to uncertainty induced along trajectories executed by the prior policy $\bpi$. Therefore, we relate in \Cref{lemma:relatingUncertaintyTrajectories} these accumulated uncertainties when executing $\ocpi$ freely in \(t\), to trajectories induced under \Cref{algo:safeRollout}, starting with the safe prior policy at timestep \(t\). This relation is visualized in \Cref{fig:relatedTrajectories}.

\begin{lemma}\label{lemma:relatingUncertaintyTrajectories}
    Suppose \Cref{assume:noise,assume:regularity,assume:RKHS,assume:safeGap} hold, the accumulated uncertainty along a trajectory induced by executing $\ocpi$ freely at \(t\), quantified by \(\E_{a_t\sim\ocpi(\cdot|s_t),\omega_t}[\Sigma_t(a_t,f)]\), can be upper-bounded by the uncertainty along a safe trajectory starting with the safe prior according to \Cref{algo:safeRollout}:
    \begin{equation*}
    \E_{a_t \sim \ocpi(\cdot|s_t)}\big[\Sigma_t(a_t,f)\big]
    \;\leq\; C\,
    \E_{a_t \sim \bpi(\cdot|s_t),\omega_t}\big[\Sigma_t(a_t,f)\big],
    \end{equation*}
    with the finite constant \(C \;:=\; p_{\max}\,\big(2\pi\sigma_a^2\big)^{d_{\set{A}}/2}\,
    \exp\!\left(\frac{D_{\set{A}}^2}{2\sigma_a^2}\right)\).
\end{lemma}
\begin{proof}
We shall use a change of measure argument. Observe that \(\Sigma_t\) is a function of the action, conditioned on \(s_t\). 
\begin{equation}
    \Sigma_t(a,f)
    = \gamma^{t}\,\|\sigma_n(s_t,a)\|
    + \E_{\hat \pi}\left[\sum_{\tau= t+1}^{\infty}\gamma^{\tau}\,
      \|\sigma_n(s_\tau,a_\tau)\|
      \middle| s_t,\,a_t=a\right],
\end{equation}
where \(s_{\tau+1}=f(s_\tau,a_\tau)+\omega_\tau\) and
\(a_\tau\sim\bpi(\cdot|s_\tau)\) for \(\tau\ge t+1\). Furthermore, we can apply the change of measure since \(\Sigma_t(a,f) \geq 0\).  
Since the backup policy is parameterized with a Gaussian distribution, we obtain a density lower bound for the prior on the compact action set \(\set{A}\).
Fix \(s\in\mathcal{S}\) and \(a\in\set{A}\). Since location parameter \(\hat \mu(s) \in\set{A}\) and \(\operatorname{diam}(\set{A})=D_{\set{A}}\), we have
\(\|a-\hat \mu(s)\|\le D_{\set{A}}\), hence it holds for the Gaussian density function \(\phi\)
\begin{equation}
    \phi_{\sigma_a}(a-\hat \mu(s)) \geq (2\pi\sigma_a^2)^{-d_{\set{A}}/2}\,
    e^{-D_{\set{A}}^2/(2\sigma_a^2)}.
\end{equation}
Moreover, \(Z_{\bpi}(s)\leq\int_{\R^{d_{\set{A}}}}\phi_{\sigma_a}=1\), so
truncation only increases the density on \(\set{A}\), and therefore
\begin{equation}\label{eq:densityLB}
    \bpi(a|s) \ge \phi_{\sigma_a}(a-\hat \mu(s)) \geq (2\pi\sigma_a^2)^{-d_{\set{A}}/2}
    e^{-D_{\set{A}}^2/(2\sigma_a^2)}
    \qquad\forall a\in\set{A}, s\in\mathcal{S}.
\end{equation}
The likelihood ratio is uniformly bounded with the maximum density \(p_{\max}\) of \(\ocpi(\cdot|s)\):
\begin{equation}\label{eq:ratioBound}
    \sup_{s\in\mathcal{S}}\,\sup_{a\in\set{A}}\,
    \frac{\ocpi(a\,|\,s)}{\bpi(a\,|\,s)}
    \;\le\; p_{\max}\,(2\pi\sigma_a^2)^{d_{\set{A}}/2}\,
    e^{D_{\set{A}}^2/(2\sigma_a^2)}= C.
\end{equation}
Next, for any measurable \(g \geq 0\), using \Cref{eq:densityLB,eq:ratioBound}, we have
\begin{align}
    \E_{a_t\sim\ocpi(\cdot|s_t)}\big[g(a_t)\big] &= \int_{\set{A}} g(a)\,\ocpi(a|s_t)da
    = \int_{\set{A}} g(a) \frac{\ocpi(a|s_t)}{\bpi(a|s_t)}
      \bpi(a|s_t)da\\ &\le C \int_{\set{A}} g(a)\bpi(a|s_t)da
    = C\E_{a_t\sim\bpi(\cdot|s_t)}\big[g(a_t)\big],
\end{align}
where the inequality uses \(g \ge 0\). This completes the proof for \(g(a)=\Sigma_t(a,f)\).
\end{proof}
Following from the relation between uncertainty accumulated along trajectories that execute $\ocpi$ freely at \(t\) and thereafter follow \(\bpi\), and those that pessimistically invoke $\bpi$ in \Cref{lemma:relatingUncertaintyTrajectories}, we can bound the first regret term when episodes are collected safely using \Cref{algo:safeRollout}.
\begin{lemma}\label{lemma:firstRegretTerm}
    Suppose \Cref{assume:noise,assume:RKHS,assume:regularity,assume:safeGap} hold and the model is calibrated according to \Cref{def:callib}.With bounded actions \( D_{\set{A}}\), the first regret term \(\PERa\) is bounded with probability \(1-\delta\) in terms of the safely rolled out policy \(\socpi\):
    \begin{align*}
    \PERa \leq \E_{\sopi_c}\bigg[ \eta_n \sum_{t=0}^\infty \E[\gamma^t\|\sigma_n(s_t,a_t)\|]\bigg],
    \quad s_{t+1} = \mu_n(s_t, a_t) + \omega_t, \quad a_t \sim \socpi(\cdot|s_t),
    \end{align*}
where \(\eta_n = \left(L_r + \frac{\bar RL_f}{\sigma}\frac{\gamma}{1-\gamma}\right)\frac{1}{1-\gamma}D_\set{A}C\frac{(2+\lambda_{\text{pessimism}})^2\lambda_{\text{pessimism}}}{\delta_c}\), with \(C = p_{\max}\,(2\pi\sigma_a^2)^{d_{\set{A}}/2}\,
    e^{D_{\set{A}}^2/(2\sigma_a^2)}\).
\end{lemma}
\begin{proof}
    Given the derived regret bound in \Cref{lemma:fromPolicyToUncertainty}, we bound the sum by the largest summation term during the safe rollout (i.e. executing $\ocpi$ under \Cref{algo:safeRollout}) using abbreviation \(C' = \left(L_r + \frac{\bar RL_f}{\sigma}\frac{\gamma}{1-\gamma}\right)\).
    \begin{align}
    \PERa \leq & \E_{\socpi,s_0}\bigg[ C' \sum_{t=0}^{\infty} \gamma^t  D_{\set{A}} \min\left\{\E_{a_t\sim\ocpi(\cdot|s_t),\omega_t}\left[\frac{\nu_n \Sigma_t(a_t,f)}{\delta_c}\right],1\right\} \bigg] \\
    \leq & \E_{\socpi,s_0}\bigg[ C'\frac{1}{1-\gamma}  D_{\set{A}} \max_{t} \left(\min\left\{\E_{a_t\sim\ocpi(\cdot|s_t),\omega_t}\left[\frac{\nu_n \Sigma_t(a_t,f)}{\delta_c}\right],1\right\}\right) \bigg] \\
    \leq & \E_{\socpi,s_0}\bigg[ C'\frac{1}{1-\gamma}  D_{\set{A}} \min\left\{\max_{t} \left(\E_{a_t\sim\ocpi(\cdot|s_t),\omega_t}\left[\frac{\nu_n \Sigma_t(a_t,f)}{\delta_c}\right]\right),1\right\} \bigg]. 
\end{align}
As we invoke the prior for the maximal term we obtain together with \Cref{lemma:relatingUncertaintyTrajectories}
\begin{align}
    \PERa \leq & \E_{\socpi,s_0}\bigg[ C'\frac{1}{1-\gamma}  D_{\set{A}} \min\bigg\{\max_{t} \bigg(C \frac{\nu_n}{\delta_c} \sum_{\tau=t}^\infty \E_{\omega_t\hat \pi}[\gamma^\tau \|\sigma_n(s_\tau,a_\tau)\|]\bigg),1\bigg\}\bigg]\\
    \leq& \E_{\socpi,s_0}\bigg[ C'\frac{1}{1-\gamma}  D_{\set{A}} C\frac{\nu_n}{\delta_c}\sum_{t=0}^\infty \E_{\omega_t}[\gamma^t \|\sigma_n(s_t,a_t)\|]\bigg],\\
    ~& s_{t+1} = f(s_t,a_t) + \omega_t, \quad a_t\sim \socpi(\cdot|s_t).
\end{align}
Since we now have the uncertainty along the entire safe rollout of \(\socpi\) following the true dynamics \(f\), we can now express the entire rollout in terms of the model dynamics \(\mu_n\), using \cite[Lemma~A.5]{sukhija2025optimism}:
\begin{align}
    \PERa \labelrel\leq{step:zeroSigmas} & \E_{\socpi,s_0}\bigg[ C'\frac{1}{1-\gamma}  D_{\set{A}}C\frac{\nu_n}{\delta_c}\bigg(\sum_{t=0}^\infty \E_{\omega_t}[ \gamma^t \|\sigma_n(s_t,a_t)\|]\\\nonumber 
    & + \sum_{t=0}^\infty \E_{\omega_t}[ \gamma^t \|\sigma_n(s'_t,a'_t)\|] - \sum_{t=0}^\infty \E_{\omega_t}[\gamma^t \|\sigma_n(s'_t,a'_t)\|]\bigg] \bigg)\\
    \labelrel\leq{step:SigmaRelation} & \E_{\socpi,s_0}\bigg[ C'\frac{1}{1-\gamma}  D_{\set{A}} C \frac{\nu_n}{\delta_c} (1+\lambda_{\text{pessimism}})\sum_{t=0}^\infty \E_{\omega_t}[\gamma^t \|\sigma_n(s'_t,a'_t)\|]\bigg]\\ 
    ~& s_{t+1} = f(s_t,a_t) + \omega_t, \quad s'_{t+1} = \mu_n(s'_t,a'_t) + \omega_t.
\end{align}
In step \ref{step:zeroSigmas}, we introduce a zero term by adding and subtracting the accumulated uncertainties along the model dynamics. Next, we bound in \ref{step:SigmaRelation} the difference of the uncertainties along the two different dynamics with \cite[Lemma~A.5]{sukhija2025optimism}.
Thus, the proof is complete by plugging back in \(C'\) and \(\nu_n\).
\end{proof}

\paragraph{\(\boldsymbol{\PERb}\) regret bound.}
Let us now consider the second term \(\PERb\) that determines the regret of the policy \(\ppi\) on \(\SM\), given \(\model_n\). Therefore, we first bound the value difference between the pessimistic value estimate \(\pVr\) and the true value of executing \(\bpi\) on \(f\). Leveraging this bound, we can upper-bound the value difference of executing \(\ppi\) on \(\CM\) and \(\SM\) in \Cref{lemma:regretPlanning}. Finally, we bound the second regret term \(\PERb\) in \Cref{lemma:regretSecond}.

\begin{lemma}\label{lemma:pess}
    Suppose \Cref{assume:noise,assume:RKHS,assume:regularity} hold and the model \(\model_n\) is well-calibrated according to \Cref{def:callib}. Given the pessimistic cost value \(\pVc\) from \Cref{eq:pessValue}, we get with probability \(1-\delta\)
    \begin{align*}
        \gamma^k|\Vr(s_k)-\pVr(s_k)| \leq \lambda_n \sum_{t=k}^\infty \E_{\hat \pi}\left[\gamma^t \|\sigma_{n}(\tilde{s}_t,\tilde{a}_t)\|\right], \quad \lambda_n = \frac{\bar R\gamma}{1-\gamma} \frac{2(1+\sqrt{d_x})\beta_{n-1}(\delta)}{\sigma},
    \end{align*}
    where \(\bar R = \max\{R_{max},k_{max}\}\). The trajectory can either be along the worst-case dynamics \(\tilde{s}_{t+1} = \mu_n(\tilde{s}_t,\tilde{a}_t) + \alpha (1+\sqrt{d_x})\beta_{n-1}(\delta)\sigma_n(\tilde{s}_t,\bpi(\tilde{s}_t)) + \omega_t\), with \(\alpha \in [-1,1]\), initial state \(\tilde s_k = s_k\), and \(\tilde a_t \sim \hat \pi(\cdot | \tilde s_t)\) or the true dynamics \(\tilde{s}_{t+1} = f(\tilde{s}_t,\tilde{a}_t) + \omega_t\), yielding the same bound.
\end{lemma}
\begin{proof}
\begin{align}
    \gamma^k|\Vr(s_k)-\pVr(s_k)| = & \bigg|\E\bigg[\sum_{t=k}^\infty \gamma^t (\Vr(s_{t+1}) - \Vr(\tilde s_{t+1}))\bigg]\bigg| \\
    \leq & \sum_{t=k}^\infty \gamma^t \E[|\E_{\omega_t} (\pVr(s_{t+1}) - \pVr(\tilde s_{t+1}))|],
\end{align}
where \(\tilde{s}_{t+1}\) is in each step the worst next state given \(\model_n\), leading to the worst value (see \Cref{eq:pessValue}). Thus with \(\tilde{s}_{t+1} = \mu_n(\tilde{s}_t,\bpi(\tilde{s}_t)) + \alpha (1+\sqrt{d_x})\beta_{n-1}(\delta)\sigma_n(\tilde{s}_t,\tilde{a}_t) + \omega_t\), where \(\alpha \in [-1,1]\) and \(\hat R(s) = ({\Vr})^2(s)\) it holds
\begin{align}
     & \gamma^k|\Vr(s_k)-\pVr(s_k)| \labelrel\leq{step:worstDynamicsBoundKakade} \sum_{t=k}^\infty \gamma \E\bigg[\sqrt{\max\{\E_{\omega_t}[\hat R(s_{t+1})],\E_{\omega_t}[\hat R(\tilde s_{t+1})]\}}\\ \nonumber & 
    \times \gamma^t \min\left\{\frac{\E_{\hat \pi}[\|f(s_t,a_t) - \mu_n(\tilde s_t, \tilde a_t) +  \alpha (1+\sqrt{d_x})\beta_{n-1}(\delta)\sigma_n(\tilde{s}_t,\tilde{a}_t)\|]}{\sigma},1\right\}\bigg]\\
    & \labelrel\leq{step:worstDynamicsUncretainty} \frac{\bar R\gamma}{1-\gamma} \frac{2(1+\sqrt{d_x})\beta_{n-1}(\delta)}{\sigma}\sum_{t=k}^\infty \E_{\hat \pi}\left[\gamma^t \|\sigma_n(s_t,a_t)\|\right].
\end{align}
We rewrite the value difference dependent on the worst-case dynamics \(\mu_n + \alpha(1+\sqrt{d_x})\beta_{n-1}( \delta)\sigma_n\) and the true dynamics \(f\) in step \ref{step:worstDynamicsBoundKakade} using \citet[Lemma~C.2]{KakadeInfo}, which can be expressed in terms of the model uncertainty in \ref{step:worstDynamicsUncretainty}  \citep[Corollary 3]{OptiActive}. 
\end{proof}

\begin{lemma}\label{lemma:regretPlanning}
    Suppose \Cref{assume:RKHS,assume:noise,assume:regularity} hold and the model \(\model_n\) is calibrated according to \Cref{def:callib}. We consider the following definitions
    \begin{align*}
    \Jtr(\ppi,f) & = \E_{\pi,s_0}\bigg[\sum_{t=0}^\infty \gamma^t \tilde{r}(s_t,a_t)\bigg], \\
    a_t &\sim \ppi(\cdot |s_t), \quad s_{t+1} = \begin{cases}
        f(s_t,a_t) + \omega_t & \Phi(s_t,a_t,\ct,\Qcn) \geq d ,s_t \neq s_\dagger,\\
        s_\dagger & \text{otherwise,}
    \end{cases}\\
    \Jtr(\ppi,\mu_n) & = \E_{\pi,s_0}\bigg[\sum_{t=0}^\infty \gamma^t \tilde{r}(\tilde s_t,\tilde a_t)\bigg],\\
     \tilde a_t &\sim \ppi(\cdot | \tilde s_t), \quad \tilde{s}_{t+1} = \begin{cases}\mu_n(\tilde{s}_t,\tilde a_t) + \omega_t & \Phi(\tilde s_t,a_t,\ct,\Qcn) \geq d , \tilde s_t \neq s_\dagger,\\
    s_\dagger & \text{otherwise,}
    \end{cases}\\
    \lambda_n & = \frac{\bar R\gamma}{1-\gamma} \frac{2(1+\sqrt{d_x})\beta_{n-1}(\delta)}{\sigma}.
    \end{align*}
Then we have for all \(n\geq 0\), with probability \(1-\delta\) 
\begin{align*}
    |\Jtr(\ppi,\mu_n) - \Jtr(\ppi,f)| & \leq \lambda_n \sum_{t=0}^\infty  \E_{\sppi,s_0}\left[\gamma^t \|\sigma_{n}(\tilde{s}_t,\tilde{a}_t)\|\right]\\
    |\Jtr(\ppi,\mu_n) - \Jtr(\ppi,f)| & \leq \lambda_n \sum_{t=0}^\infty  \E_{\sppi,s_0}\left[\gamma^t \|\sigma_{n}(s_t,a_t)\|\right].
\end{align*}
\end{lemma}
\begin{proof}
    We prove the first inequality \(|\Jtr(\ppi,\mu_n) - \Jtr(\ppi,f)| \leq \lambda_n \sum_{t=0}^\infty  \E_{\ppi,s_0}\left[\gamma^t \|\sigma_{n-1}(\tilde{s}_t,\tilde{a}_t)\|\right]\) and the second one holds for a analogous argument. We extend the argument of \citet[Lemma~A.5]{sukhija2025optimism} to account for the termination state \(s_{\dagger}\).
    \begin{align}
    |\Jtr(\ppi,\mu_n) - \Jtr(\ppi,f)| = \left|\E_{\ppi,s_0,\mu_n} \left[ \sum_{t=0}^\infty\gamma^{t+1}(V^{\ppi}_{r,f}(\tilde{s}_{t+1})-V^{\ppi}_{r,f}(s_{t+1}))\right]\right|, 
\end{align}
where \(V_{r,f}^\pi\) is the value  following \(\pi\) under dynamics \(f\), and we define the states \(\tilde{s}_{t+1}\) and \(s_{t+1}\) as
\begin{align}
    s_{t+1} & = \begin{cases}
        f(\tilde{s}_t,\tilde a_t) + \omega_t & \Phi(\tilde s_t,\tilde a_t,\ct,\Qcn) \geq d, \tilde s_t \neq s_\dagger, \\
        s_\dagger & \text{otherwise,}
    \end{cases}, \quad \tilde a_t \sim \ppi(\cdot | \tilde s_t),\\ \label{eq:modelPlanningDynamics}
     \tilde{s}_{t+1} & = \begin{cases}\mu_n(\tilde{s}_t,\tilde a_t) + \omega_t & \Phi(\tilde s_t,\tilde a_t,\ct,\Qcn) \geq d, \tilde s_t \neq s_\dagger, \\
    s_\dagger & \text{otherwise.}
    \end{cases}, \quad \tilde a_t \sim \ppi(\cdot | \tilde s_t).
\end{align}
Let \(\frac{(1+\sqrt{d_s})\beta_n(\delta)}{\sigma}R(s) = \left(V^\pi_{r,f}(s)\right)^2\) and with \(R(s) \leq \lambda_n\), if \(\ppi\) does not invoke the safe prior \(\bpi\), we obtain the following bound following \citet[Lemma~A.5]{sukhija2025optimism}
\begin{align}\label{eq:regretSafe}
     |\Jtr(\ppi,\mu_n) - \Jtr(\ppi,f)| \leq  & \sum_{t=0}^\infty \gamma \E_{\ppi,s_0}\bigg[\sqrt{\max\{\E_{\omega_t}[R(\tilde{s}_{t+1})],\E_{\omega_t}[R(s_{t+1})]\}} \\ \nonumber & \times \gamma^t \min\left\{\frac{\|s_{t+1} - \tilde{s}_{t+1}\|}{\sigma},1\right\}\bigg].\\
     \leq & \sum_{t=0}^\infty \gamma \E_{\ppi,s_0}\bigg[\sqrt{\max\{\E_{\omega_t}[R(\tilde{s}_{t+1})],\E_{\omega_t}[R(s_{t+1})]\}}\\ \nonumber & \times \gamma^t \min\left\{\frac{\|f(\tilde{s}_t,\tilde a_t)-\mu_n(\tilde{s}_t,\tilde a_t)\|}{\sigma},1\right\}\bigg]\\
     \leq & \frac{\bar R\gamma}{1-\gamma} \frac{(1+\sqrt{d_x})\beta_{n-1}(\delta)}{\sigma}\sum_{t=0}^\infty  \E_{\ppi,s_0}\left[\gamma^t \|\sigma_n(\tilde{s}_t,\tilde a_t)\|\right].
\end{align}
In case the agent invokes \(\bpi\) at some timestep \(t=k\), we formulate the regret as
\begin{align}
    |\Jtr(\ppi,\mu_n) - \Jtr(\ppi,f)| \labelrel\leq{step:decomposeSum}  & \sum_{t=0}^{k-1} \gamma \E_{\ppi,s_0}\bigg[\sqrt{\max\{\E_{\omega_t}[R(\tilde{s}_{t+1})],\E_{\omega_t}[R(s_{t+1})]\}} \\ \nonumber & \times \gamma^t \min\left\{\frac{\|s_{t+1} - \tilde{s}_{t+1}\|}{\sigma},1\right\}\bigg] + \gamma^k |V_r^{\bpi}(\tilde s_k) - \underline{V}_r^{\bpi}(\tilde s_k)| \\ \nonumber & + \sum_{t=k+1}^\infty \sqrt{\E_{\omega_t}[R(s_\dagger)]} \gamma^t \min\bigg\{\frac{\|s_\dagger - s_\dagger\|}{\sigma},1\bigg\}\\ 
    \labelrel={step:zeroDiff}  \label{eq:regretSwitch}& \sum_{t=0}^{k-1} \gamma \E_{\ppi,s_0}\bigg[\sqrt{\max\{\E_{\omega_t}[R(\tilde{s}_{t+1})],\E_{\omega_t}[R(s_{t+1})]\}} \\ \nonumber & \times \gamma^t \min\left\{\frac{\|s_{t+1} - \tilde{s}_{t+1}\|}{\sigma},1\right\}\bigg] + \gamma^k |V_r^{\bpi}(\tilde s_k) - \underline{V}_r^{\bpi}(\tilde s_k)|,\\
    \labelrel\leq{step:lemmaDiffVale} &\sum_{t=0}^{k-1} \gamma \E_{\ppi,s_0}\bigg[\sqrt{\max\{\E_{\omega_t}[R(\tilde{s}_{t+1})],\E_{\omega_t}[R(s_{t+1})]\}}\\\nonumber & \times \gamma^t \min\left\{\frac{\|f(\tilde{s}_t,\tilde a_t)-\mu_n(\tilde{s}_t,\tilde{a}_t)\|}{\sigma},1\right\}\bigg]\\\nonumber &
    +\frac{\bar R\gamma}{1-\gamma} \frac{2(1+\sqrt{d_x})\beta_{n-1}(\delta)}{\sigma}\sum_{t=k}^\infty \E_{\bpi,\tilde s_k}\left[\gamma^t \|\sigma_n(\tilde{s}_t,\tilde{a}_t))\|\right]\\
    \labelrel\leq{step:relateToUncertainty} & \frac{\bar R\gamma}{1-\gamma} \frac{(1+\sqrt{d_x})\beta_{n-1}(\delta)}{\sigma} \sum_{t=0}^{k-1} \E_{\ppi,s_0}[\gamma^t|\sigma_n(\tilde{s}_t,\tilde{a}_t)\|] \\\nonumber & + \frac{\bar R\gamma}{1-\gamma} \frac{2(1+\sqrt{d_x})\beta_{n-1}(\delta)}{\sigma}\sum_{t=k}^\infty \E_{\bpi,\tilde s_k}\left[\gamma^t \|\sigma_n(\tilde{s}_t,\tilde{a}_t)\|\right]\\
    \labelrel\leq{step:joinPolicies} &  \frac{\bar R\gamma}{1-\gamma} \frac{2(1+\sqrt{d_x})\beta_{n-1}(\delta)}{\sigma}\sum_{t=0}^\infty  \E_{\sppi,s_0}\left[\gamma^t \|\sigma_n(\tilde{s}_t,\tilde{a}_t)\|\right]. \label{eq:regretInter}
\end{align}
In step \ref{step:decomposeSum}, we split the infinite sum along the trajectory, when the policy prior \(\bpi\) is invoked. We therefore follow the trajectory along the model dynamics in \Cref{eq:modelPlanningDynamics}. Further, we eliminate the difference term for the terminal state \(s_\dagger\) in \ref{step:zeroDiff}. In step \ref{step:lemmaDiffVale}, we use \Cref{lemma:pess} to bound the value difference between the true value \(\Vr\) along \(f\) and the pessimistic value \(\pVr\). Using \citet[Corollary~3]{OptiActive} we bound the difference between the next states in \ref{step:relateToUncertainty}. Finally, in step \ref{step:joinPolicies}, we formulate the bound over \(\sppi\), allowing us to upper bound both terms with the uncertainty using the safely rolled out policy \(\sppi\) and abbreviate \(\lambda_n = \frac{\bar R\gamma}{1-\gamma} \frac{2(1+\sqrt{d_x})\beta_{n-1}(\delta)}{\sigma}\).
\end{proof}
\begin{lemma}\label{lemma:regretSecond}
    Suppose \Cref{assume:RKHS,assume:regularity,assume:noise} hold and the model is well-calibrated according to \cref{def:callib}. Given the optimal policy \(\opi\), then we obtain the  following per-episode regret on \(\SM\) in episode \(n\), 
    \(\PERb = \Jtr(\opi,f) - \Jtr(\pi,f)\) with \(\lambda_n = \frac{\bar R\gamma}{1-\gamma} \frac{2(1+\sqrt{d_x})\beta_{n-1}(\delta)}{\sigma}\) with probability \((1-\delta)\)
    \begin{align*}
        \PERb & =  \Jtr(\opi,f) - \Jtr(\ppi,f) \leq \Jtr(\opi,\mu_n) + \lambda_n \sum_{t=0}^\infty  \E_{\sopi,s_0}\left[\gamma^t \|\sigma_n(s_t,a_t)\|\right] - \Jtr(\ppi,f),
    \end{align*}
    where \(s_{t+1} = \mu_n(s_t,a_t)+\omega_t\)
\end{lemma}
\begin{proof}
    We can show this by applying \Cref{lemma:regretPlanning} to an optimal policy \(\opi\) on \(\SM\).
\end{proof}

\paragraph{Intrinsic rewards for exploration and expansion.}
Next, we show with \Cref{lemma:optPlanningBound,lemma:optPlanningUncertainty} that the policy \(\ppi\) in \Cref{eq:optiPlanning} achieves both optimism and expansion by upper-bounding the sum of both regret terms \(\Delta_n = \PERa + \PERb\).
\begin{lemma}\label{lemma:optPlanningBound}
    Given \Cref{assume:noise,assume:regularity,assume:RKHS,assume:safeGap} hold and the model \(\model_n\) is well-calibrated according to \Cref{def:callib}, the policy \(\ppi\) as defined in \Cref{eq:optiPlanning} 
    \begin{align*} \label{eq:optiPlanning}
    \begin{aligned}
            \ppi = & \arg\max_\pi \E_{\pi}\left[\sum_{t=0}^\infty \gamma^t\left(\tilde{r}(s_t,a_t) + (\lambda_{\text{explore}}+\lambda_{\text{expand}})\|\sigma_n(s_t,a_t)\|\right)\right],\\
        &~s_{t+1} = \mu_n(s_t,a_t ) + \omega_t, \quad a_t \sim \pi(\cdot | s_t), \quad s_0 \sim \rho_0(\cdot),
    \end{aligned}
    \end{align*}
    with \(\lambda_{\text{explore}} = 3\lambda_n\) (\Cref{lemma:regretSecond}) and 
    \(\lambda_{\text{expand}} = 3\eta_n\) (\Cref{lemma:firstRegretTerm}), satisfies with probability \(1-\delta\):
    \begin{align*}
     \Delta_n = &\PERa + \PERb \leq \eta_n \sum_{t=0}^\infty \E_{\ocpi,s_0}[\gamma^t\|\sigma_n(s_t,a_t)\|] + \Jtr(\opi,\mu_n)\\& + \lambda_n \sum_{t=0}^\infty  \E_{\sopi,s_0}\left[\gamma^t \|\sigma_n(s_t,a_t)\|\right] - \Jtr(\ppi,f)  \\
    \leq & \sum_{t=0}^\infty \E_{\ppi,s_0}[\gamma^t(\lambda_{\text{explore}} + \lambda_{\text{expand}})\|\sigma_n(s_t,a_t)\|] + \Jtr(\ppi,\mu_n)- \Jtr(\ppi,f).
    \end{align*}
\end{lemma}
\begin{proof}\label{proof:rlemmaSeven}
  Recall that by construction \(\ppi\) maximizes
  \begin{equation}
    J(\pi)=\E_\pi \left[\sum_{t=0}^\infty \gamma^t\,\tilde r\bigl(s_t,a_t)\bigr)+\sum_{t=0}^\infty\gamma^t\Bigl(3\eta_n+3\lambda_n\Bigr)\|\sigma_n(s_t,a_t))\|\right].
  \end{equation}
  Hence
  \begin{equation}\label{eq:opt-J}
    J(\ppi)\geq J(\pi)\quad\forall\pi.
  \end{equation}
  Apply \Cref{eq:opt-J} first to the policy
  \(\socpi\) (i.e.\ the policy induced by using \(\ocpi\) and then invoking the safe prior policy \(\bpi\)). By the linearity of expectations and
 non-negativity of all terms,
  \begin{align}
     J(\ppi) \geq&  \Jtr(\socpi,\mu_n) + \E_{\socpi,s_0}\left[(3\eta_n  + 3\lambda_n) \sum_{t=0}^{\infty} \gamma^t \|\sigma_n(s_t,a_t)\|\right]. 
    \label{eq:bound-hybrid}
  \end{align}
  Next apply \Cref{eq:opt-J} to the optimal policy \(\opi\) on \(\model_n\), yielding
  \begin{equation}\label{eq:bound-planner}
    J(\ppi)\ge \Jtr(\opi,\mu_n)+ \E_{\opi,s_0}\left[(3\eta_n  + 3\lambda_n) \sum_{t=0}^{\infty} \gamma^t \|\sigma_n(s_t,a_t)\|\right].
  \end{equation}
  Adding \Cref{eq:bound-hybrid} and \Cref{eq:bound-planner} results in
  \begin{align}\label{eq:addedTerms}
       2 J(\ppi) \geq &  \Jtr(\socpi,\mu_n) + \E_{\socpi,s_0}\left[(3\eta_n  + 3\lambda_n) \sum_{t=0}^{\infty} \gamma^t \|\sigma_n(s_t,a_t)\|\right] \\ \nonumber
         & + \Jtr(\opi,\mu_n)+ \E_{\opi,s_0}\left[(3\eta_n  + 3\lambda_n) \sum_{t=0}^{\infty} \gamma^t \|\sigma_n(s_t,a_t)\|\right].
  \end{align}
  Let us derive a bound for \(\Jtr(\socpi,\mu_n)\), that we can relate to the optimal policy \(\opi\) for \(\SM\).
  \begin{align}
        \Jtr(\opi,\mu_n) \labelrel\leq{step:planning} & \Jtr(\opi,f) + \lambda_n \E_{\opi,s_0}\left[\sum_{t=0}^{\infty} \gamma^t \|\sigma_n(s_t,a_t)\|\right] \\
        \labelrel\leq{step:SMdef} &\Jr(\sopi,f) + \lambda_n \E_{\sopi,s_0}\left[\sum_{t=0}^{\infty} \gamma^t \|\sigma_n(s_t,a_t)\|\right] \\
        \labelrel\leq{step:opti} &\Jr(\ocpi,f) + \lambda_n \E_{\sopi,s_0}\left[\sum_{t=0}^{\infty} \gamma^t \|\sigma_n(s_t,a_t)\|\right] \\
        \labelrel\leq{step:firstRegret} &\Jr(\socpi,f) + \lambda_n \E_{\sopi,s_0}\left[\sum_{t=0}^{\infty} \gamma^t \|\sigma_n(s_t,a_t)\|\right] \\ \nonumber& + \eta_n \E_{\socpi,s_0}\left[\sum_{t=0}^{\infty} \gamma^t \|\sigma_n(s_t,a_t)\|\right].
  \end{align}
  Step \ref{step:planning}  follows from \Cref{lemma:regretPlanning} and \ref{step:SMdef} is derived from the definition of \(\SM\) in \Cref{eq:terminalTransition,eq:terminalReward}. Next, step \ref{step:opti} uses the optimality of \(\ocpi\) and finally step \ref{step:firstRegret} uses the upper bound of the first regret term \(\PERa\) in \Cref{lemma:firstRegretTerm}. Hence, we can lower bound parts of the term by \(\Jtr(\opi,\mu_n)\).
  Therefore, \(2J(\ppi)\) in \Cref{eq:addedTerms} is lower bounded by
  \begin{align}
        2 J(\ppi) \geq & 2\Jtr(\opi,\mu_n) + 2\eta_n \E_{\socpi,s_0}\left[\sum_{t=0}^{\infty} \gamma^t \|\sigma_n(s_t,a_t)\|\right]\\\nonumber & + 3\lambda_n \E_{\socpi,s_0}\left[\sum_{t=0}^{\infty} \gamma^t \|\sigma_n(s_t,a_t)\|\right] + 3\eta_n \E_{\sopi,s_0}\left[\sum_{t=0}^{\infty} \gamma^t \|\sigma_n(s_t,a_t)\|\right]\\\nonumber & + \E_{\sopi,s_0}\left[ 2\lambda_n \sum_{t=0}^\infty \gamma^t \|\sigma_n(s_t,a_t)\|\right] .
    \end{align}
    Since all terms are non-negative, we can drop the terms \(3(\lambda_n + \eta_n) \E_{\socpi,s_0}\left[\sum_{t=0}^{\infty} \gamma^t \|\sigma_n(s_t,a_t)\|\right]\) and obtain by dividing both sides by two:
    \begin{align}
        2 J(\ppi) \geq & 2 \Jtr(\opi,\mu_n) + 2 \lambda_n \E_{\sopi,s_0}\left[\sum_{t=0}^{\infty} \gamma^t \|\sigma_n(s_t,a_t)\|\right] \\\nonumber  & + \E_{\socpi,s_0}\left[ 2\eta_n \sum_{t=0}^\infty \gamma^t \|\sigma_n(s_t,a_t)\|\right]\\
        J(\ppi) \geq \label{eq:optiPolicyBound}&  \Jtr(\opi,\mu_n) +  \lambda_n \E_{\sopi,s_0}\left[\sum_{t=0}^{\infty} \gamma^t \|\sigma_n(s_t,a_t)\|\right] \\\nonumber  & + \E_{\socpi,s_0}\left[\eta_n \sum_{t=0}^\infty \gamma^t \|\sigma_n(s_t,a_t)\|\right].
    \end{align}
    By definition of \(J(\ppi)\) and \Cref{eq:optiPolicyBound}, subtracting \(\Jtr(\ppi,f)\) completes the proof.
\end{proof}
In \Cref{lemma:optPlanningBound} we show that the objective upper bounds the performance and the per-episode regret on the \emph{model dynamics}. Next, we show in \Cref{lemma:optPlanningUncertainty} that we can also bound the regret on the \emph{true dynamics} \(f\) using \(\sppi\).
\begin{lemma}\label{lemma:optPlanningUncertainty}
    Given \Cref{assume:RKHS,assume:noise,assume:regularity,assume:safeGap} hold and the model \(\model_n\) is well-calibrated according to \Cref{def:callib}, the per-episode regret in \(n\), \(\Delta_n = \PERa + \PERb\), is upper-bounded with probability of at least \(1-\delta\) for all \(n>0\) by
    \begin{equation*}
        \Delta_n \leq (3\lambda_n^2 + 4\lambda_n + 3\eta_n^2 + 3\eta_n)) \sum_{t=0}^\infty \E_{\sppi,s_0}[\gamma^t \|\sigma_n(s_t,a_t)\|],
    \end{equation*}
    for \(s_{t+1} = f(s_t,a_t) + \omega_t\) following the true dynamics with \(a_t \sim \sppi(\cdot|s_t,\ct,\Qcn)\).
\end{lemma}
\begin{proof}\label{proof:rlemmaEight}
Using the derivation in \Cref{lemma:optPlanningBound}, we can upper bound the regret in terms of \(\ppi\) by
\begin{align}
    \Delta_n \labelrel\leq{step:lemma12} &\sum_{t=0}^\infty \E_{\ppi,s_0}[\gamma^t(\lambda_{\text{explore}} + \lambda_{\text{expand}})\|\sigma_n(s_t',a_t')\|] + \Jtr(\ppi,\mu_n)- \Jtr(\ppi,f)\\
    \labelrel\leq{step:uncertaintyRelate} & 3(\eta_n+\lambda_n) \E_{\sppi,s_0}\left[\sum_{t=0}^\infty \gamma^t \|\sigma_n(s_t',a_t')\|\right] + \Jtr(\ppi,\mu_n) - \Jtr(\ppi,f),
\end{align}
where \(s_t'=\mu_n(s_t',a_t')+\omega_t\) follows the true dynamics with \(a_t'\sim \ppi(\cdot|s_t')\) and we insert \Cref{lemma:optPlanningBound} in \ref{step:lemma12}. Step \ref{step:uncertaintyRelate} follows from the uncertainty along the safely executed policy \(\sppi\) being larger than for \(\ppi\), which terminates in \(s_\dagger\). 
We next use \Cref{lemma:regretPlanning} to bound \(\Jtr(\ppi,\mu_n) - \Jtr(\ppi,f)\) and obtain
\begin{align}
    \Delta_n \leq & 3(\eta_n+\lambda_n) \E_{\sppi,s_0}\left[\sum_{t=0}^\infty \gamma^t \|\sigma_n(s_t',a_t')\|\right] +  \lambda_n \E_{\sppi,s_0}\left[\sum_{t=0}^\infty \gamma^t \|\sigma_n(s_t,a_t)\|\right] \\
    \labelrel={step:zeroUncertainties} & 3\eta_n \E_{\sppi,s_0}\left[\sum_{t=0}^\infty \gamma^t \|\sigma_n(s_t,a_t)\|\right] + 3\eta_n \bigg(\E_{\sppi,s_0}\left[\sum_{t=0}^\infty \gamma^t \|\sigma_n(s_t',a_t')\|\right] \label{eq:extended} \\ \nonumber  &- \E_{\sppi,s_0}\left[\sum_{t=0}^\infty \gamma^t\|\sigma_n(s_t,a_t)\|\right]\bigg) 
    + 3\lambda_n\bigg( \E_{\sppi,s_0}\left[\sum_{t=0}^\infty \gamma^t \|\sigma_n(s_t',a_t')\|\right] \\ \nonumber  & - \E_{\sppi,s_0}\left[\sum_{t=0}^\infty \gamma^t\|\sigma_n(s_t,a_t)\|\right]\bigg) +  4\lambda_n \E_{\sppi,s_0}\left[\sum_{t=0}^\infty \gamma^t \|\sigma_n(s_t,a_t)\|\right],
\end{align}
where step \ref{step:zeroUncertainties} introduces zero terms by adding and subtracting accumulated uncertainties. 
We further bound the per episode regret using non-negativity  
\begin{align}
    \Delta_n
     \leq & (3\lambda_n^2 + 4\lambda_n + 3\eta_n^2 + 3\eta_n) \E_{\sppi,s_0}\left[\sum_{t=0}^\infty \gamma^t \|\sigma_n(s_t,a_t)\|\right].
\end{align}
\end{proof}

In \Cref{lemma:optPlanningUncertainty}, we derived an upper bound on the per episode regret \(\Delta_n\), based on the policy \(\sppi\) executed on the true dynamics \(s_{t+1} = f(s_t,a_t) + \omega_t\).
Using \citet{Chowdhury}, we obtain a sublinear cumulative regret for our safe exploration algorithm \algname{SOOPER}.
\rthrm*
\begin{proof}\label{proof:rthrm}
Given the per episode regret bound in \Cref{lemma:optPlanningUncertainty}, we sum over all episodes \(n=1,\dots,N\)
\begin{align}
R(N) = \sum_{n=1}^N \PERa + \PERb \leq \sum_{n=1}^N (3\lambda_n^2 + 4\lambda_n + 3\eta_n^2 + 3\eta_n) \sum_{t=0}^\infty \E_{\sppi,s_0}\left[\gamma^t \|\sigma_n(s_t, a_t)\|\right].
\end{align}
As by definition \cite{Chowdhury}, for the well-calibrated model \(\model_n\), \(\beta_n \leq \beta_N\) for all \(n=1,\dots,N\), we rewrite
\begin{align}
R(N) \leq & (4\lambda_N + 3(\lambda_N^2 +\eta_N^2 +\eta_N))  \sum_{n=1}^N \sum_{t=0}^\infty \E_{\sppi,s_0}\left[\gamma^t \|\sigma_n(s_t, a_t)\|\right]\\
\labelrel\leq{step:CauchySchwarz1} & ((4\lambda_N + 3(\lambda_N^2 +\eta_N^2 +\eta_N)) \sqrt{N} \sqrt{\sum_{n=1}^N \E_{s_0}\bigg[\bigg(\sum_{t=0}^\infty \E_{\sppi}\left[\gamma^t \|\sigma_n(s_t, a_t)\|\right]\bigg)^2\bigg]}\\
\labelrel\leq{step:CauchySchwarz2} & (4\lambda_N + 3(\lambda_N^2 +\eta_N^2 +\eta_N)) \sqrt{\frac{N}{1-\gamma}} \sqrt{\sum_{n=1}^N \E_{s_0}\bigg[\sum_{t=0}^\infty \E_{\sppi}\left[\gamma^t \|\sigma_n(s_t, a_t)\|^2\right]\bigg]}.
\end{align}
 We use \emph{Cauchy-Schwarz} in step \ref{step:CauchySchwarz1} for the outer sum over episodes \(n=1,\dots,N\) and in \ref{step:CauchySchwarz2} for the inner sum over timesteps \(t\). Further, we can bound the regret term for episodes \(n=1,\dots, N\), using the truncated horizon \(T_n = -\frac{\log(n)}{\log(\gamma)}\), by decomposing the term into
 \begin{align}
     & \sum_{n=1}^N \E_{s_0}\bigg[\sum_{t=0}^\infty \E_{\sppi}\left[\gamma^t\|\sigma_n(s_t, a_t)\|^2\right]\bigg]\\
     & = \sum_{n=1}^N \left(\sum_{t=0}^{T_n-1} \E_{\sppi}\left[\gamma^t \|\sigma_n(s_t, a_t)\|^2\right] +\sum_{t=T_n}^{\infty} \E_{\sppi}\left[\gamma^t \|\sigma_n(s_t, a_t)\|^2\right]\right)\\
     & \leq \sum_{n=1}^N \sum_{t=0}^{T_n-1} \E_{\sppi}\left[\gamma^t\|\sigma_n(s_t, a_t)\|^2\right] +\sum_{t=1}^{N} \gamma^{T_n}\frac{k_{max}^2d_{\set{S}}}{1-\gamma}\\
     & \labelrel\leq{step:exponent-log-trick} \sum_{n=1}^N \sum_{t=0}^{T_n-1} \E_{\sppi}\left[\gamma^t \|\sigma_n(s_t, a_t)\|^2\right] +\sum_{n=1}^{N} \frac{1}{n}  \frac{k_{max}^2d_{\set{S}}}{1-\gamma} \\
     & \labelrel\leq{step:theorem931} \sum_{n=1}^N \sum_{t=0}^{T_n-1} \E_{\sppi}\left[\gamma^t \|\sigma_n(s_t, a_t)\|^2\right] + \frac{k_{max}^2d_{\set{S}}}{1-\gamma}(\log(N) + 1).
 \end{align}
Step \ref{step:exponent-log-trick} follows from \(\smash{T_n=-\frac{\log(n)}{\log(\gamma)}}\), with \(\smash{\frac{\log(n)}{\log(\gamma)}} = \log_\gamma(n)\) and hence \(\gamma^{-\log_\gamma(n)} = n^{-1}\). In \ref{step:theorem931}, we use the fact that \(\frac{1}{n}\) is non-decreasing in \(n\) and therefore due to \citet[Theorem~3.9.1]{Lehman2017mathematics} we bound \(\sum_{n=1}^N \frac{1}{n} \leq \log(N) + 1\). Further, we bound the first term \(\sum_{n=1}^N \sum_{t=0}^{T_n-1} \E_{\sppi}\left[\gamma^t \|\sigma_n(s_t, a_t)\|^2\right]\) using \citet[theorem~5.7]{sukhija2025optimism} with the martingale in \cite[lemma~4]{vignola2026sampling} to bound the expected uncertainties with the executed trajectories
\begin{align}
\leq & (4\lambda_N + 3(\lambda_N^2 +\eta_N^2 +\eta_N)) \sqrt{\frac{R_\gamma N \Gamma_{N\log(N)}}{1-\gamma}+ \frac{R_\gamma \sigma_{max}^2N(\log(N)+1)}{1-\gamma}},
\end{align}
where \(\smash{R_\gamma = \frac{s_{max}}{\log(1+s_{max})}}\), with \(\smash{s_{max} = \frac{\sigma^{-2}d_{\set{S}}\sigma_{max}^2}{1-\gamma}}\).
As \(T_n\) is non-decreasing, we fix \(T_N\) and since \(\lambda_N\propto \frac{\beta_N}{1-\gamma}\), \(\eta_n \propto \frac{\beta_N^3}{1-\gamma}\), we obtain from the definition of \(\beta_N\) from \citet{Chowdhury} that 
\begin{equation}
    R_N \leq \mathcal{O}\left(\Gamma^{7/2}_{N\log(N)}\sqrt{N}\right).
\end{equation}
\end{proof}

\section{Theoretical Extension to World Models} \label{sec:extension-world-models}
For our theoretical analysis in \Cref{sec:inter,sec:regret-analysis}, we assume, for simplicity in presentation, that the reward and cost functions are known. All arguments can be extended to the more general setting of unknown costs and rewards, which are learned as part of a ``world model'', in addition to the unknown dynamics $f$. In practice, we consider this general setting in our experiments (\Cref{sec:experiments}). Below we briefly illustrate how one can extend the analysis above to this more general case.

First, one needs to impose an additional RKHS assumption on the cost and reward functions, analogous to \Cref{assume:RKHS} for the unknown dynamics \(f\). Accordingly, we jointly learn with the unknown dynamics \(f\) a calibrated model for the reward and cost function with nominal predictions \(\mu_n^r, \mu_n^c\) and epistemic uncertainty \(\sigma_n^r, \sigma_n^c\). Hence, it holds \(\forall (s,a) \in \set{S}\times\set{A}: r(s,a) \in \{r' \mid | r' - \mu^r_n| \leq \beta_n\sigma_n^r\}\) and \(\forall (s,a) \in \set{S}\times\set{A} : c(s,a) \in \{c' \mid | c' - \mu^c_n| \leq \beta_n\sigma_n^c\}\). 

Given only access to the calibrated model predictions, we modify the reward and cost structure on the planning MDP \(\SM\) as follows:
\begin{align}
    \tilde{r}(s_t,a_t) &\eqdef \begin{cases}
        \underline{V}^{\bpi}_r(s_t) & \text{if } \Phi(a_t,s_t,c_{<t},\Qcn) \geq d,\\
        0 & \text{if } s_t =  s_\dagger, \\
        \mu_n^r(s_t,a_t) + \beta_n \sigma_n^r(s_t,a_t) & \text{otherwise},
    \end{cases}\label{eq:modified_reward}\\
    \tilde{c}(s_t,a_t) &\eqdef \mu_n^c(s_t,a_t) + \beta_n \sigma_n^c(s_t,a_t). \label{eq:modified_cost}
\end{align}
By utilizing the above cost and reward, we demonstrate safe exploration and achieve sublinear cumulative regret, even in the presence of unknown cost and reward functions. 

\paragraph{Safety guarantee for unknown costs.}
Under the modified cost $\tilde c$ in \Cref{eq:modified_cost}, the update of the pessimistic Q-value is given by:
\begin{align} \label{eq:modified-cost-value}
    \Qtcn(s, a) \eqdef \E_{\bpi} \left[\sum_{t=0}^\infty\gamma^t (\tilde c(s_t,a_t) + \lambda_{\text{pessimism}} \|\sigma_n(s_t,a_t)\|) ~\bigg|~ s_0=s, a_0=a\right].
\end{align}
For this modified update using \(\tilde c\), it is possible to show that \Cref{slemma:intrinsicBound} holds using uncertain predictions. One can derive it by bounding the difference between the value \(\tVc(s)\) following the model dynamics with the true cost function, and the value \(\tVtc(s)\) following the model dynamics with the mean predictions
\begin{align}
    |\tVc(s) - \tVtc(s)| & = \left|\E_{\bpi} \left[\sum_{t=0}^\infty \gamma^t (c(s_t,a_t)-\mu_n^c(s_t,a_t))\right] \mid s_0=s\right| \\
    & \leq \E_{\bpi} \left[\sum_{t=0}^\infty \gamma^t \left|c(s_t,a_t)-\mu_n^c(s_t,a_t)\right|\mid s_0=s\right] \\
    & \leq \E_{\bpi} \left[\sum_{t=0}^\infty \gamma^t \beta_n \sigma_n^c(s_t,a_t)\mid s_0=s\right],\\
    ~& s_{t+1} = \mu_n(s_t,a_t) + \omega_t, \quad s_0 = s.
\end{align}
Hence, by bringing the actually observable value \(\tVtc(s)\) to the right side, it follows that
\begin{align}
    \tVc(s) &\leq \tVtc(s) + \E_{\bpi} \left[\sum_{t=0}^\infty \gamma^t \beta_n \sigma_n^c(s_t,a_t)\mid s_0=s \right]\\
    & = \E_{\bpi} \left[\sum_{t=0}^\infty \gamma^t( \mu_n^c(s_t,a_t) + \beta_n \sigma_n^c(s_t,a_t))\mid s_0=s\right]\\
    & = \E_{\bpi} \left[\sum_{t=0}^\infty \gamma^t\tilde c(s_t,a_t) \mid s_0=s \right], \label{eq:cost-bound}\\
    ~& s_{t+1} = \mu_n(s_t,a_t) + \omega_t, \quad a_t = \bpi(s_t).
\end{align}
Using the derived upper bound and plugging it into \Cref{eq:upperBoundCVp1} yields
\begin{align}
    \pVc(s) &\leq \E_{\bpi}\left[\sum_{t=0}^\infty \gamma^t\left(\tilde c(s_t,a_t) + \lambda_{\text{pessimism}} \|\sigma_n(s_t,a_t)\|\right) \mid s_0=s \right],\\
        ~  s_{t+1}&=\mu_n(s_t,a_t) + \omega_t, \quad a_t = \bpi(s_t).
\end{align}
Hence, the proof of \Cref{theo:safety} holds using \(\Qtcn(s,a)\) from \Cref{eq:modified-cost-value} in \Cref{eq:Qcn-bound}. 
We can also upper-bound this in a more compact form by assuming the same kernel \(k\) for the cost and the dynamics prediction, so that we have \(\beta_n\sigma_n^c(s_t,a_t) + \lambda_{\text{pessimism}}\|\sigma_n(s_t,a_t)\| \leq  (\lambda_{\text{pessimism}} + \beta_n) \|\sigma_n^{f,c}(s_t,a_t)\|\). This follows from the non-negativity of \(\sigma\) and with \(\sigma_n^{f,c}\) being defined as the concatenation of the prediction uncertainties for the dynamics \(\sigma_n\) and the cost \(\sigma_n^c\).

\paragraph{Regret bound for unknown reward and cost function.}
Similar to the modified cost update in \Cref{eq:modified-cost-value}, we use \(\tilde r\) in \Cref{eq:modified_reward}, and optimize for
\begin{align}\label{eq:modeified-objective}
    \ppi = & \arg\max_\pi \E_{\tilde{p}, \pi, s_0}\left[\sum_{t=0}^\infty \gamma^t\left(\tilde{r}(s_t,a_t) + (\lambda_{\text{explore}}+\lambda_{\text{expand}})\|\sigma_n(s_t,a_t)\|\right)\right],\\
        &~s_{t+1} = \mu_n(s_t,a_t ) + \omega_t, \quad a_t \sim \pi(\cdot | s_t).
\end{align}
For the derivation of the regret bound, the first regret term is affected by enlarging the pessimistic additive cost with \(\lambda_{\text{pessimism}}+\beta_n\). This contributes an additive \(\beta_n\) term, which does not change the proportionality of \(\eta_n \propto \frac{\beta_n^3}{1-\gamma}\). For the second regret term, we obtain an additional term that quantifies the uncertainty in the reward prediction, similar to the derivation in \Cref{eq:cost-bound}. Hence, we obtain the additive term in \Cref{lemma:regretSecond} while following the mean dynamics and using the model's mean predictions \(\mu_n^r\).
\begin{align}
    \PERb =&  \Jtr(\opi,f) - \Jtr(\ppi,f) \\ 
    \leq& \Jtr(\opi,\mu_n) + \lambda_n \sum_{t=0}^\infty  \E_{\sopi,s_0}\left[\gamma^t \|\sigma_n(s_t,a_t)\|\right] - \Jtr(\ppi,f)\\
    \leq& \sum_{t=0}^\infty \E_{\sopi} \left[\gamma^t( \mu_n^r(s_t,a_t) + \beta_n \sigma_n^r(s_t,a_t))\mid s_0=s\right]\\ &+ \lambda_n \sum_{t=0}^\infty  \E_{\sopi,s_0}\left[\gamma^t \|\sigma_n(s_t,a_t)\|\right] - \Jtr(\ppi,f).
\end{align}
By definition of the modified objective in \Cref{eq:modeified-objective}, we still upper bound in \Cref{lemma:optPlanningBound} the regret with the behavioral policy \(\ppi\), so that the following proof steps still hold. Further, we can again assume that the dynamics and the reward are part of the same kernel, resulting in the bound \(\beta_n\sigma_n^r(s_t,a_t) + \lambda_{n}\|\sigma_n(s_t,a_t)\| \leq  (\lambda_{n} + \beta_n) \|\sigma_n^{f,r}(s_t,a_t)\|\), where \(\sigma_n^{f,r}\) is defined as the concatenation of the prediction uncertainties of the dynamics \(\sigma_n\) and the reward \(\sigma_n^r\). Thereby, also for the second regret term, the proportionality remains \(\lambda_n \propto \frac{\beta_n}{1-\gamma}\). Consequently, the regret bound remains unchanged \(\smash{R_N \leq \mathcal{O}\left(\Gamma^{7/2}_{N\log(N)}\sqrt{N}\right)}\) for unknown reward and cost functions.

\section{Additional Experiments}
\label{sec:experiments-appendix}
We conduct further empirical analysis of \algname{SOOPER}, focusing on the following key aspects: 
\begin{enumerate*}[label=\textbf{(\roman*)}]
    \item demonstrating the role of prior policies in accelerating learning;
    \item leveraging terminations for safe policy improvement;
    \item ablating the role of \Cref{algo:safeRollout};
    \item robustness of \algname{SOOPER} to prior policies with varying pessimism levels;
    \item the necessity of acting pessimistically for safety, and
    \item the influence of optimism on the agent's regret.
\end{enumerate*}
The goal of these ablations is to provide additional empirical evidence to support the different design choices of \algname{SOOPER}.

\paragraph{Safe transfer under task misspecification.}
\looseness=-1
We demonstrate the effectiveness of policy priors compared to ``backup'' policies that are commonly considered in previous works. For this, we reduce the size of the goal in the PointGoal1 task from SafetyGym (see \Cref{fig:safety-gym}) from $0.3$ to $0.15$. This makes learning harder, as additional exploration is required to find the goal. We compare the following methods:
\begin{enumerate*}[label=\textbf{(\roman*)}]
    \item \algname{SOOPER} that uses a policy trained on the original $0.3$ goal task as a safe prior;
    \item \algname{SAILR} that uses a hand-crafted ``backup'' policy that brakes the robot upon triggering, following \algname{SAILR}'s ``advantage-based filtering''.
    The policy $\ppi$ in this case is initialized with the same policy prior that \algname{SOOPER} uses;
    \item \algname{SOOPER} that uses the policy prior only within \Cref{algo:safeRollout}, but learns its policy $\ppi$ completely from scratch;
    \item finally, a baseline that uses the hand-crafted braking backup policy, falling back to it as done in in \algname{SAILR}, while training the policy $\ppi$ from scratch.
\end{enumerate*}
\Cref{fig:task-transfer} demonstrates two results: 
\begin{enumerate*}[label=\textbf{(\alph*)}]
    \item as expected, both baselines that learn from scratch fail to converge within the training budget of 200 episodes.
    \item On the other hand, when finetuning the prior policy, using a hand-crafted braking ``backup'' policy indeed satisfies the constraint, however at the cost of slower learning compared to \algname{SOOPER}.
\end{enumerate*}
This result illustrates the importance of using policy priors not only to maintain safety, but also to guide and accelerate learning.

\begin{figure}[h]
    \centering
    \includegraphics{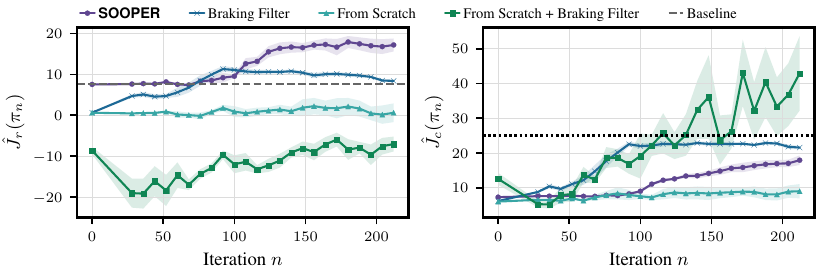}
    \caption{Learning curves of the objective and constraint for PointGoal1 goal size transfer. \algname{SOOPER} performs significantly better using the prior policy trained on a larger goal size, significantly helping in exploration when learning to reach a smaller goal.}
    \label{fig:task-transfer}
\end{figure}

\paragraph{Learning from terminations.}
We study the effect of learning from terminations on the agent's ability to reduce its dependency on the policy prior. To this end, we ablate the use of terminations in \Cref{eq:terminalTransition,eq:terminalReward} and compare \algname{SOOPER} with a variant that does not terminate in $\SM$.
In \Cref{fig:termination_ablations}, we observe a decrease in the incurred costs for \algname{SOOPER} with terminations, while \algname{SOOPER} without terminations still encounters high costs for both CartpoleSwingup and RaceCar. If the optimal unconstrained policy violates the safety budget, the agent trained without terminations fails to achieve the optimal safe performance. This failure occurs because, in the absence of termination signals, the agent receives no feedback about triggering the prior policy during online rollouts.
\begin{figure}[h]
    \centering
    \includegraphics{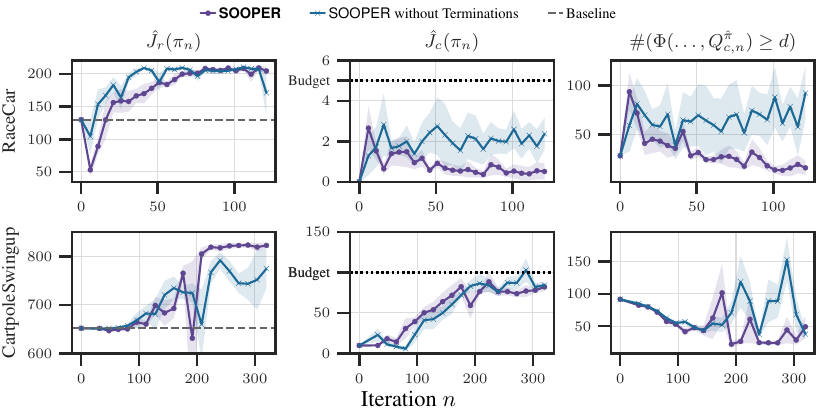}
    \caption{Learning curves of the objective, constraint, and number of times the safe prior is triggered during online rollouts. The number of times \algname{SOOPER} uses the prior shrinks over training.}
    \label{fig:termination_ablations}
\end{figure}

\paragraph{Robustness to policy priors.}
Different degrees of sim-to-real gap or lack of data coverage in the offline setting require increasing levels of conservatism.
We demonstrate that \algname{SOOPER} can effectively learn from policy priors with varying levels of conservatism and initial performance. For this, we train three prior policies on PointGoal2 for different sim-to-real gaps with sufficient pessimism levels and show that \algname{SOOPER} converges to optimal performance while satisfying the constraints using all three policy priors in \Cref{fig:abaltion_pi_prior}. These results highlight \algname{SOOPER}'s ability to safely explore in highly uncertain environments, given conservative policy priors.
\begin{figure}[h]
    \centering
    \includegraphics{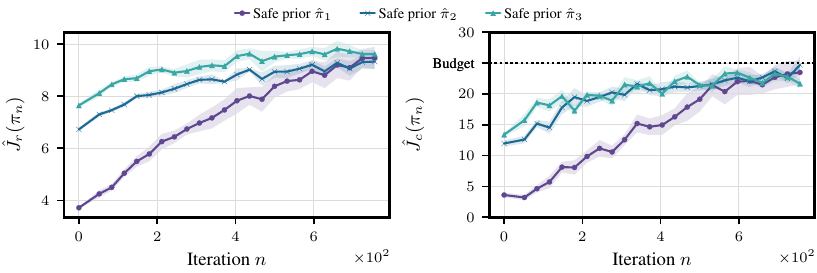}
    \caption{Learning curves of the objective and the constraint, starting from different safe policy priors with various levels of conservatism, differing in initial performances.}
    \label{fig:abaltion_pi_prior}
\end{figure}

\paragraph{Robustness to noise.}
Our problem formulation in \Cref{eq:dynamics} focuses on \emph{additive Gaussian} noise. While this modeling assumption is key to our theoretical development of \algname{SOOPER}, we now demonstrate that \algname{SOOPER} remains safe under substantial disturbances that deviate from the additive Gaussian assumption. In particular, we consider noise of the form $\tilde{a}_t = a_t + \eta_t, \; \eta_t \sim \text{Unif}(-\eta, \eta)$. Notably, this form of noise enters non-linearly to the dynamics $f$. To study this setting, we sweep over the interval of the disturbance distribution with $\eta \in \{0, 0.05, 0.1, 0.15, 0.2, 0.25\}$ for all tasks considered in \Cref{fig:sim-to-sim}. We use the same prior policies of the experiment in \Cref{fig:sim-to-sim}; these policies are trained in simulation using $\eta = 0$.
Note that for all tasks in our experiments $\mathcal{A} = [-1, 1]^{d_{\mathcal{A}}}$, hence $\eta = 0.25$ is a non-negligible noise. We present our results in \Cref{fig:noise-ablation} where we report the performance at the end of training and the largest average accumulated cost measured during training for all tasks. As shown, \algname{SOOPER} maintains safety without suffering performance drops, even under significant noise disturbances.
\begin{figure}[h]
    \centering
        \centering
        \includegraphics{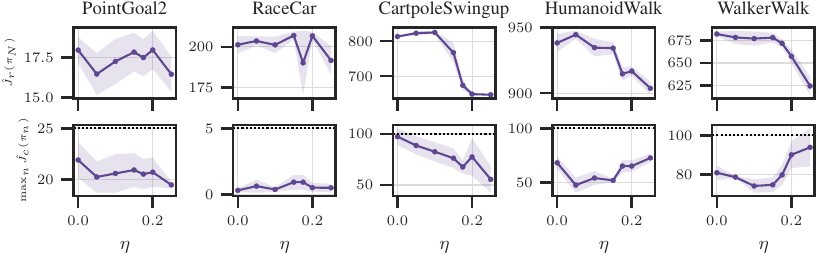}
    \caption{\algname{SOOPER} achieves good performance and satisfies the constraint under non-Gaussian, non additive noise.}
    \label{fig:noise-ablation}
\end{figure}
Crucially, these results show that even though $\pi_0$ is obtained under $\eta = 0$, \algname{SOOPER} can still satisfy the constraint, as long is $\pi_0$ is conservative enough. These results are in line with \Cref{theo:safety}: safety can be guaranteed for a very general class of noise structures as long as $\pi_0$ satisfies the constraint.

\paragraph{Pessimism for safe exploration.}
Given the previous ablation on conservatism of safe policy priors, we now study the scale of pessimism used \emph{during} learning.  As shown for CartpoleSwingup in \Cref{fig:pessimism_scaling}, using increased \(\lambda_{\text{pessimism}}\), the agent consistently satisfies the constraints without significant sacrifice in performance. Moreover, \algname{SOOPER} demonstrates robustness to increased \(\lambda_{\text{pessimism}}\), and showcases higher performance for lower chosen pessimism values. Crucially, this behavior is expected as it reflects an inherent tradeoff between pessimism and exploration.
\begin{figure}[h]
    \centering
        \centering
        \includegraphics{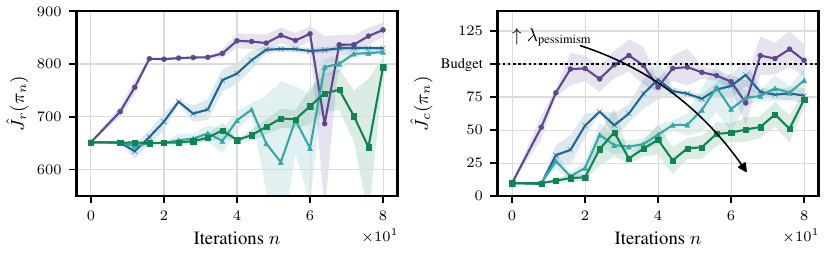}
    \caption{Learning curves of the objective and the constraint for \algname{SOOPER} with increasing \(\lambda_\text{pessimism}\). Pessimism improves robustness of constraint satisfaction.}
    \label{fig:pessimism_scaling}
\end{figure}

\paragraph{Optimistic planning for sublinear regret.}
\looseness=-1
We now study the influence of \algname{SOOPER}'s intrinsic reward on its performance, safety during learning and on the cumulative regret. Specifically, we study how varying $\lambda_{\text{expand}}$ together with $\lambda_{\text{explore}}$ improves \algname{SOOPER}'s performance in the CartpoleSwingup task. Intuitively, high values of $\lambda_{\text{expand}} + \lambda_{\text{explore}}$ may lead to excessive exploration, while small values of $\lambda_{\text{expand}} + \lambda_{\text{explore}}$ may not explore enough. As visible in \Cref{fig:optimism_scaling}, empirically, the cumulative regret \(\frac{1}{N}R(N)\) exhibits a single optimum, since for small optimism scales the agent under-explores and fails to expand the safe region, and for large optimism scale the agent over-explores, which reduces short-term greediness and lowers cumulative reward over the finite horizon. As \(N\) increases, short-term suboptimality due to over-exploration decreases and the cumulative reward becomes less sensitive to the optimism scale, consistent with the sublinear-regret guarantee for sub-linear growth in \(N\).
\begin{figure}[h]
    \centering
    \includegraphics{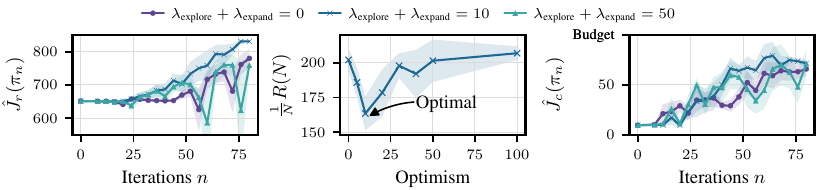}
    \caption{Learning curves of the objective and the constraint, as well as cumulative regret, for different optimism scales. A good value for the optimism value can be identified, where both under- and over-exploration are avoided and regret is minimized.}
    \label{fig:optimism_scaling}
\end{figure}

\paragraph{Safe offline-to-online on real hardware.}
We repeat our hardware experiment in \Cref{sec:experiments} on the race car however now use \algname{SOOPER} using an offline-trained policy. We collect 25K real-world transitions, corresponding to roughly ten minutes of real-time data. To collect this data, we use the simulation-trained policy used in our sim-to-real experiment as baseline. This dataset is then used to train a pessimistic policy using MOPO \citep{yu2020mopo} with a primal-dual solver for the constraint. We introduce pessimism by penalizing the reward and cost with model uncertainty \citep[see][]{yu2020mopo}. We repeat our hardware experiment with five random seeds, reporting the mean and standard error of the accumulated rewards and costs of each iteration $n$ in \Cref{fig:hardware-offline}. As shown, \algname{SOOPER} leverages the offline-trained policy prior to learn near-optimal policy while satisfying the constraint throughout learning.
\begin{figure}[h]
    \centering
    \includegraphics{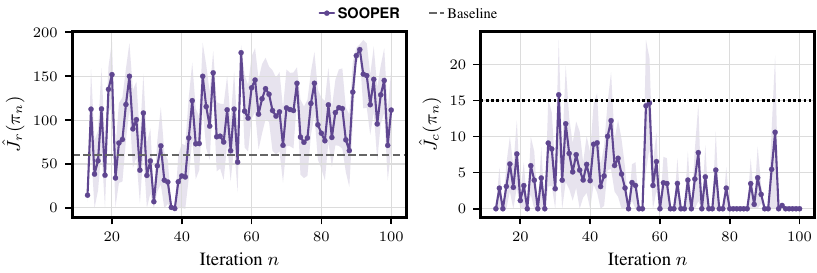}
    \caption{Performance and safety on the real race car given a prior policy that is trained via a fixed offline data. \algname{SOOPER} improves the performance while satisfying the constraint during learning.}
    \label{fig:hardware-offline}
\end{figure}

\paragraph{Learning directly from costs.}
While we demonstrate that \algname{SOOPER} can learn safety from an unconstrained MDP with termination, our method can be extended to learning directly from costs via model-based rollouts with any off-the-shelf CMDP solver. While this variant enjoys the safety guarantee of \Cref{theo:safety} using \Cref{algo:safeRollout}, it is not guaranteed to learn an optimal policy. We instantiate this variant with \algname{CRPO} and refer to this baseline as \algname{SafeCRPO}. The evaluations across various tasks show that \algname{SOOPER}'s \Cref{algo:safeRollout} can be paired with alternative CMDP solvers and ensures safe learning. However, \algname{SafeCRPO} consistently learns slower and achieves lower performance than \algname{SOOPER}. This demonstrates the advantage of \algname{SOOPER} using terminations paired with \Cref{algo:safeRollout} to learn safely and guarantee convergence.
\begin{figure}
    \centering
    \includegraphics{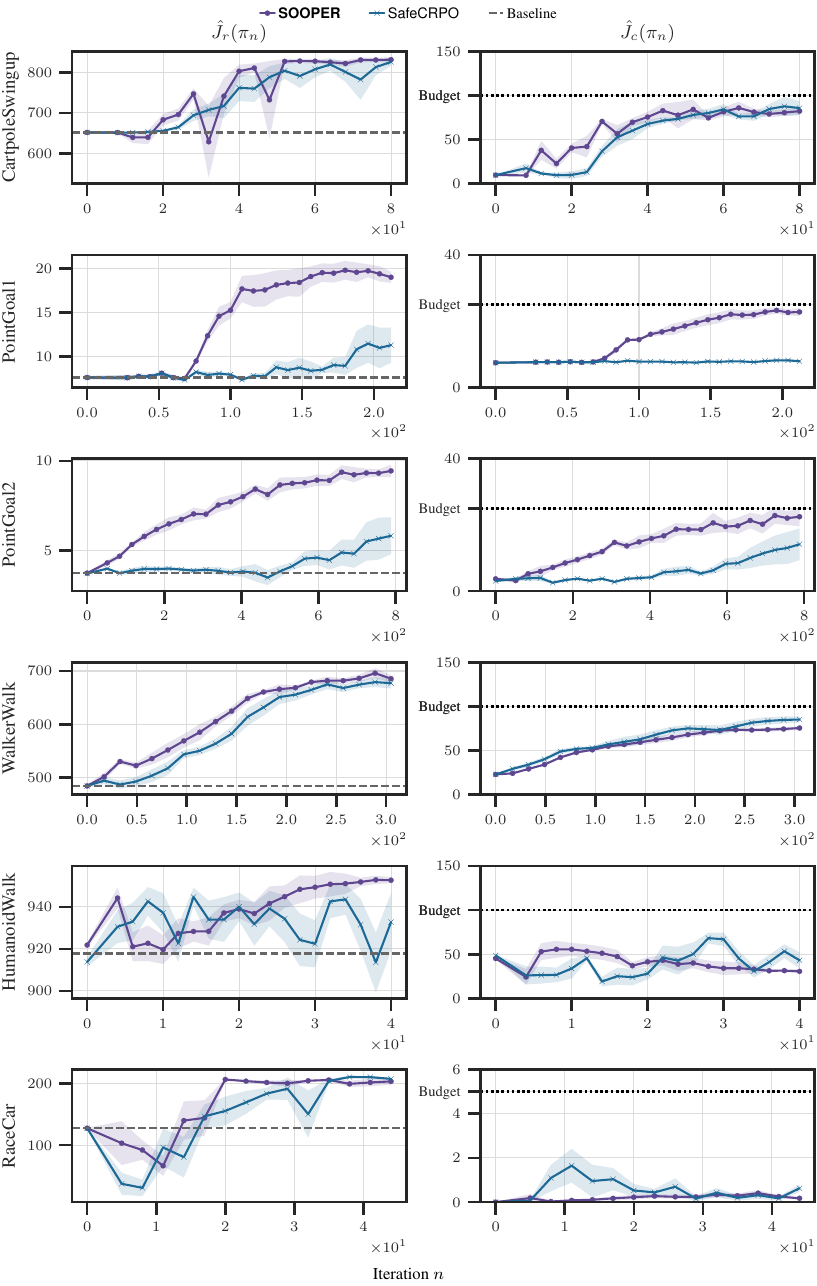}
    \caption{Learning curves of the objective and the constraint for \algname{SOOPER} and \algname{SafeCRPO} across all environments. \algname{SOOPER} consistently outperforms and shows asymptotic convergence.}
    \label{fig:ablation-cmdp}
\end{figure}

\paragraph{Runtime.}
\begin{wrapfigure}[15]{r}{0.36\textwidth}
\vspace{-1.4\baselineskip}
\begin{center}
    \includegraphics{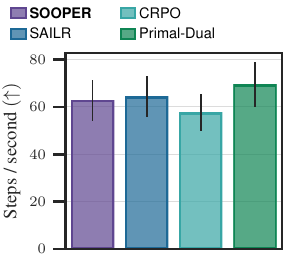}
\end{center}
\caption{
\looseness=-1The number of environment steps per second of \algname{SOOPER} compared to other baselines. 
}
\label{fig:wallclock}
\end{wrapfigure}
\looseness=-1Even though \algname{SOOPER} builds on MBPO \citep{mbpo} and therefore enjoys a relatively simple implementation, its use of \Cref{algo:safeRollout} and post-hoc relabeling of transitions with terminations deviate from the standard ``RL toolbox''.
Here, we study the effect of these components on the runtime of \algname{SOOPER} compared to other model-based baselines. Specifically, we record training throughput, namely, the total number of environment steps divided by the overall wall-clock time for the runs of the experiments in \Cref{fig:sim-to-sim}. In \Cref{fig:wallclock} we aggregate the number of steps per second across all tasks (using five seeds) and report the mean and standard deviation. As shown, even though \algname{SOOPER} has additional algorithmic components, it does not suffer from significant runtime overhead.

\paragraph{Additional learning curves.}
\looseness=-1
\Cref{fig:convergent-learning-behavior} complements \Cref{fig:sim-to-sim} in \Cref{sec:experiments} and shows the learning curves on all sim-to-sim tasks. \algname{SOOPER} consistently achieves near-optimal performance without violating the constraint during learning. In addition to the baselines presented in \Cref{sec:experiments}, we include another CMDP solver that is based on Log Barriers \citep{JMLR:v25:22-0878}. This solver is theoretically guaranteed to maintain feasibility during learning by penalizing iterates that approach the boundary of the feasible set. For all of our sim-to-sim experiments, we train the safe policy priors following SPiDR \citep{as2025spidrsimpleapproachzeroshot}. SPiDR leverages domain randomization to introduce the conservatism required for constraint satisfaction when deploying RL policies under distribution shifts. In particular, to obtain pessimistic policies, SPiDR penalizes the cost according to the disagreement in next-state predictions of the randomized dynamics.
\begin{figure}
    \centering
    \includegraphics{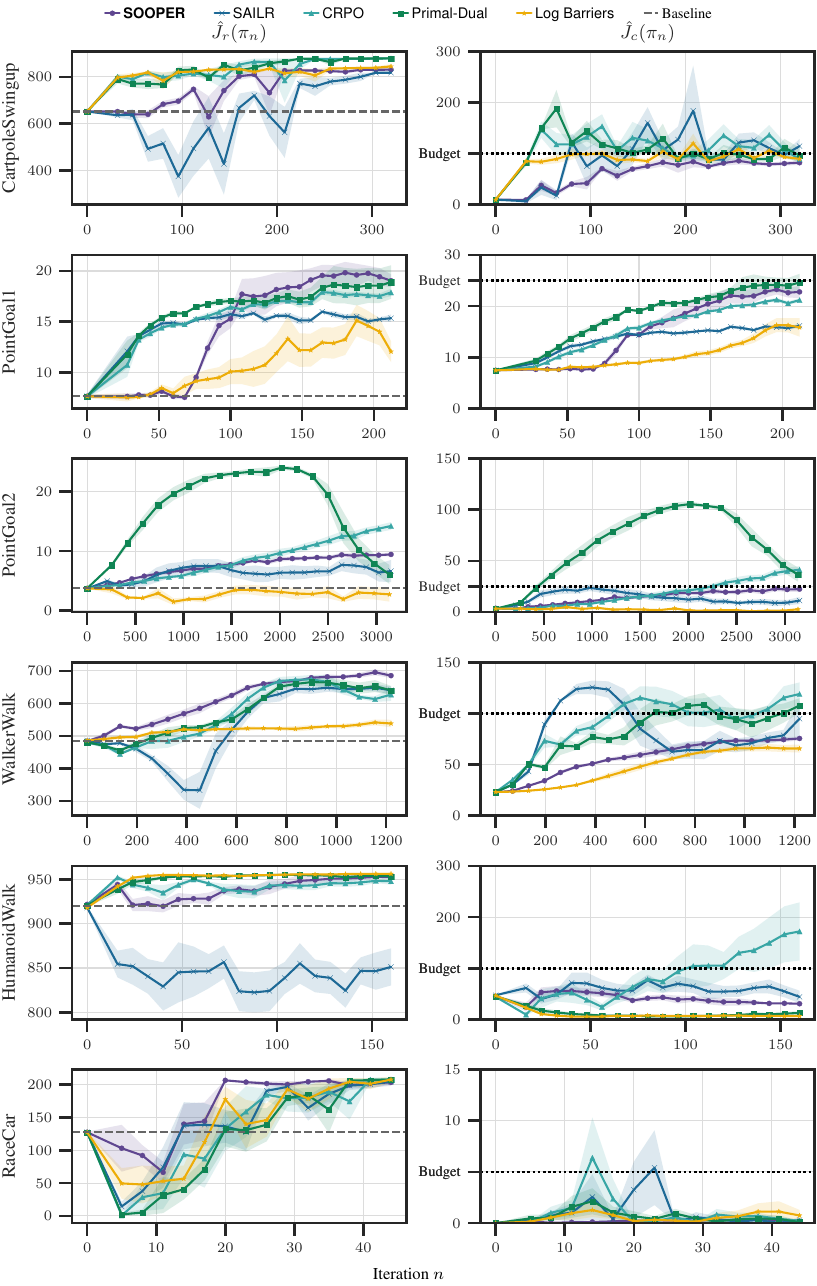}
    \caption{\looseness=-1Learning curves of the objective and the constraint for \algname{SOOPER} and comparison baselines. \algname{SOOPER} maintains safety throughout training in all tasks while outperforming the baselines in terms of performance.}
    \label{fig:convergent-learning-behavior}
\end{figure}

\section{SafetyGym}
\label{sec:safety-gym}
\paragraph{PointGoal environments.} In these tasks, the agent must navigate to a target location while avoiding hazards which include free-moving vases and designated hazard zones. The environment is depicted in \Cref{fig:safety-gym}. The initial positions of the agent, goal, vases, and hazards are randomized at the beginning of each episode. The reward function is defined as the change in Euclidean distance to the goal between successive steps
\begin{equation*}
    r_t(s_t, a_t) \eqdef d_{t-1} - d_t + \1[d_t \leq \epsilon],
\end{equation*}
where $d_t = \|\mathbf{x}_t - \mathbf{x}_{\text{goal}}\|_2$ is the Euclidean distance from the robot to the goal. The term $\1[d_t \leq \epsilon]$ is an indicator that gives a reward bonus when the agent reaches the goal, i.e. when it is within $\epsilon = 0.3$ of the center of the goal. The goal position is resampled to another free position in the environment once reached. PointGoal1 and PointGoal2 differ in the amount of obstacles located in the environment. A cost of 1 is incurred when the agent collides with a vase $v$, when one of the vases crosses a linear velocity threshold (after collision) or when the agent is inside a hazard zone $h$: 
\begin{equation*}
    c_t(s_t, a_t) \eqdef \1[\exists v\in V:\text{collides}(\mathbf{x}_t, \mathbf{x}_{v})] + \1[\exists v\in V:\mathbf{x}'_v \geq \gamma] + \1[\exists h \in H: d_t \leq \rho],
\end{equation*}
where  $\gamma=5e^{-2}$ and $\rho=0.2$. Please see the implementation of \citet{Ray2019} for more specific details.

\paragraph{Simulation gap.}
In our sim-to-sim experiments (in \Cref{fig:sim-to-sim}), before training, we uniformly sample the gear parameters of the actuators that control the linear and angular velocity of the robot. The precise ranges are given in \cref{tab:domain-randomization}. Note that the $z$-joint is a hinge joint that allows the agent to rotate around the $z$-axis and the $x$-joint is a slide joint that allows translation in the $xy$-plane.

\begin{figure}[h]
\centering
\begin{minipage}{0.45\linewidth}
    \centering
    \includegraphics[width=\linewidth]{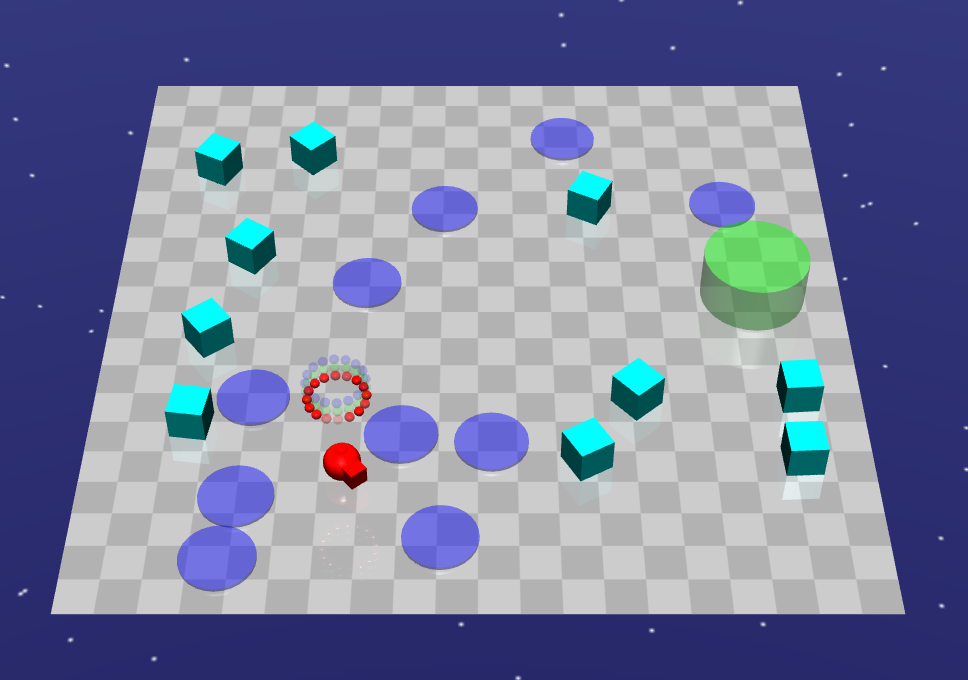}
    \caption{Visualization of a random initialization of the PointGoal2 environment. The red pointmass is the agent, the green transparent cylinder is the goal, the cyan boxes are vases and the blue circles are hazard zones.}
    \label{fig:safety-gym}
\end{minipage}\hfill
\begin{minipage}{0.47\linewidth}
    \centering
    \vspace{-0.4cm}
    {\renewcommand{\arraystretch}{1.5}
    \captionof{table}{Domain randomization parameters and ranges used during training.}
    \label{tab:domain-randomization}
    \begin{tabular}{l l}
        \hline\hline 
        \textbf{Parameter}   & \textbf{Value} \\
        \hline  
        Gear (x) & [–0.2, 0.2] \\
        Gear (z)            & [–0.1, 0.1] \\
        \hline  
    \end{tabular}
    }
\end{minipage}
\end{figure}

\section{RWRL Benchmark}
\label{sec:rwrl}
We use the RWRL benchmark suite \citep{challengesrealworld}, which adds safety constraints and distribution shifts to the tasks from the DeepMind Control benchmark suite \citep{tassa2018deepmind}. 
\begin{figure}[h]
    \centering
    \begin{tabular}{@{\hskip 5pt} c @{\hskip 5pt} c @{\hskip 5pt} c @{\hskip 5pt} c @{\hskip 5pt} c}
        \includegraphics[height=70pt, trim={2.5pt 2.5pt 2.5pt 2.5pt 2.5pt}, clip]{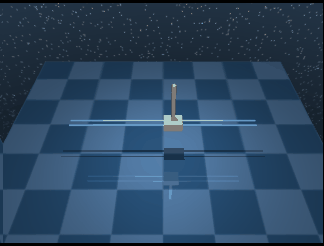} &
        \includegraphics[height=70pt, trim={2.5pt 2.5pt 2.5pt 2.5pt 2.5pt}, clip]{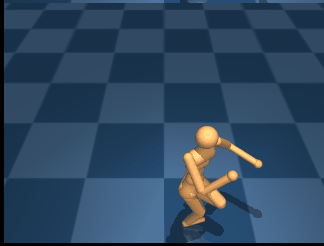} &
        \includegraphics[height=70pt, trim={2.5pt 2.5pt 2.5pt 2.5pt 2.5pt}, clip]{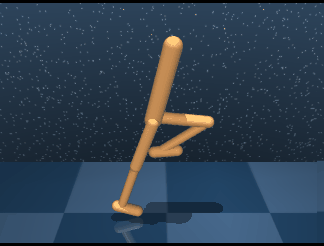}
    \end{tabular}
    \caption{RWRL tasks.}
    \vspace{-0.5cm}
    \label{fig:rwrl-demo}
\end{figure}

\paragraph{Constraints.}
We use the joint position limits constraint for HumanoidWalk, joint velocity limits for WalkerWalk, and slider position limits for CartpoleSwingup. These are the standard constraints proposed by \citet{challengesrealworld}.

\paragraph{Simulation gap.}
In our sim-to-sim experiments from \Cref{fig:sim-to-sim}, we follow a similar experimental setup as \citet{ramu}, introducing variability in the physical properties of the system during training to simulate distribution shifts. In \Cref{tab:domain-randomization-multitask} we provide the specific parameters we perturb in each task.
\begin{table}[h!]
\caption{Additive domain randomization parameters and ranges used during training of the policy prior across tasks from RWLR.}
\vspace{0.1cm}
\centering
{\renewcommand{\arraystretch}{1.5} 
\begin{tabular}{l l}
\hline\hline 
\textbf{Parameter}   & \textbf{Value} \\
\hline  
\multicolumn{2}{c}{\textbf{CartpoleSwingup}} \\
\hline
Pole Length     & [-0.1, 0.1] \\
Gear            & [-1.0, 1.0] \\
Knee Gear       & [-40., 40.0] \\
\hline
\multicolumn{2}{c}{\textbf{HumanoidWalk}} \\
\hline
Friction        & [-0.05, 0.05] \\
Hip Gear (x)    & [-20.0, 20.0] \\
Hip Gear (y)    & [-20.0, 20.0] \\
Hip Gear (z)    & [-60.0, 60.0] \\
\hline
\multicolumn{2}{c}{\textbf{WalkerWalk}} \\
\hline
Torso Length    & [-0.1, 0.1] \\
Gear            & [-5.0, 5.0] \\
\hline
\end{tabular}
}

\label{tab:domain-randomization-multitask}
\end{table}
\section{RaceCar Environment}
\label{sec:racecar}
In this environment the agent is asked to navigate to a goal position while avoiding static obstacles. This environment is implemented both in our sim-to-sim experiments, as well as on the real-world race car experiments.
\paragraph{Reward and cost.}
The reward at timestep $t$ is given by
\begin{equation*}
    r_t(s_t, a_t) \eqdef d_{t-1} - d_t + \1[d_t \leq \epsilon] - \lambda_c \|a_t\|_2 - \lambda_l \|a_t - a_{t-1}\|_2^2,
\end{equation*}
where $d_t$ is the Euclidean distance to the goal, and $a_t \in \mathbb{R}^2$ denotes the action applied at time $t$ (consisting of steering and throttle). The term $\1[d_t \leq \epsilon]$ is an indicator function that gives a reward bonus when the agent is within $\epsilon = 0.3$ of the goal. The penalties $\lambda_c$ and $\lambda_l$ weight the control effort (magnitude of the action) and the change in action between consecutive timesteps, respectively.
The cost function at time $t$ is defined as
\begin{equation*}
    c_t(s_t, a_t) \eqdef \sum_{i=1}^3 \1\left[\|x_t - p_i\| < \rho_i\right] E^k_t + \1[x_t \notin \mathcal{V}],
\end{equation*}
where $x_t \in \mathbb{R}^2$ is the car's position, $p_i$ and $\rho_i$ are the position and radius of the $i$-th obstacle and $E^k_t$ is the kinetic energy of the car at time $t$, simulating a plastic collision between the car and obstacles. This choice of cost function allows us to penalize more severely collisions in which the car smashes with high velocity into obstacles, as opposed to softly touching them. The second term penalizes the agent for leaving the valid area $\mathcal{V}$, which corresponds to a bounded rectangular arena.

\paragraph{Simulation gap.}
In the previous sim-to-sim environments, we model the sim-to-sim gap for the experiments in \Cref{fig:sim-to-sim} by introducing an auxiliary dynamics parameter (e.g., pendulum length) that is not observed during training. In contrast, in the RaceCar environment, the car dynamics in the training environments are governed by a semi-kinematic bicycle model that does not account for interactions between the tire and the ground. On the other hand, in evaluation, we use the dynamical bicycle model of \citet{kabzan2020amz}. We refer the reader to \citet{kabzan2020amz} for the detailed equations of motion.



\end{document}